\documentclass[11pt,a4paper]{article}

\usepackage[a4paper,top=1in,bottom=1in,left=1in,right=1in]{geometry} 
\usepackage{setspace}
\usepackage{natbib}
\usepackage{amsmath,amssymb,amsfonts,amsthm,bm}
\usepackage{graphicx,color,xcolor}
\usepackage{multirow,booktabs,array}
\usepackage{float,subfigure,rotating}
\usepackage{algorithm,algorithmicx,algpseudocode}
\usepackage{url,hyperref}
\usepackage{lipsum}
\usepackage{indentfirst}
\usepackage{authblk}
\usepackage{textcomp}
\usepackage{threeparttable}
\usepackage{pdflscape}
\usepackage{enumerate}

\hypersetup{
    colorlinks=true,
    linkcolor=blue,
    anchorcolor=blue,
    citecolor=blue,
    urlcolor=blue,
    CJKbookmarks=true
}

\newtheorem{theorem}{Theorem}
\newtheorem{assumption}{Assumption}
\newtheorem{lemma}{Lemma}
\newtheorem{remark}{Remark}

\newtheorem{definition}{Definition}

\def\Var{\operatorname{Var}}
\def\diag{\operatorname{diag}}
\DeclareMathOperator*{\argmax}{argmax}
\DeclareMathOperator*{\argmin}{argmin}

\makeatletter \@addtoreset{equation}{section} \makeatother

\author{Jingfu Peng}
\author{Yuhong Yang}

\affil{Yau Mathematical Sciences Center, Tsinghua University}

\date{\today}  

\title{Adversarial learning for nonparametric regression: Minimax rate and adaptive estimation}

\begin{document}

\maketitle

\begin{abstract}
Despite tremendous advancements of machine learning models and algorithms in various application domains, they are known to be vulnerable to subtle, natural or intentionally crafted perturbations in future input data, known as adversarial attacks. While numerous adversarial learning methods have been proposed, fundamental questions about their statistical optimality in robust loss remain largely unanswered. In particular, the minimax rate of convergence and the construction of rate-optimal estimators under future $X$-attacks are yet to be worked out.

In this paper, we address this issue in the context of nonparametric regression, under suitable assumptions on the smoothness of the regression function and the geometric structure of the input perturbation set. We first establish the minimax rate of convergence under adversarial $L_q$-risks with $1 \leq q \leq \infty$ and propose a piecewise local polynomial estimator that achieves the minimax optimality. The established minimax rate elucidates how the smoothness level and perturbation magnitude affect the fundamental limit of adversarial learning under future $X$-attacks. Furthermore, we construct a data-driven adaptive estimator that is shown to achieve, within a logarithmic factor, the optimal rate across a broad scale of nonparametric and adversarial classes.
\end{abstract}


\tableofcontents

\addtocontents{toc}{\setcounter{tocdepth}{2}}
\section{Introduction}\label{sec:intro}

Over the past decade, significant breakthroughs in machine/deep learning models have enabled the large-scale deployment of these models across a wide range of domains, including autonomous driving, medical diagnosis, and robotics. Many of these applications are safety-critical, meaning that the outputs made by machine learning models can directly impact human well-being. Currently, one of the most pressing threats to the safety of machine learning models is adversarial attacks on future inputs, where even seemingly small, imperceptible worst-case perturbations to the input data can drastically degrade performance \citep{Biggio2013, Szegedy2014Intriguing, Goodfellow2014Explaining}. For example, in image classification, modifying just a single pixel can cause deep neural networks to misclassify \citep{Su2019}. In traffic sign detection, physical perturbations such as varying distances and angles of the viewing camera, or putting a sticker on the traffic sign can compromise the performance of a deep learning system \citep{Eykholt_2018_CVPR}. These security concerns have led to the emergence of adversarial machine learning, focusing around analyzing attacks against machine learning algorithms and developing defenses against them. Adversarial training algorithms \citep[see, e.g.,][]{Cohen2019, Goodfellow2014Explaining, Papernot2016, Madry2018, Wong2018, Zhang2019}, which typically involve minimizing adversarial empirical losses, have proven effective in enhancing the robustness of machine learning models against a variety of adversarial perturbations.

Besides the remarkable progress in practical applications, recent research has advanced the theoretical understanding of adversarial machine learning. The existence and explicit forms of the ideal adversarial estimators and risks have been explored by \citep{Bhagoji2019, Pydi2020, Awasthi2021, Bungert2023, Frank2024}. The generalization errors of adversarially trained estimators have been studied, with a main focus on the Rademacher complexity of the adversarial loss function class \citep{khim2018adversarial, Yin2019, Tu2019, Awasthi2021, Mustafa2022}. Studies have also examined the tradeoff between robust and standard generalization errors in parametric models \citep{Dobriban2023, Javanmard2020, Javanmard2022, Schmidt2018, raghunathan2019adversarial, Tsipras2019, Zhang2019, Min2021}. More recently, \cite{Hassani2024} investigated the effects of overparameterization on adversarial robustness, showing that under certain conditions, overparameterization can hurt robust generalization. For a Bayesian perspective of adversarial learning, see \cite{Gallego2024} and the references therein.

In the aforementioned work, adversarial learning has been investigated from two primary perspectives: (i) a pointwise risk approach with respect to a given parametric model, which includes the optimal adversarial risk evaluation and the comparison between standard and adversarial risks; and (ii) an analysis of the properties of some specific training procedures under future $X$-attacks, such as adversarial training algorithms. Given the growing importance of adversarial machine learning and the recent surge of works on proposing new defenses, it is crucial to understand adversarial robustness from another perspective: What are the fundamental limits for any estimation method under adversarial attacks, and how to construct estimation procedures to achieve these limits?

In this paper, we aim to answer the aforementioned questions by determining the minimax rate of convergence under adversarial risks. To the best of our knowledge, minimax risk theory has received relatively little attention in adversarial learning, with only a few exceptions. For linear classification, the minimax rates for adversarial regret (i.e., the excess adversarial risk relative to the ideal adversarial risk across all linear classifiers) were derived by \cite{Schmidt2018} and \cite{Dan2020} for Bernoulli and Gaussian mixture models, respectively. Similar results were obtained by \cite{Dohmatob2024} under less restrictive assumptions on the data distribution. In the case of linear regression with Gaussian covariates, \cite{Xing2021} established the minimax rate under the Euclidean-norm attack, showing that the optimal rate can be achieved by a ridge regression estimator, where the ridge parameter depends on the attack's magnitude. In nonparametric regression, \cite{Peng2024} established the minimax rate of convergence for the adversarial sup-norm risk under the specific $X$-attack that perturbs every $X$ by a worst-case element in a same set. However, the results under the adversarial $L_q$-risk ($1\leq q <\infty$) with general perturbation set remain unexplored.

\subsection{Contributions}

In this paper, we focus on nonparametric regression problems where the true regression function belongs to a H\"{o}lder-smooth class, and we derive the minimax rates of convergence under adversarial $L_q$-risk for $1 \leq q \leq \infty$. We consider a general adversarial framework, in which the adversary selects a perturbation for each input from a class of perturbation sets. The geometric requirements on these perturbation sets are quite mild, which encompasses the commonly used bounded $\ell_p$-norm attacks ($0 < p \leq \infty$) as special cases. Importantly, our theory provides a general characterization of adversarial robustness, extending beyond the specific parametric models or attack types studied in the existing literature. By analyzing this general setting, we identify the intrinsic factors that influence adversarial robustness and propose theoretically optimal procedures to defend against the future $X$-attacks.

We show that the minimax rate of convergence under adversarial $L_q$-risk can be expressed as the sum of two key terms. The first is the minimax rate under the standard risk, which depends on characteristics of the true data distribution, such as the smoothness of the regression function and the dimensionality of the input. The second term captures the impact of the adversarial attack, which is influenced by the maximal perturbation magnitude and the smoothness of the target regression function. This two-term expression implies a phase transition in the minimax risk. Specifically, there exists a critical value for the attack magnitude, which approaches zero at a rate determined by the smoothness of the function class. If the attack magnitude is below this critical value, the adversarial minimax risk remains the same order as the standard risk. However, when the attack magnitude exceeds this threshold, it significantly increases the minimax risk and hence triggers a successful attack.

We propose a defense strategy based on a piecewise local polynomial (PP) method with regularization. Our approach extends the classical local polynomial (LP) methods studied in \cite{Stone1982global, Fan1995jrssb, Tsybakov2009book} to the adversarial setting. And the regularization technique we employ is of independent interest, as it guarantees the optimal convergence in expectation within the random design regression setting. We demonstrate that when the attack magnitude is small, the proposed PP estimator with the standard bandwidth choice can achieve the optimal rates under adversarial $X$-attacks. In contrast, as the attack strengthens, the bandwidth needs to be adjusted accordingly to enhance robustness so as to maintain the optimal adversarial rate. Since the PP estimator depends on the knowledge of both the attack magnitude and the smoothness level, we further propose an adaptive estimation method based on Lepski's method \citep[see, e.g.,][]{Lepski1991, Lepski1997Inhomogeneous, Lepski1997pointwise}. It achieves the minimax rate, up to at most an extra logarithmic factor, over a broad scale of function and adversarial classes.

\subsection{Other related work}

We briefly review the literature on other adversarial models along two lines.

One focuses on distribution shifts, where an adversary can alter the data distribution for future observations, making it different from the training data distribution. The trade-off between adversarial risk and standard risk in this context has been examined by \cite{Mehrabi21}. To defend against such attacks, distributionally robust optimization (DRO) have been considered, which optimizes the empirical worst-case risk over a set of possible distributions \citep[see, e.g.,][and references therein]{Ben-Tal:RobustOptimization, Shapiro2017, Sinha2018, Duchi2021, Duchi2023}. Interestingly, adversarial training can be viewed as a special case of DRO, where the set of distributions is defined by a Wasserstein distance metric \citep{staib2017distributionally}. This connection has been leveraged in works such as \cite{Lee2018, Tu2019, Liu2024} to derive upper bounds on the risk of adversarially trained estimators. However, there remains a gap due to the lack of lower bounds necessary to determine the minimax risk.

In the second line of work, adversarial attacks are on the training dataset, known as poisoning attacks \citep{Biggio2012}. In this context, \cite{Lai2020} proposed a measure for adversarial robustness, and developed an optimal estimator to minimize this measure. In a similar setting where a subset of the training samples can be arbitrarily modified by an attacker, \cite{Zhao_Wan_2024} established the minimax rates for nonparametric Lipschitz regression functions under both $L_2$ and $L_{\infty}$ losses. However, poisoning attacks differ from the adversarial $X$-attack considered in the present paper, and their results and proof techniques may not be applied to establish the minimax rate we are concerned with.

\subsection{Organization}

The rest of the paper is organized as follows. In Section~\ref{sec:setup}, we introduce some basic notations and
definitions, such as H\"{o}lder-smooth function class, adversarial $X$-attack, and adversarial $L_q$-risk, and set up the nonparametric regression problem. In Section~\ref{sec:minimax_rate}, we introduce our main result on the minimax rate, describe a construction of the rate-optimal estimator based on a PP estimator
and propose a data-driven adaptive estimator. A roadmap for the proof of our minimax lower bound is provided in Section~\ref{sec:roadmap}. We then conclude the paper with a discussion. Detailed proofs are provided in the Appendix.

\section{Problem setup}\label{sec:setup}

We first introduce notations that will be used throughout the paper and then present the nonparametric regression framework. Next, the adversarial risks are defined and compared with related ones in the literature. After that, we introduce the definition of the H\"{o}lder-smooth function class and provide a general construction of the adversarial perturbation set. We then formally define the minimax risk for the regression estimation problem under adversarial attacks.

\subsection{Notation}\label{sec:setup1}

We use the symbols $C$, $C'$, $C_1$ to represent positive constants that do not depend on the sample size $n$, and which may vary from line to line. For any positive sequences $a_n$ and $b_n$, we write $a_n = O(b_n)$ or $a_n \lesssim b_n$ if there exist constants $C > 0$ and $N > 0$ such that for all $n \geq N$, $a_n \leq C b_n$. If both $a_n = O(b_n)$ and $b_n = O(a_n)$ hold, we denote this as $a_n \asymp b_n$. We write $a_n = o(b_n)$ and $a_n \sim b_n$ for $\lim_{n \to \infty}a_n/b_n=0$ and $\lim_{n \to \infty}a_n/b_n=1$, respectively. For any $\beta \in \mathbb{R}$, we define $\lfloor \beta \rfloor$ as the greatest integer strictly less than $\beta$. The support of a probability measure $\mu$, denoted by $\operatorname{supp}(\mu)$, is defined as the complement of the largest open set on which $\mu$ is zero. Let $\lambda$ denote the Lebesgue measure. We use $1_{S}$ to represent the indicator function of the set $S$, which equals $1$ if $x \in S$ and $0$ otherwise. The notation $\mathrm{Card}(S)$ stands for the cardinality of the set $S$.

For any multi-index $s=\left(s_1, \ldots, s_d\right) \in \mathbb{N}^d$ and any vector $x=\left(x_1, \ldots, x_d\right)^{\top} \in \mathbb{R}^d$, we define $|s|=\sum_{i=1}^d s_i$, $s!=s_{1}!\cdots s_{d}!$, and $x^s=x_1^{s_1} \cdots x_d^{s_d}$. Let $D^s$ denote the differential operator $D^s \triangleq \frac{\partial^{s_1+\cdots+s_d}}{\partial x_1^{s_1} \cdots \partial x_d^{s_d}}$. The notation $\|x \|_p \triangleq (\sum_{i=1}^{d} |x_i|^p)^{1/p}$ represents the $\ell_p$-norm of the vector $x$, with $\|x\|$ reserved for the $\ell_2$-norm (i.e., suppress the subscript when $p=2$). For $0 < p \leq \infty$, we define $B_p(x, r) \triangleq \{z : \|z - x\|_p \leq r\}$ as the $\ell_p$-ball centered at $x$ with radius $r$. When $p=2$, we simplify this notation to $B(x, r) \triangleq B_2(x, r)$. Finally, $I_d$ denotes the identity matrix in $\mathbb{R}^{d \times d}$, $\mathrm{diag}(\alpha) \in \mathbb{R}^{d \times d}$ denotes the diagonal matrix whose diagonal entries equal to those in the vector $\alpha \in \mathbb{R}^{d}$, and $\lambda_{\min}(B)$ represents the smallest eigenvalue of any square matrix $B$.

\subsection{Model and data}\label{sec:setup2}

Let $(X, Y)$ be a random pair in $\mathbb{R}^d \times \mathbb{R}$ with joint distribution $\mathbb{P}_{(X, Y)}$, and assume that we are given a sample $Z_n \triangleq \{ (X_1, Y_1), \ldots , (X_n, Y_n) \}$, where $(X_i, Y_i)$ are independent copies of $(X, Y)$. We denote by $\mathbb{P}_X$ the marginal distribution of $X$, $\mathbb{P}_{\otimes^n} \triangleq \prod_{i=1}^{n}\mathbb{P}_{(X_i,Y_i)}$ the product probability measure of the sample, and $\mathbb{P} \triangleq \mathbb{P}_{(X,Y)} \times \mathbb{P}_{\otimes^n}$. Define $\mathbb{E}[\cdot]$, $\mathbb{E}_X[\cdot]$, and $\mathbb{E}_{\otimes^n}[\cdot]$ the expectations with respect to the probability measures $\mathbb{P}$, $\mathbb{P}_X$, and $\mathbb{P}_{\otimes^n}$, respectively.

Let $\Omega \triangleq \operatorname{supp}(\mathbb{P}_X)$ denote the support of $\mathbb{P}_X$. Suppose that for all $x \in \Omega$, the regression function $f(x) = \mathbb{E}(Y | X=x )$ exists and is finite. Set $\xi \triangleq Y - \mathbb{E}(Y | X )$. Then the above data generating process can be equivalently written as
\begin{equation}\label{eq:model}
  Y_i=f(X_i)+\xi_i,\quad i = 1,\ldots,n,
\end{equation}
where the random error terms satisfy $\mathbb{E}(\xi|X)=0$ almost surely. The goal of a regression procedure is to construct an estimator $\hat{f}$ of $f$ based on the observation $Z_n$, where $\hat{f}(x)=\hat{f}(x,Z_n)$ is a measurable function mapping each $x \in \Omega$ to $\mathbb{R}$.

\subsection{Adversarial $L_q$-risks}\label{subsect:adv-risk}

In this paper, we set up the regression estimation problem in the presence of the adversarial $X$-attacks in prediction, where the adversary presents any estimator with adversarially modified inputs that may damage the performance of the estimator. This $X$-attack in prediction differs from training data attacks that maliciously modify the training data $Z_n$ to hurt the performance of the resulting estimator.
Let $A:\Omega \to \mathcal{B}$ (the set of all measurable subsets of $\Omega$) denote a future $X$-attack, i.e., given any new input $x \in \Omega$, the adversary selects a perturbed input $x' \in A(x)$, where $A(x) \in \mathcal{B}$ contains all the possible adversarial examples for the point $x$. The adversary then presents $x'$ to the estimator $\hat{f}$, which returns $\hat{f}(x')$ as the output. Consequently, the worst-case performance of $\hat{f}$ under the adversarial perturbation is naturally measured by the adversarial $L_q$-risk, defined as:
\begin{equation}\label{eq:adv_risk}
  R_{A,q}(\hat{f}, f) \triangleq
    \left\{\begin{array}{ll}
\mathbb{E}\sup_{X' \in A(X)}\left| \hat{f}(X') - f(X) \right|^q &\quad  1 \leq q < \infty, \\
\mathbb{E}_{\otimes^n}\sup_{x \in \Omega}\sup_{x' \in A(x)}\left| \hat{f}(x') - f(x) \right| &\quad q = \infty.\\
    \end{array}\right.
\end{equation}
A representation example of $A$ is $A(x) =\Omega \cap B_p(x, r)$ with $0< p \leq \infty$. If $A$ is chosen as $A(x) = \{ x \}$, the adversarial risks (\ref{eq:adv_risk}) are returned to the standard $L_q$-risk $R_q(\hat{f}, f) \triangleq \mathbb{E}| \hat{f}(X) - f(X) |^q$ for $1 \leq q < \infty$ and the sup-norm risk $R_{\infty}(\hat{f}, f) \triangleq \mathbb{E}_{\otimes^n}\sup_{x \in \Omega}| \hat{f}(x) - f(x) |$, respectively.

The adversarial $L_q$-risk defined in (\ref{eq:adv_risk}) provides a general metric for assessing adversarial robustness in regression function estimation. In the special case with $q=2$ and $A(x)=\Omega \cap B(x,r)$, (\ref{eq:adv_risk}) has been adopted in the linear regression setting by \cite{Donhauser2021, Xing2021, Scetbon2023, hao2024surprising}. Additionally, it has been considered in \cite{Roth2020, Kumano2023} to examine the adversarial robustness of neural networks. It is worth mentioning that besides (\ref{eq:adv_risk}), there are several alternative measures of adversarial robustness, each emphasizing different aspects of the adversarial learning problem. Section~\ref{sec:discuss_adve_risk} presents a more detailed discussion of these measures, including comparisons and their relationships with (\ref{eq:adv_risk}). In the subsequent analysis, we focus on deriving both minimax upper and lower bounds for the adversarial risks defined in (\ref{eq:adv_risk}).

\subsubsection{Other performance measures}\label{sec:discuss_adve_risk}

As discussed in Section~\ref{subsec:other_1} of the Appendix, (\ref{eq:adv_risk}) implies a sensible assumption that the adversary operates the attack without access to the future $Y$ value. An alternative definition of adversarial risk considers the difference between $\hat{f}(X')$ and $Y$ \citep{Szegedy2014Intriguing, Goodfellow2014Explaining, Carlini2017Towards, Madry2018}, under the assumption that the adversary has knowledge of the true output $Y$. When the $L_1$-loss is adopted, it is readily seen that
\begin{equation}\label{eq:L_1_bound}
  \mathbb{E}_{(X,Y)}\max_{X' \in A(X)}| \hat{f}(X') - Y | \leq R_{A,1}(\hat{f}, f) + \mathbb{E}|\xi|.
\end{equation}
A similar upper bound in order holds for general $L_q$-losses; see Section~\ref{subsec:other_3} of the Appendix for detailed discussion. The inequality (\ref{eq:L_1_bound}) indicates that convergence in $R_{A,1}(\hat{f}, f)$ also guarantees robustness in $\mathbb{E}_{(X,Y)}\max_{X' \in A(X)}| \hat{f}(X') - Y |$, up to an additional term that accounts for the fluctuation of the random error.

Interestingly, our definition of adversarial risk is closely related to the formulation proposed by \cite{Zhang2019} that defines:
\begin{equation}\label{eq:risk_zhang}
  T_{\phi,t}(\hat{f})\triangleq\mathbb{E}[\{\phi( \hat{f}(X), Y ) + \sup_{X' \in A(X)}\phi( \hat{f}(X'), \hat{f}(X) )/t ]
\end{equation}
as the objective for adversarial training, where $\phi(\cdot,\cdot)$ is a loss function and $t$ is a positive constant that balances the regular prediction error and perturbation degradation. Optimizing the empirical counterpart of (\ref{eq:risk_zhang}) leads to the well-known \emph{TRADES} algorithm in \cite{Zhang2019}, which has been one of the most effective defense strategies in practical applications. When $\phi$ is the quadratic loss, $t=1$, and $\{x\} \subseteq A(x)$, we prove in Section~\ref{subsec:other_2} of the Appendix that
\begin{equation*}
[T_{\phi,t}(\hat{f}) - \mathbb{E}|\xi|^2]/5 \leq  R_{A,2}(\hat{f}, f) \leq 2 [T_{\phi,t}(\hat{f}) - \mathbb{E}|\xi|^2],
\end{equation*}
which shows that $R_{A,2}(\hat{f}, f)$ is equivalent to $T_{\phi,t}(\hat{f})- \mathbb{E}|\xi|^2$ up to a constant multiplier.

\subsection{Basic definitions}\label{sec:setup3}

We introduce the following definitions regarding the marginal distributions of $X$, the distributions of $\xi$ conditioned on $X$, the smoothness of the regression function $f$, and the geometric structure of the adversarial perturbation sets.

\begin{definition}[Sub-Gaussianity]\label{def:error_moment}

For a constant $\sigma \geq 0$, the class of sub-Gaussian random errors $\mathcal{E}(\sigma)$ with $\sigma \geq 0$ is defined as all random variables $\xi$ such that for any $t \in \mathbb{R}$ and $x \in \Omega$,
  \begin{equation*}
    \mathbb{E}\left[ \exp(t \xi) | X = x \right] \leq \exp\left(\frac{t^2\sigma^2}{2}\right), \quad \text{a.s.}
  \end{equation*}

\end{definition}

\begin{definition}[Bounded density]\label{def:strong_density}

For some constants $0 < \mu_{\min} < \mu_{\max} < \infty$, the marginal distribution $\mathbb{P}_X$ is said to satisfy the bounded density assumption if its support $\Omega$ is the compact set $[0, 1]^d$, and it admits a density $\mu$ with respect to the Lebesgue measure $\lambda$ on $\Omega$. The density $\mu$ satisfies the following condition:
\begin{equation}\label{eq:density_bounded}
  \mu_{\min}<\mu(x)<\mu_{\max}, \quad \forall x \in \Omega.
\end{equation}

\end{definition}
Definition~\ref{def:strong_density} restricts our analysis to compact supports and regularly sampled $X$ observations, which are central to statistical analyses \citep[see, e.g.,][]{Stone1982global, Mammen1999}. In Definition~\ref{def:strong_density}, we set $\Omega = [0, 1]^d$ for simplicity. This setting has been widely adopted in many nonparametric estimation problems \citep[see, e.g.,][]{Juditsky2009, Lepski2014single, Schmidt-Hieber2020, Cai2021}. We define $\mathcal{S}(\mu_{\min}, \mu_{\max})$ as the set of distributions supported on $[0, 1]^d$ that satisfy the bounded density assumption with parameters $\mu_{\min}$ and $\mu_{\max}$.

\begin{definition}[H\"{o}lder smoothness]\label{def:holder_smooth}

Let $\beta > 0$ and $C_\beta >0$. The $(\beta, C_{\beta})$-H\"{o}lder class of functions, denoted by $\mathcal{F}(\beta,C_{\beta})$, is defined as the set of functions $f: [0,1]^d \to \mathbb{R}$ with continuous partial derivatives of order $\lfloor \beta \rfloor$ that satisfy
\begin{equation}\label{eq:holder_smooth}
\max_{0 \leq |s| \leq \lfloor \beta \rfloor}\sup_{x \in [0,1]^d}\left| D^s f(x) \right| + \max_{|s| = \lfloor \beta \rfloor}\sup_{x,z \in [0,1]^d}
  \frac{\left|D^s f(x)-D^s f(z)\right|}{\|x-z\|^{\beta - \lfloor \beta \rfloor}} \leq C_{\beta}.
\end{equation}

\end{definition}
In this paper, we consider the nonparametric regression problem when the regression function $f$ belongs to a $\left(\beta, C_\beta\right)$-Hölder class, the marginal distribution $\mathbb{P}_{X}$ satisfies the bounded density assumption, and the random error term is sub-Gaussian. Putting Definitions~\ref{def:error_moment}--\ref{def:holder_smooth} together, we define the nonparametric distribution class:
\begin{equation*}
\begin{split}
   &\mathcal{P}\left(\beta, C_\beta, \mu_{\min}, \mu_{\max}, \sigma\right)\triangleq \left\{\mathbb{P}_{(X,Y)}:f \in \mathcal{F}(\beta,C_{\beta}),\mathbb{P}_X \in \mathcal{S}\left(\mu_{\min}, \mu_{\max}\right), \xi \in \mathcal{E}(\sigma)\right\} .
\end{split}
\end{equation*}
We may use the abbreviation $\mathcal{P}(\beta)$ to highlight the role of $\beta$ when the context is clear.

\begin{figure}[!t]
  \centering
  \includegraphics[width=3.5in]{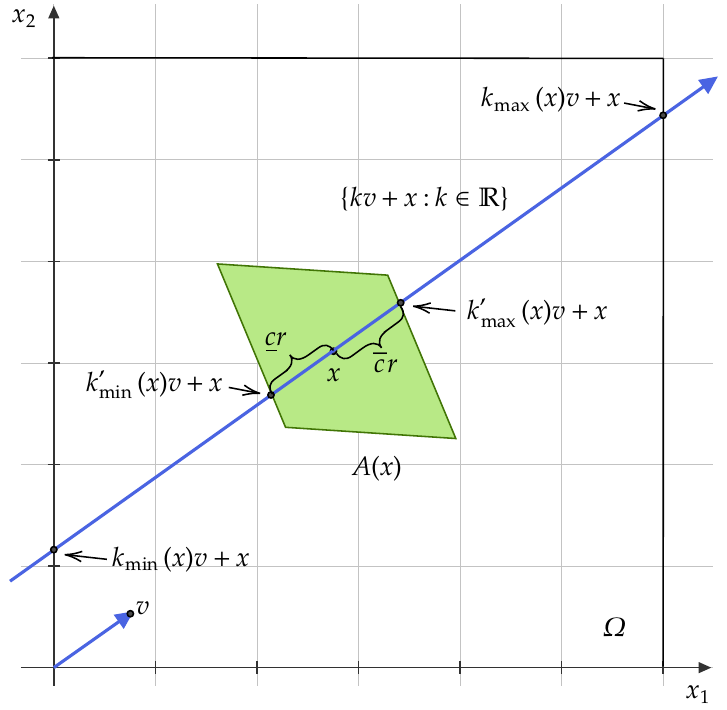}
  \caption{An illustration of the $(v,r)$-SODA satisfying the conditions in Definition~\ref{def:adversarial}. The shaded region $A(x)$ denotes the perturbation set corresponding to the input $x$.}
  \label{fig:adset}
\end{figure}

Next, we assume that for each $x \in \Omega$, the perturbation set satisfies $x \in A(x) \subseteq \Omega$. The magnitude of the adversarial attack is measured by
 \begin{equation*}
   r = r_n \triangleq \sup_{x \in \Omega,x' \in A(x)}\| x' - x \|.
 \end{equation*}
 We introduce the following definition to characterize a required geometric nature of the future $X$-attack $A$.
\begin{definition}[Shift-in-at-least-one-direction attack, SODA]\label{def:adversarial}
Let $v \in \mathbb{R}^d$ be a fixed unit vector. For any $x \in \Omega$, define $k_{\min}(x) \leq 0$ and $k_{\max}(x) \geq 0$ such that $\{k v+ x: k_{\min}(x) \leq k \leq k_{\max}(x)\} = \{k v+ x: k \in \mathbb{R}\} \cap \Omega $. Suppose there exist two constants $0\leq \underline{c},\bar{c} \leq 1$ such that the following conditions hold:
 \begin{description}
   \item[(i)] For each $x \in \Omega$, there exist $k_{\min}'(x),k_{\max}'(x) \in \mathbb{R}$ satisfying $\{k v+ x: k_{\min}'(x) \leq k \leq k_{\max}'(x)\} = \{k v+ x: k_{\min}(x) \leq k \leq k_{\max}(x)\} \cap A(x)$;
   \item[(ii)] Moreover, $k_{\min}'(x) = k_{\min}(x)  \vee (-\underline{c} r)$ and $k_{\max}'(x) = k_{\max}(x) \wedge (\bar{c} r)$.
 \end{description}
Such a future $X$-attack $A$ is called $(v, r)$-SODA, where $r$ measures the  perturbation magnitude and $v$ represents a principal direction of perturbations.
\end{definition}

Let $\mathcal{T}(r)$ denote the set of all the $(v, r)$-SODAs for which such a $v$ exists. We now provide some insight on Definition~\ref{def:adversarial}. The points $k_{\min}(x) v+ x$ and $k_{\max}(x) v+ x$ lie on the boundary of $\Omega$ and are defined as the endpoints of the line segment formed by intersecting the line $\{k v+ x: k \in \mathbb{R}\}$ with $\Omega$. Condition (i) above ensures that the intersection between $\{k v+ x: k_{\min}(x) \leq k \leq k_{\max}(x)\}$ and the perturbation set $A(x)$ is a line segment. Condition (ii) requires that this line segment has length on the order of at least $r$, except when $x$ is near the vertices of $\Omega$. Clearly, any $\ell_p$-ball attack $B_p(x,r) \cap \Omega$ satisfies Definition~\ref{def:adversarial}, with each axis serving as a principal direction. A general example of a $(v, r)$-SODA in two-dimensional domain is illustrated in Figure~\ref{fig:adset}.

In this paper, we study the adversarial learning for nonparametric regression through the minimax paradigm. Given the class $\mathcal{P}(\beta)$ of data distributions $\mathbb{P}_{(X, Y)}$ and an adversarial nature $A$ in the class $\mathcal{T}(r)$, we define the adversarial minimax risk as
\begin{equation*}
  V_{A,q}(\beta,r) \triangleq \inf_{\hat{f}}\sup_{\mathbb{P}_{(X,Y)} \in \mathcal{P}(\beta)}R_{A,q}(\hat{f}, f).
\end{equation*}
The infimum in the above formula is taken over all measurable estimators. An estimator $\tilde{f}$ is then said to be a minimax optimal defender over the classes $\mathcal{P}(\beta)$ and $\mathcal{T}(r)$ if
\begin{equation*}
  \sup_{\mathbb{P}_{(X,Y)} \in \mathcal{P}(\beta)}R_{A,q}(\tilde{f}, f) \lesssim V_{A,q}(\beta,r)
\end{equation*}
holds uniformly for all $A \in \mathcal{T}(r)$.

\section{Minimax rate of convergence}\label{sec:minimax_rate}

In this section, we establish the minimax rate of convergence for nonparametric regression estimation under the adversarial $L_q$-risk and propose a minimax optimal procedure based on the LP method. Additionally, we present a data-driven adaptive estimator that does not rely on prior knowledge of the smoothness parameter $\beta$ or the perturbation magnitude $r$.

The following theorem summarizes our main result regarding the minimax rate of convergence.

\begin{theorem}[Minimax rates of adversarial learning]\label{theo:Minimax_rate}
  Given a fixed smoothness level $0 < \beta < \infty$ and a perturbation magnitude $0 \leq r < \infty$, we have
  \begin{equation}\label{eq:minimax_rate}
    \inf_{\hat{f}}\sup_{\mathbb{P}_{(X,Y)} \in \mathcal{P}(\beta)}R_{A,q}(\hat{f}, f) \asymp \left\{\begin{array}{ll}
r^{q(1 \wedge \beta)} + n^{-\frac{q\beta}{2\beta+d}} &\quad  1 \leq q < \infty, \\
r^{1 \wedge \beta} + \left(\frac{n}{\log n}\right)^{-\frac{\beta}{2\beta+d}} &\quad q = \infty,\\
    \end{array}\right.
  \end{equation}
  for all $A \in \mathcal{T}(r)$.
\end{theorem}

We make three brief comments on the minimax rates derived from Theorem~\ref{theo:Minimax_rate}. First, when $r=0$, our results reduce to the standard minimax rates for the H\"{o}lder-smooth class, which have been established in various contexts by \cite{Stone1982global, Yang1999Information, Tsybakov2009book, Lepski2015, Chhor2024}. Second, our theory demonstrates that the adversary affects the minimax rates solely through the perturbation magnitude $r$. The rates are independent of additional geometric structures beyond the SODA nature. For example, $A$ may not be convex and can be of a lower dimension than $d$. Third, the SODA nature is utilized in lower bounding the minimax rates in (\ref{eq:minimax_rate}); see Section~\ref{sec:roadmap}. We show in Section~\ref{sec:local_poly} that the optimal rates in (\ref{eq:minimax_rate}) are attainable under a more general class of attacks than SODA.

\begin{figure}[!t]
  \centering
  \includegraphics[width=6.2in]{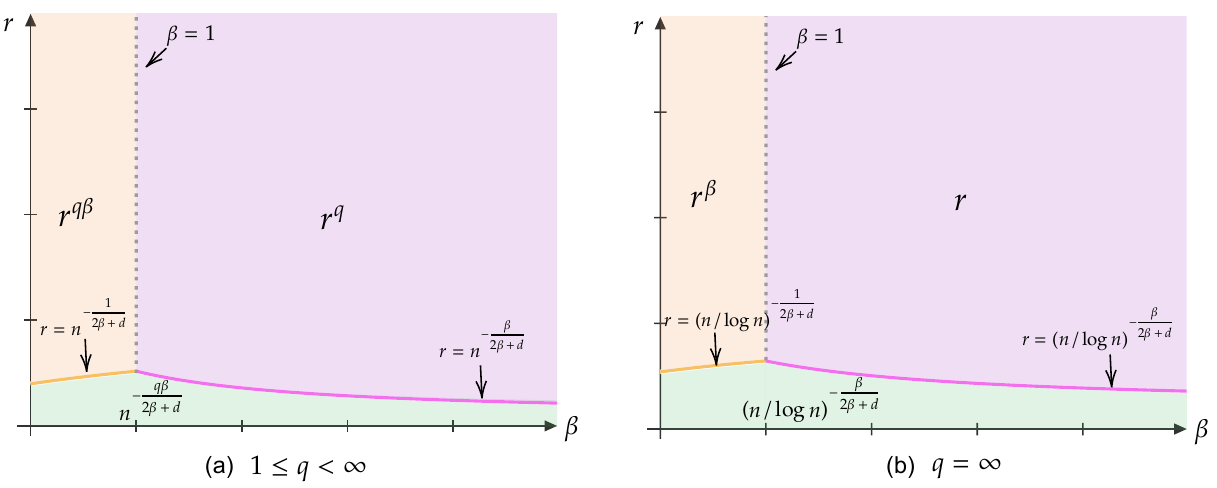}
  \caption{Diagram of the minimax rates for a fixed $d \in \mathbb{N}^+$ under the adversarial $L_q$-risks: panel (a) illustrates the case where $1 \leq q < \infty$, while panel (b) corresponds to $q = \infty$.}
  \label{fig:rate}
\end{figure}

Figure~\ref{fig:rate} depicts the three regimes of convergence behavior identified in Theorem~\ref{theo:Minimax_rate}, determined by which term in the minimax rate (\ref{eq:minimax_rate}) dominates. As shown in the left panel of the diagram, transitions occur at $r = n^{-1/(2\beta+d)}$ and $r = n^{-\beta/(2\beta+d)}$ (in order), depending on whether $0 < \beta \leq 1$ or not. When $r \leq n^{-1/(2\beta+d)}$ for $0 < \beta \leq 1$ and $r \leq n^{-\beta/(2\beta+d)}$ for $\beta > 1$, meaning that the perturbation magnitude is relatively small, the adversarial minimax rate equals the standard minimax rate $n^{-q\beta/(2\beta+d)}$. However, when the attack level exceeds these critical values, the adversarial minimax rate shifts to $r^{q\beta}$ for $0 < \beta \leq 1$ and $r^{q}$ for $\beta > 1$, leading to successful adversarial attacks in nonparametric adversarial learning. Similar arguments apply to the right panel for the minimax sup-norm risk.

Another interesting phenomenon arises in the behavior of the critical value of $r$, which exhibits a two-phase pattern as the smoothness parameter $\beta$ increases. Initially, as $\beta$ increases from $0$ to $1$, the critical value of $r$ increases, but once $\beta > 1$, it generally decreases. This is due to the relative changes between two components of the minimax rate: the maximum deviation of the target function value under adversarial attacks and the standard minimax rate. For $0 < \beta < 1$, the maximum change occurs in functions with $\beta$-H\"{o}lder smoothness, and $r^{q\beta}$ decreases more slowly than the standard minimax rate as $\beta$ increases. In contrast, for $\beta > 1$, the functions in the target class all are Lipschitz functions, and the maximum deviation becomes $r^q$, independent of $\beta$. As $\beta$ increases, the standard minimax rate decreases, making the attacker more easy to achieve the phase transitions, thus leading to a decreasing critical value of $r$.

\begin{remark}

The minimax adversarial rates in (\ref{eq:minimax_rate}) are the same for all $(v,r)$-SODA $A$ with different attack directions $v$. This is because functions in the H\"{o}lder class $\mathcal{F}(\beta, C_\beta)$ satisfy the same smoothness constraints in all directions. In contrast, under an anisotropic smoothness class, where functions may exhibit different smoothness along different axes, the attack direction $v$ may affect the adversarial risk. For example, consider a true regression function $f(x_1, x_2)$ that has H\"{o}lder smoothness $\beta_1$ along the $x_1$-axis and $\beta_2$ along $x_2$, where $0 < \beta_1 < \beta_2 \leq 1$. Then, adversarial perturbations $A$ of the magnitude $r$ in the direction $x_1$ may result in a greater degradation of performance of a well-converging estimator $\hat{f}$. For a more detailed analysis of the adversarial sup-norm risk in anisotropic smoothness classes, see Section~4.2 of \cite{Peng2024}.

\end{remark}

\begin{remark}
  In the case of adversarial sup-norm risk, \cite{Peng2024} established the minimax rate of convergence under the assumption that the perturbation set takes the form $A(x) = \{x + \delta: \delta \in \Delta \}\cap\Omega$, where $\Delta$ is a common set for all $x$. Without imposing additional conditions on $\Delta$, the minimax rate was derived for a general function class $\mathcal{F}$. In contrast, the present paper considers future $X$-attacks without assuming the particular structure of \cite{Peng2024}, and establishes the minimax rates under general $L_q$-risk for $1 \leq q \leq \infty$.
\end{remark}

\subsection{Minimax-optimal estimators via LP method}\label{sec:local_poly}

In this subsection, we describe an estimator constructed via LP method that achieves the optimal minimax rates in Theorem ~\ref{theo:Minimax_rate}. The construction is based on a regularized LP estimation, which will be first discussed in Subsection~\ref{sec:regu_lp}. Building on these regularized estimators, we propose a piecewise LP (PP) estimation method in Subsection~\ref{sec:piece_lp} that achieves theoretically optimal robustness.

\subsubsection{Regularized LP estimators}\label{sec:regu_lp}

To construct specific estimator $\hat{f}$, we introduce a regularized LP estimator in this subsection. The LP method has attracted much attention due to its massive practical success and appealing theoretical properties \citep[see][among many others]{Stone1977Consistent, Cleveland1979Robust, Tsybakov1986robust, Fan1995jrssb, Goldenshluger1997spatially, Spokoiny1998, Xiao2003, Gaiffas:2005:CRP, Audibert2007fast, Chhor2024}.

We first recall some notations related to the classical LP estimators. For a fixed integer $\ell \geq 0$, define $N_{\ell,d}$ as the cardinality of the set of multi-indices $\{ s \in \mathbb{N}^d,0 \leq |s| \leq \ell \}$. For any $x \in \mathbb{R}^d$, define a vector $U(x) \in \mathbb{R}^{N_{\ell,d}}$ as
\begin{equation*}
  U(x) \triangleq \left(\frac{x^s}{s!}\right)_{0 \leq |s| \leq \ell}.
\end{equation*}
The elements of $U(x)$ are ordered (ascending in $|s|$, with any tie-breaker method). Thus, the first component of $U(x)$ is $1$ for any $x$. For any function $f \in \mathcal{F}(\beta,C_{\beta})$, define
\begin{equation}\label{eq:taylor_poly}
  f_u(x) \triangleq \sum_{0 \leq |s| \leq \ell} \frac{\tilde{D}^s f(u)}{s!}(x - u)^s,
\end{equation}
where $\tilde{D}^s f(u) = D^s f(u)$ if $ 0 \leq |s| \leq \lfloor \beta \rfloor $, and $\tilde{D}^s f(u) = 0$ otherwise. Therefore, $f_u$ can be seen as the Taylor polynomial of $f$ of degree $\ell \wedge \lfloor \beta \rfloor$ at the point $u$. We emphasize that both $U(x)$ and $f_u(x)$ depend on the choice of $\ell$. For brevity, we omit $\ell$ from these notations, as well as from those defined below, and will mention this dependence when necessary.

Let $K:\mathbb{R}^d \to \mathbb{R}_+$ be a kernel function and $h = h_n>0$ be a bandwidth depending on $n$. Define the estimated coefficient vector as
\begin{equation}\label{eq:lpe_coeff}
  \hat{\theta}_{uh} = \argmin_{\theta \in \mathbb{R}^{N_{\ell,d}}} \sum_{i=1}^{n} \left[ Y_i - \theta^{\top}U\left(\frac{X_i - u}{h}\right) \right]^2 K\left( \frac{X_i-u}{h} \right).
\end{equation}
A LP estimator of $f(x)$ of degree $\ell$ is given by
\begin{equation}\label{eq:lpe_classical}
  \hat{f}_{uh}(x) =\hat{\theta}_{uh}^{\top}U\left(\frac{x - u}{h}\right).
\end{equation}
When the local point $u$ is chosen as $x$, $\hat{f}_{uh}(x)=\hat{f}_{xh}(x)$ is the first coordinate of the estimated vector $\hat{\theta}_{xh}$, which is reduced to the classical LP estimator considered in \cite{Tsybakov2009book}. When $u \neq x$, $\hat{f}_{uh}$ can be seen as an estimator for $f_u$, and the $s$-th component of $\hat{\theta}_{uh}$ is an estimator for $\tilde{D}^s f(u) h^{|s|}$.

It is more convenient to write the optimization problem (\ref{eq:lpe_coeff}) in matrix notation. Define a matrix $B_{uh} \in \mathbb{R}^{N_{\ell,d} \times N_{\ell,d}}$ as
\begin{equation}\label{eq:B}
  B_{uh} \triangleq \frac{1}{nh^d}\sum_{i=1}^{n} U\left( \frac{X_i - u}{h} \right) U^{\top}\left( \frac{X_i - u}{h} \right)K\left( \frac{X_i-u}{h} \right)
\end{equation}
and a vector $a_{uh} \in \mathbb{R}^{N_{\ell,d}}$ as
\begin{equation}\label{eq:a}
  a_{uh} \triangleq \frac{1}{nh^d} \sum_{i=1}^{n} Y_i U\left( \frac{X_i - u}{h} \right) K\left( \frac{X_i-u}{h} \right).
\end{equation}
A necessary condition for $\hat{\theta}_{uh}$ to be the minimizer of (\ref{eq:lpe_coeff}) is that it satisfies the linear equation $B_{uh} \hat{\theta}_{uh} = a_{uh}$. If the matrix $B_{uh}$ is positive definite, then the solution of this equation is unique and is given by $\hat{\theta}_{uh} = B_{uh}^{-1}a_{uh}$. The LP regression estimator (\ref{eq:lpe_classical}) is then defined as
\begin{equation*}
  \hat{f}_{uh}(x) = U^{\top}\left(\frac{x - u}{h}\right)B_{uh}^{-1}a_{uh}.
\end{equation*}

To address the cases where the matrix $B_{uh}$ is degenerate or its smallest eigenvalue is too small, we propose a modified version of LP estimator. From now on, we assume that the kernel function $K$ satisfies the following assumption.
\begin{assumption}\label{ass:kernel}
  There exists constants $0 <K_{\min} < K_{\max} < \infty$ and $0 < \Delta < 1$ such that
\begin{equation}\label{eq:K_ass_1}
  K_{\min} 1_{\{ \| u \| \leq \Delta \}} \leq K(u) \leq K_{\max} 1_{\{ \| u \| \leq 1 \}}, \quad \forall u \in \mathbb{R}^d,
  \end{equation}

\end{assumption}
Note that the upper bound in (\ref{eq:K_ass_1}) implies that the support of $K$ is contained within the unit ball $B(0,1)$. The lower bound on $K(u)$ for $\| u \| \leq \Delta$ is standard and is commonly imposed in the existing literature \citep[see, e.g., Theorem 1.7 of][]{Tsybakov2009book}. Classical kernels that satisfy Assumption~\ref{ass:kernel} include the rectangular kernel $K(u)=(1/v_d) 1_{\| u \| \leq 1}$ and the Epanechnikov kernel $K(u) = (d+2)/(2 v_d)(1-\|u\|^2)1_{\| u \| \leq 1}$, where $v_d$ is the volume of the unit ball in $\mathbb{R}^d$.

Let us define $n_{uh} \triangleq \mathrm{Card}\{X_i:X_i \in B(u, h) \}$ as the number of observations in the closed Euclidean ball $B(u,h)$. When $n_{uh} > 0$, we introduce a regularized matrix
\begin{equation}\label{eq:B_tilde}
  \tilde{B}_{uh} \triangleq B_{uh} + \tau I_{N_{\ell,d}} 1_{\{ \lambda_{\min}(B_{uh}) < \tau \}},
\end{equation}
where $\tau = \tau_n>0$ is a regularization parameter that tends to $0$. Note that we always have $\lambda_{\min}(\tilde{B}_{uh}) \geq \tau>0$ when $n_{uh}>0$. Then the regularized LP estimator $\mathrm{LP}(\ell,h,\tau)$ of $f(x)$ is defined as
\begin{equation}\label{eq:lpe_modified}
  \tilde{f}_{uh}(x) = \left\{\begin{array}{ll}
U^{\top}\left(\frac{x - u}{h}\right)\tilde{B}_{uh}^{-1}a_{uh} &\quad n_{uh}>0, \\
0 &\quad n_{uh}=0.
\end{array}\right.
\end{equation}

\begin{remark}
For pointwise estimation in a univariate regression setting, a similar regularization technique was adopted by \cite{Gaiffas:2005:CRP} to address matrix degeneration. In the classification problem, \cite{Audibert2007fast} introduced a LP estimator, which is defined as $\hat{f}_{xh}(x)$ when $\lambda_{\min}(B_{xh}) \geq (\log n)^{-1}$, and as $0$ otherwise. However, neither approach has been shown to achieve minimax optimality in expectation for general random design nonparametric regression. More recently, \cite{Chhor2024} proposed estimating the coefficient vector as $B_{xh}^+a_{xh}$, where $B_{xh}^+$ denotes the Moore-Penrose inverse of $B_{xh}$, and then truncated the LP estimator by the maximum value $\max_{1\leq i \leq n}|Y_i|$. Their method is related to (\ref{eq:lpe_modified}) via the limiting relationship $B_{xh}^+a_{xh} = \lim_{\tau \to 0}(B_{xh} + \tau I_{N_{\ell,d}})^{-1}a_{xh}$ \citep[see, e.g., Section~2.2 of][for further discussion on this relation]{Hastie2022}. However, to achieve minimax optimality in expectation, our estimator does not rely on the truncation technique used in \cite{Chhor2024}.
\end{remark}

\subsubsection{PP estimators}\label{sec:piece_lp}

In the conventional nonparametric regression setting, with an appropriate choice of the bandwidth $h$, several versions of LP estimators $\hat{f}_{xh}$ have been shown to achieve minimax rates of convergence for the standard $L_q$-loss/risk \citep[see, e.g.,][]{Stone1982global, Gyorfi2002book, Chhor2024}.

In adversarial setting, one has access to clean (unperturbed) data $Z_n$ generated from $\mathbb{P}_{(X,Y)}$, while the performance of learned estimators is measured by the adversarial risks (\ref{eq:adv_risk}). The main challenge is that the results established under the standard $L_2$ risks cannot imply the adversarial robustness directly, even when the attack magnitude is extremely small. To illustrate this, consider the case where $d=1$, $\Omega = [0,1]$, $\mu(x) = 1_{\{ x\in [0,1] \}}$, $f(x)=0$, $A(x) = [0\vee (x-r), (x+r)\wedge 1]$ with $r = \exp(-n)$. An estimator whose pointwise risks in $x$ do not converge at $\lfloor\exp(n)\rfloor$ equally spaced points on $[0,1]$ can still achieve optimal performance in terms of the $L_2$-risk, as the non-converging points occur on a set of zero measure. In contrast, the adversarial $L_2$-risk of this estimator does not converge at all. This example illustrates that achieving adversarial robustness may require stronger forms of convergence than those implied by the $L_2$-risk.

Therefore, to accommodate the existing LP methods in the adversarial setting and establish theoretically provable upper bounds on the adversarial $L_q$-risks, we propose a piecewise version of LP estimators. For any given point $x \in \Omega$, the estimation strategy for $f(x)$, referred to as the PP estimator, is summarized as follows:
\begin{description}
  \item[Step 1:] For an integer $M = M_n \geq 1$, define a set of discretization points on $\Omega$ as
  \begin{equation}\label{eq:discrete_set}
    \Lambda_{M} \triangleq \left\{ \left( \frac{2k_1+1}{2M},\ldots, \frac{2k_d+1}{2M}\right):k_i \in \{ 0,\ldots,M-1 \},i=1,\ldots,d \right\}.
  \end{equation}
  Let $u_x \in \Lambda_{M}$ be the closest point to $x \in \Omega$ among the points in $\Lambda_{M}$. If there are multiple points in $\Lambda_{M}$ closest to $x$, we define $u_x$ as the one closest to $0$.

  \item[Step 2:] The regression function at the point $x\in \Omega$ is estimated by
  \begin{equation}\label{eq:lpe_piese}
    \hat{f}_{\mathrm{PP}}(x) = \tilde{f}_{u_xh}(x),
  \end{equation}
  where $\tilde{f}_{u_xh}(x)$ is the regularized LP estimator $\mathrm{LP}(\ell,h,\tau)$, as defined in (\ref{eq:lpe_modified}), with $u = u_x$.
\end{description}

The final regression estimator $\hat{f}_{\mathrm{PP}}(x)$ is called the PP estimator $\mathrm{PP}(M,\ell,h,\tau)$. The performance of the PP estimator $\hat{f}_{\mathrm{PP}}$ depends on the choice of $(M,\ell,h,\tau)$. The next theorem gives a set of choices of $(M,\ell,h,\tau)$ and provides upper bounds on the adversarial $L_q$-risks, offering a performance guarantee for the PP estimator under adversarial $X$-attacks.

\begin{theorem}[Minimax upper bound]\label{theo:upper}

Let $\hat{f}_{\mathrm{PP}}$ be the PP estimator $\mathrm{PP}(M,\ell,h,\tau)$ with $\ell \geq \lfloor \beta \rfloor$,
\begin{equation}\label{eq:tuning_hn}
    h \asymp \left\{\begin{array}{ll}
r \vee n^{-\frac{1}{2\beta+d}} &\quad  1 \leq q < \infty, \\
r \vee \left(\frac{n}{\log n}\right)^{-\frac{1}{2\beta+d}} &\quad q = \infty,\\
    \end{array}\right.
\end{equation}
$M \asymp 1/h$, $\tau\to 0$, and $ n^{\gamma} \tau\to \infty$ for some fixed $\gamma>0$. Then
  \begin{equation}\label{eq:upper}
    \sup_{\mathbb{P}_{(X,Y)} \in \mathcal{P}(\beta)}\sup_{A \in \mathcal{T}^{\prime}(r)}R_{A,q}(\hat{f}_{\mathrm{PP}}, f) \lesssim \left\{\begin{array}{ll}
r^{q(1 \wedge \beta)} + n^{-\frac{q\beta}{2\beta+d}} &\quad  1 \leq q < \infty, \\
r^{1 \wedge \beta} + \left(\frac{n}{\log n}\right)^{-\frac{\beta}{2\beta+d}} &\quad q = \infty,\\
    \end{array}\right.
  \end{equation}
  where $\mathcal{T}^{\prime}(r)$ denotes the class of $A$ satisfying $x \in A(x) \subseteq \Omega$ and $\sup_{x \in \Omega,x' \in A(x)}\| x' - x \| \leq r$.
\end{theorem}

Note that the upper bounds in (\ref{eq:upper}) hold for all future $X$-attacks with attack magnitude bounded by $r$. Since $\mathcal{T}(r) \subset \mathcal{T}^{\prime}(r)$, the proposed PP estimator $\hat{f}_{\mathrm{PP}}$ also achieves the minimax rates established in Theorem~\ref{theo:Minimax_rate} under SODA. Theorem~\ref{theo:upper} suggests the critical role of bandwidth order in the adversarial robustness of the PP estimator. When the attack magnitude is small, the PP estimator with the standard bandwidth achieves the optimal adversarial rate. In contrast, when the attack magnitude is dominant, the bandwidth should be chosen proportional to $r$ to enhance robustness.

In the following, we highlight the main ingredient in proving the minimax upper-bound for $1 \leq q < \infty$, which also reveals the factors influencing an estimator's convergence in adversarial risk. By the definition of the adversarial $L_q$-risk in (\ref{eq:adv_risk}) and Jensen's inequality, we can derive the following basic inequality:
\begin{equation}\label{eq:connection1}
  \begin{split}
    & R_{A,q}(\hat{f}_{\mathrm{PP}}, f) = \mathbb{E}\sup_{X' \in A(X)}\left| \hat{f}_{\mathrm{PP}}(X') - f(X) \right|^q \\
     &  \leq 2^{q-1} \mathbb{E}\sup_{X' \in A(X)}\left| \hat{f}_{\mathrm{PP}}(X') - f(X')  \right|^q + 2^{q-1}\mathbb{E}_X\sup_{X' \in A(X)}\left|  f(X') - f(X) \right|^q.
  \end{split}
\end{equation}
Thus, the task of upper bounding the adversarial risk of $\hat{f}_{\mathrm{PP}}$ boils down to controlling both the \emph{localized uniform risk} $\mathbb{E}\sup_{X' \in A(X)}| \hat{f}_{\mathrm{PP}}(X') - f(X') |^q$ induced by the perturbation set $A(x)$ and the maximum deviation of the regression function $f$ within the adversarial perturbation set, denoted by $\mathbb{E}_X\sup_{X' \in A(X)}|  f(X') - f(X) |^q$. Note that the second term is independent of the estimator and can be bounded by analyzing the local Lipschitz properties of the functions in $\mathcal{F}(\beta, C_{\beta})$. To bound the first term, we utilize the approximation properties of polynomial functions and bound the estimation error by analyzing the sub-Gaussian complexity terms induced by the adversarial set. For a detailed proof, refer to Section~\ref{sec:proof_upper} of the Appendix.

The inequality (\ref{eq:connection1}) reveals a connection between adversarial robustness and localized sup-norm convergence, which seems to have not been stated in the existing literature. To the best of our knowledge, defending adversarial attacks from a sup-norm perspective remains largely unexplored in adversarial learning, except for the recent work by \cite{Peng2024}, which focuses specifically on the convergence of adversarial $L_\infty$-risk. Based on (\ref{eq:connection1}), we conjecture that some local nonparametric estimators exhibiting optimal global sup-norm convergence \citep[see, e.g.,][]{Stone1975nearest, Stone1982global, Bertin2004sup, Gaiffas2007Sharp} may also achieve optimality in localized sup-norm convergence, thereby attaining adversarial optimality.

\begin{remark}
  A distinct feature of the proposed estimator (\ref{eq:lpe_piese}) compared to adversarial training strategies \citep{Madry2018} is that the latter relies on the optimization of an adversarial empirical loss, whereas (\ref{eq:lpe_piese}) is a local averaging estimator. The piecewise structure and the adjusted bandwidth are crucial in achieving adversarial robustness. The minimax rates established in Theorem~\ref{theo:Minimax_rate} show that the proposed estimator is indeed optimal and cannot be improved under the general assumptions we consider. Investigating the optimality of adversarial training strategies is of great interest, however, is beyond the scope of this paper.
\end{remark}

\subsection{Data-driven adaptive estimator}\label{sec:adaptive}

In the previous subsections, we established the minimax optimal rates for adversarial learning over the nonparametric class $\mathcal{P}(\beta)$ and the adversarial attack class $\mathcal{T}^{\prime}(r)$. These optimal rates can be achieved by the $\mathrm{PP}(M,\ell,h,\tau)$ estimator with an $r$-truncated bandwidth, as shown in Theorem~\ref{theo:upper}. A major limitation of this estimator is its dependence on the smoothness parameter $\beta$ and the perturbation magnitude $r$, which are typically unknown in practice. An intriguing and practically important question arises: Can we construct an adaptive estimator based on the data that attains the optimal rates of convergence across some wide scales of parameter spaces $\mathcal{P}(\beta)$ and $\mathcal{T}^{\prime}(r)$? Recall that $\mathcal{T}^{\prime}(r)$, as defined in Theorem~\ref{theo:upper}, consists of all $A$ with attack magnitude no greater than $r$.

In this subsection, we address the aforementioned adaptive problem over the scales
\begin{equation}\label{eq:chuzou}
  \mathbf{P} \triangleq \left\{ \mathcal{P}(\beta): 0 < \beta \leq \beta_{\max} \right\}
\end{equation}
and
\begin{equation}\label{eq:jiuju}
  \mathbf{T} \triangleq \left\{ \mathcal{T}^{\prime}(r): 0 \leq r \leq r_{\max} \right\},
\end{equation}
where $\beta_{\max}>0$ and $r_{\max}>0$ are two arbitrary constants. In the conventional setting without adversarial attacks, adaptation over the scales of smooth function classes can be achieved using Lepski's method \citep{Lepski1991, Lepski1997Inhomogeneous, Lepski1997pointwise, Gaiffas2009sinica} and its extension \citep{GOLDENSHLUGER2008}, model selection \citep[see, e.g.,][]{Barron1999, Yang1999sinica}, and model aggregation \citep[see, e.g.,][]{Yang2001ARM, Juditsky2000, Tsybakov2003, Catoni2004}. Recently, \cite{Chhor2024} achieved the minimax adaptation under the standard $L_2$-risk over the scale $\mathbf{P}$, utilizing the cross-validation approach developed in \cite{Wegkamp2003}.

We construct a data-driven procedure to achieve adaptation over the scales $\mathbf{P}$ and $\mathbf{T}$ simultaneously. The main idea is to first consider a collection of candidate $\mathrm{PP}(M,\ell,h,\tau)$ estimators and then select one of them using Lepski's method. The candidate estimators are constructed with the following parameter settings. Assuming $n$ is large enough, we fix $M \asymp n$ and $\ell \geq \lfloor \beta_{\max} \rfloor$ for all candidates. Following the construction of weakly geometrically increasing blocks from \cite{Goldenshluger2001, Chhor2024}, we define grid points for the smoothness levels $\beta_j$ as follows:
\begin{equation}\label{eq:exinnv5}
  \beta_j \triangleq \left( 1 + \frac{1}{\log n} \right)^j, \quad j = - J,\ldots, J_{\max},
\end{equation}
where $J = 2 \lfloor \log n \log\log n  \rfloor$ and $J_{\max} = J \wedge \lfloor \log n \log \beta_{\max}  \rfloor$. Using (\ref{eq:exinnv5}), we define a set of candidate bandwidths:
\begin{equation}\label{eq:H_set}
  H \triangleq \left\{ h_j: h_j = n^{-\frac{1}{2\beta_j+d}}, j =  - J,\ldots, J_{\max} \right\}.
\end{equation}
 The regularization parameter is set as $\tau = \tau(h) \triangleq (nh^d)^{-1}$. From now on, we further assume that the kernel function satisfies the following assumption.
\begin{assumption}\label{ass:kernel_2}
  Suppose the kernel function satisfies Assumption~\ref{ass:kernel}. Moreover, for any $u_1,u_2 \in \mathbb{R}^d$ such that $\| u_1 \| \leq \| u_2 \|$, we assume that $K(u_1) \geq K(u_2)$.
\end{assumption}

Based on the above construction, selecting the candidate $\mathrm{PP}(M,\ell,h,\tau)$ estimators reduces to choosing a bandwidth from the set (\ref{eq:H_set}). To present the bandwidth selection process via Lepski's method, we introduce the following notation. Define $R_{h'h}$ as the diagonal matrix
\begin{equation}\label{eq:jiaquan3}
  R_{h'h} \triangleq \diag\left[\left(\frac{h^{'|s|}}{h^{|s|}}\right)_{0 \leq |s| \leq \ell}\right],
\end{equation}
whose diagonal entries consist of the ratios of the corresponding powers of the bandwidths $h'$ and $h$, indexed by the order $|s|$. For any $x \in \Omega$, we then define
\begin{equation}\label{eq:D_uk}
  D_{u_xh} \triangleq \diag\left[\left(\frac{1}{nh^d}\sum_{i=1}^{n} \frac{1}{(s!)^2}\left(\frac{X_i - u_x}{h} \right)^{2s} K\left( \frac{X_i-u_x}{h} \right)\right)_{0 \leq |s| \leq \ell}\right],
\end{equation}
and $\tilde{D}_{u_xh} \triangleq D_{u_xh} + \tau I_{N_{\ell,d}} 1_{\{ \lambda_{\min}(D_{u_xh}) < \tau \}}$, where $u_x \in \Lambda_M$ is the closest point in $\Lambda_M$ to $x$. The adaptive estimation proceeds in the following three steps:

\begin{description}
  \item[Step 1:] Construct a set of discrete points $\Lambda_M$, as defined in (\ref{eq:discrete_set}).

  \item[Step 2:] Select the data-driven bandwidth by
  \begin{equation}\label{eq:select:band}
  \begin{split}
     \hat{h}_{x}  & \triangleq \argmax_{ h \in H} \left\{h: \forall h' \in H, h' \leq h, \left\| \tilde{D}_{u_xh'}^{-\frac{1}{2}}\left(  a_{u_xh'} - B_{u_xh'}R_{h'h}\tilde{B}_{u_xh}^{-1}a_{u_xh}\right) \right\|_{\infty}\right.\\
     &\left. \qquad\qquad\qquad\qquad\qquad\qquad\quad\, \leq  C\sqrt{\frac{\log n}{nh'^d}} \right\},
  \end{split}
  \end{equation}
  where $C > 0$ is a sufficiently large constant, and $B_{u_xh}$, $\tilde{B}_{u_xh}$, and $a_{u_xh}$ are defined in (\ref{eq:B}), (\ref{eq:B_tilde}), and (\ref{eq:a}), respectively.
  \item[Step 3:] Output the adaptive estimator at $x$ as
  \begin{equation}\label{eq:buxianliang1}
    \hat{f}_{\mathrm{AP}}(x) = \tilde{f}_{u_x\hat{h}_x}(x),
  \end{equation}
  where $\tilde{f}_{uh}(\cdot)$ is defined in (\ref{eq:lpe_modified}).
\end{description}

Notably, the proposed adaptive estimator $\hat{f}_{\mathrm{AP}}$ does not require knowledge of the true smoothness parameter $\beta$ or the perturbation magnitude $r$. The following theorem provides an upper bound on the adversarial $L_q$-risks of $\hat{f}_{\mathrm{AP}}$.

\begin{theorem}[Adaptive upper bound]\label{theo:adaptive1}
  For any $\mathcal{P}(\beta) \in \mathbf{P}$ and $\mathcal{T}^{\prime}(r) \in \mathbf{T}$, we have
  \begin{equation}\label{eq:adaptive}
    \sup_{\mathbb{P}_{(X,Y)} \in \mathcal{P}(\beta)}\sup_{A \in \mathcal{T}^{\prime}(r)}R_{A,q}(\hat{f}_{\mathrm{AP}}, f) \lesssim \left\{\begin{array}{ll}
r^{q(1 \wedge \beta)} + \left(\frac{n}{\log n}\right)^{-\frac{q\beta}{2\beta+d}} &\quad  1 \leq q < \infty, \\
r^{1 \wedge \beta} + \left(\frac{n}{\log n}\right)^{-\frac{\beta}{2\beta+d}} &\quad q = \infty.\\
    \end{array}\right.
  \end{equation}
\end{theorem}

The proof of Theorem~\ref{theo:adaptive1} is based on a decomposition similar to that in (\ref{eq:connection1}), which bounds the adversarial $L_q$-risk by the localized sup-norm risk and the maximum deviation of the regression function. The proposed criterion (\ref{eq:select:band}) extends the classical Lepski's method \citep{Lepski1991, Lepski1997Inhomogeneous, Lepski1997pointwise} to the nonparametric regression setting with random design, aiming to achieve adaptation over the scales $\mathbf{P}$ and $\mathbf{T}$ under the localized sup-norm risk. The detailed proof of Theorem~\ref{theo:adaptive1} is provided in Section~\ref{sec:proof:buxiang} of the Appendix.

Comparing the convergence rates in (\ref{eq:adaptive}) with the (non-adaptive) minimax optimal rates in (\ref{eq:minimax_rate}), we observe that the data-driven estimator $\hat{f}_{\mathrm{AP}}$ attains the minimax optimal rate over the scales $\mathbf{P}$ and $\mathbf{T}$ under the adversarial sup-norm risk. In the case of adversarial $L_q$-risk with $1 \leq q < \infty$, however, the adaptive upper bound in (\ref{eq:adaptive}) incurs a logarithmic loss.

\begin{remark}

The logarithmic efficiency loss observed in Theorem~\ref{theo:adaptive1} is a common phenomenon in other adaptation problems employing Lepski's method \cite[see, e.g.,][]{Lepski2014single, Cai2021, Tang2023manifold}.  In the standard setting without adversarial perturbations ($r=0$), the non-adaptive minimax $L_2$ rate matches the adaptive minimax rate, which can be achieved through the methods including adaptive regression by mixing \citep{Yang2001ARM} and cross-validation \citep{Chhor2024}. However, when the risk is evaluated locally, such as in a pointwise risk, there is an inherent logarithmic gap between adaptive and non-adaptive optimal rates \citep[see, e.g.,][]{Lepski1991, Brown1996constrained, Lepski1997pointwise, Gaiffas_2007_ESAIM, Cai2005shrinking}. Whether the logarithmic increase is also an unavoidable cost for adaptive adversarial learning remains an open question.

\end{remark}

\section{Proof of main results}\label{sec:roadmap}

To establish the minimax rates presented in Theorem~\ref{theo:Minimax_rate}, it is sufficient to construct the minimax upper and lower bounds and to show that these bounds match to constant factors. The proofs of the minimax upper bounds follow the strategy outlined in Section~\ref{sec:local_poly} and are provided in the Appendix due to space constraints. In this section, we prove the following minimax lower bounds on the adversarial risks.

\begin{theorem}[Minimax lower bound]\label{theo:lower}
  For any $A \in \mathcal{T}(r)$, the minimax adversarial $L_q$-risks are lower bounded by
  \begin{equation}\label{eq:lower_bound}
    \inf_{\hat{f}}\sup_{\mathbb{P}_{(X,Y)} \in \mathcal{P}(\beta)}R_{A,q}(\hat{f}, f) \gtrsim \left\{\begin{array}{ll}
r^{q(1 \wedge \beta)} + n^{-\frac{q\beta}{2\beta+d}} &\quad  1 \leq q < \infty, \\
r^{1 \wedge \beta} + \left(\frac{n}{\log n}\right)^{-\frac{\beta}{2\beta+d}} &\quad q = \infty,\\
    \end{array}\right.
  \end{equation}
  where the infimum is taken over all measurable estimators $\hat{f}$.
\end{theorem}

The lower bound $r^{1 \wedge \beta} + (n/\log n)^{-\beta/(2\beta+d)}$ for $q = \infty$ can be derived using the technique developed in \cite{Peng2024}. A complete proof for this case is provided in the Appendix. In the following, we focus on establishing the lower bound for the cases $1 \leq q < \infty$. The primary strategy for deriving minimax lower bounds involves fixing an $A\in \mathcal{T}(r)$ and constructing a subset of distributions within $\mathcal{P}(\beta)$ that are statistically hard to distinguish. In particular, under the adversarial attack $A$, no estimation procedure can uniformly estimate the regression functions of this subset of distributions with high accuracy.

Specifically, we derive our lower bounds using a localized version of Fano's method \citep{Yang1999Information, Yang2004improve, Wang2014Adaptive}. For standard nonparametric regression problems, the classic Fano's and Le Cam's methods \citep[see, e.g.,][]{LeCam1973, Birge1983, Tsybakov2009book} typically suffice to establish minimax lower bounds by considering a set of distributions that differ locally, without imposing particular requirements on the base distribution around which they are localized. Often, the base distribution is chosen to have simple features, such as a regression function $f(x) = 0$ \citep[see, e.g., Chapter~2.6.1 of][]{Tsybakov2009book}. However, in the adversarial setting, this construction makes the set of regression functions overly flat and therefore incapable of obtaining the slower rate $r^{q(1 \wedge \beta)}$ when the attack magnitude is relatively large. In this section, we construct a set of distributions localized around a carefully chosen base distribution, ensuring that any estimation procedure achieving good regression estimation on these distributions remains significantly vulnerable to the adversarial attack $A$.

We introduce several notations that will be useful. For any functions $f$ and $g$ in the $L_q([0,1]^d)$ space, define $\| f \|_q \triangleq [\int_{[0,1]^d} |f(x)|^q dx]^{1/q}$ as the standard $L_q$-norm of $f$, and $d_q(f,g) \triangleq \| f - g\|_q$ as the standard $L_q$-distance between the functions $f$ and $g$. Given an adversarial strategy $A$, the adversarial $L_q$-distance between $f$ and $g$ is defined as
\begin{equation}\label{eq:gunbal1}
  d_{A,q}(f,g) \triangleq \left[\int_{[0,1]^d} \sup_{x' \in A(x)}\left|f(x') - g(x)\right|^q dx\right]^{\frac{1}{q}}.
\end{equation}
It is worth noting that $d_{A,q}$ is neither a metric nor a semi-metric, as it is not symmetric and does not satisfy the triangle inequality. Additionally, we introduce the functional
\begin{equation}\label{eq:gunbal2}
  G_{A,q}(f) \triangleq  \frac{1}{2}\left\{\int_{[0,1]^d} \left[\sup_{x' \in A(x)}f(x') - \inf_{x' \in A(x)}f(x')\right]^q dx\right\}^{\frac{1}{q}}.
\end{equation}
to measure the maximum deviation of $f$ within the perturbation set $A$. For any real number $a$, let $a_+ \triangleq \max(a,0)$.

\subsection{Construction of base distributions}\label{sec:base}

Without loss of generality, we focus on an adversarial attack $A \in \mathcal{T}(r)$ with the principle direction $v=(1,0,\ldots,0)^{\top}$. Otherwise, we can construct new functions from the functions $f_0$ defined below by rotations of axes. Note that the rotation of a function does not change the smoothness properties of the original function. To construct the base distributions, we consider a setting where the marginal distribution $\mathbb{P}_X$ is a uniform distribution on $[0,1]^d$, and $\xi$ follows a Gaussian distribution $N(0,\sigma^2)$ conditioned on $X$. This setup satisfies the conditions outlined in Definitions~\ref{def:error_moment}--\ref{def:strong_density}. The regression function for the base distribution is constructed as follows.

Suppose $r< 1/8$. For $0 < \beta \leq 1$, let us define a function $\phi_0$ on $[0,1]$ by
\begin{equation}\label{eq:phi_00}
\phi_0(x)\triangleq\left\{\begin{array}{ll}
 0  &\quad 0\leq x<\frac{1}{4}, \\
 \left(x - \frac{1}{4}\right)^{\beta} &\quad \frac{1}{4} \leq  x < \frac{1}{2}, \\
 \frac{1}{4^{\beta}} &\quad \frac{1}{2} \leq  x <  \frac{3}{4},\\
 (1-x)^{\beta} &\quad \frac{3}{4} \leq  x \leq 1.
\end{array}\right.
\end{equation}
Define $a_k \triangleq 8 k r $ for $k=0,\ldots,K$, where $K \triangleq \lfloor  1/(8r) \rfloor >1$. The base regression function $f_0:[0,1]^d \to \mathbb{R}$ is defined by
\begin{equation}\label{eq:f_00}
f_0(x)=\left\{\begin{array}{ll}
 \frac{C_{\beta}}{2}(8r)^{\beta}\phi_0\left( \frac{x_1-a_{k-1}}{8r} \right)  &\quad a_{k-1}\leq x_1<a_k,\text{\quad for\;} k=1,\ldots,K, \\
 0 &\quad a_{K} \leq x_1 \leq 1. \\
\end{array}\right.
\end{equation}
An example of the graph of (\ref{eq:f_00}) along the coordinate $x_1$ is depicted in Figure~\ref{fig:bump1}.
\begin{figure}[!t]
  \centering
  \includegraphics[width=5in]{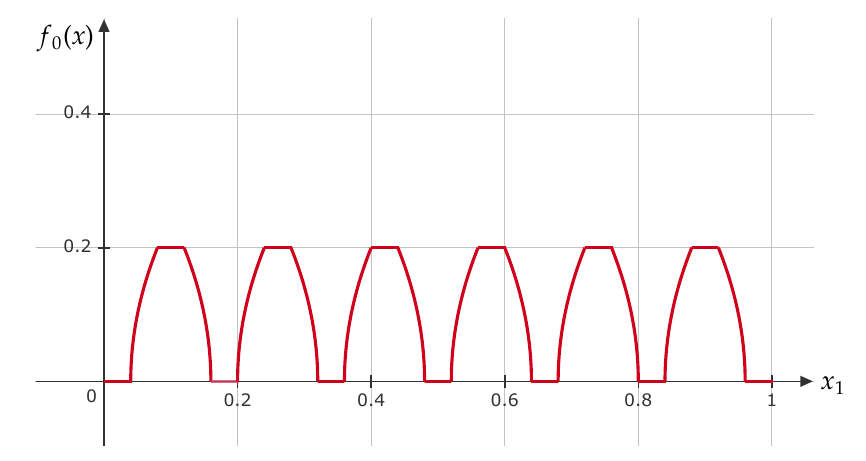}
  \caption{An example of $f_0$ defined in (\ref{eq:f_00}) with $C_{\beta} = 2$, $\beta = 1/2$, and $r=0.02$.}
  \label{fig:bump1}
\end{figure}

For the cases where $\beta>1$, consider the bump function on $\mathbb{R}$ defined by:
\begin{equation}\label{eq:bump000}
  \psi_0(x) \triangleq \exp\left[ -\frac{1}{1 - (x - 1)^2} \right] 1_{\{0 \leq x \leq 1\}}.
\end{equation}
We construct the base function $f_0:[0,1]^d \to \mathbb{R}$ as:
\begin{equation}\label{eq:bump}
f_0(x)=\left\{\begin{array}{ll}
 0  &\quad 0\leq x_1<2r, \\
 B_{\beta}\psi_0\left( \frac{x-2r_n}{1-4r_n} \right) &\quad 2r \leq  x_1 < 1 - 2r, \\
 B_{\beta}\exp(-1) &\quad 1 - 2r \leq  x_1 \leq  1,
\end{array}\right.
\end{equation}
where $B_{\beta} >0$ is a constant. An example of (\ref{eq:bump}) is illustrated in Figure~\ref{fig:bump2}.
\begin{figure}[!t]
  \centering
  \includegraphics[width=4.7in]{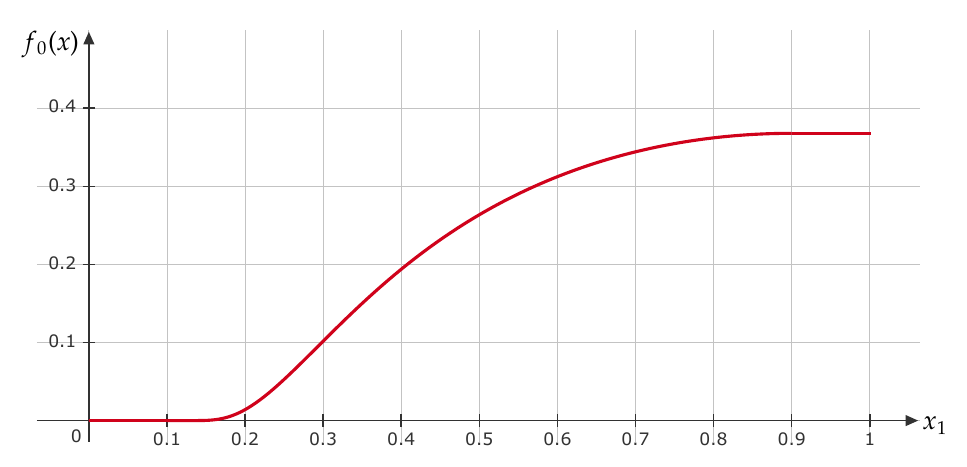}
  \caption{An example of $f_0$ defined in (\ref{eq:bump}) with $B_{\beta} = 1$ and $r=0.05$. }
  \label{fig:bump2}
\end{figure}

The following lemma demonstrates that the base regression functions (\ref{eq:f_00}) and (\ref{eq:bump}) possess two crucial properties. First, there exists a constant $C>0$ such that $G_{A,q}(f_0) \geq Cr^{1\wedge \beta}$. In addition, the flat-top and flat-bottom constructions ensure that any local maximum or minimum of $f_0$ occurs on a set of non-zero measure. As will be shown in the proof of Lemma~\ref{lem:mirror}, this property guarantees that any function $g$ close to $f_0$ in $L_q$-distance also exhibits a large deviation as $G_{A,q}(f_0)$.

\begin{lemma}\label{lem:mirror}

Consider the adversarial attack $A$ in $\mathcal{T}(r)$ with $0 \leq r < 1/8$. Suppose $A$ has a principal direction $v=(1,0,\ldots,0)^{\top}$. Then, we have $f_0 \in \mathcal{F}(\beta, C_{\beta}/2)$ if $B_{\beta}$ is chosen to be sufficiently small. Moreover, for any function $g$ satisfying $\| g - f_0\|_q \leq C_1 r^{1\wedge \beta}$, we have
  \begin{equation*}
    G_{A,q}(g) \geq C_2r^{1\wedge \beta},
  \end{equation*}
  where $1 \leq q < \infty$, and $C_1$ and $C_2$ are two positive constants with $0<C_2 < 2C_1$.
\end{lemma}

\begin{proof}[Proof of Lemma~\ref{lem:mirror}]

Define $\delta_0^u(x)\triangleq\argmax_{x' \in A(x)}f_0(x') $, $\delta_0^l(x)\triangleq\argmin_{x' \in A(x)}f_0(x') $, ${\delta}^u(x)\triangleq\argmax_{x' \in A(x)}g(x')$, and ${\delta}^l(x)\triangleq\argmin_{x' \in A(x)}g(x') $. Let $u_{0}(x)\triangleq\sup_{x'\in A(x)}f_{0}(x') = (f_0 \circ\delta_0^u)(x)$, $l_{0}(x)\triangleq\inf_{x'\in A(x)}f_{0}(x')=(f_0 \circ\delta_0^l)(x)$, $u(x)\triangleq\sup_{x'\in A(x)}g(x') = (g \circ\delta^u)(x)$, and $l(x)\triangleq\inf_{x'\in A(x)}g(x') = (g \circ\delta^l)(x)$, where $\circ$ is the function composition operator.

Based on the definitions in (\ref{eq:gunbal2}), $G_{A,q}(f_{0})$ can be upper bounded by
\begin{equation}\label{eq:youtou2}
\begin{split}
     G_{A,q}(f_{0})&= \frac{1}{2}\left\| u_0- l_0 \right\|_q=\frac{1}{2}\left\| u_0- l_0 - u + l +u -l \right\|_q \\
     & \leq \frac{1}{2}\left\| u_0- u\right\|_q+ \frac{1}{2}\left\| l_0 - l \right\|_q  + \frac{1}{2}\left\| u -l \right\|_q\\
     & = \frac{1}{2}\left\| u_0- u\right\|_q+ \frac{1}{2}\left\| l_0 - l \right\|_q  + G_{A,q}(g),
\end{split}
\end{equation}
where the inequality follows from the fact that $\| \cdot\|_q$ satisfies the triangle inequality. The first term in the last line of (\ref{eq:youtou2}) can be further upper bounded by
\begin{equation}\label{eq:hahah1}
\begin{split}
     \left\| u_0- u\right\|_q &= \left\| f_0 \circ\delta_0^u - g \circ\delta^u\right\|_q \\
     & = \left\| f_0 \circ\delta_0^u - g \circ \delta_0^u +g \circ \delta_0^u - g\circ \delta^u\right\|_q \\
     & \leq \left\| f_0 \circ\delta_0^u - g \circ \delta_0^u \right\|_q + \left\| g \circ \delta_0^u - g\circ \delta^u\right\|_q\\
     & \leq \left\| f_0 \circ\delta_0^u - g \circ \delta_0^u \right\|_q + 2G_{A,q}(g),
\end{split}
\end{equation}
where the first inequality follows from the triangle inequality, and the second inequality follows from the definition of $G_{A,q}(g)$. Similarly, the second term in the last line of (\ref{eq:youtou2}) can be upper bounded by
\begin{equation}\label{eq:hahah2}
\begin{split}
     \left\| l_0- l\right\|_q & \leq \left\| f_0 \circ \delta_0^l - g \circ  \delta_0^l \right\|_q + 2G_{A,q}(g).
\end{split}
\end{equation}
Note that the mappings $\delta_0^l,\delta_0^u$ may not be unique, and they will be uniquely defined later.

Inserting the bounds from (\ref{eq:hahah1})--(\ref{eq:hahah2}) into the right-hand side of (\ref{eq:youtou2}), we obtain
\begin{equation}\label{eq:cat3}
  G_{A,q}(f_{0}) \leq 3G_{A,q}(g) + \left\| f_0 \circ\delta_0^u - g \circ \delta_0^u \right\|_q  + \left\| f_0 \circ \delta_0^l - g \circ  \delta_0^l \right\|_q.
\end{equation}
The above inequality (\ref{eq:cat3}) is equivalent to
\begin{equation}\label{eq:cat3miao}
\begin{split}
G_{A,q}(f_{0})& \leq 3G_{A,q}(g) + \left\{\int_{[0,1]^d}\left|f_0(x) - g(x)\right|^q\delta_{0}^u\sharp\mu(x)dx\right\}^{\frac{1}{q}}\\
     &\qquad\qquad\quad\: + \left\{\int_{[0,1]^d}\left|f_0(x) - g(x)\right|^q\delta_{0}^l\sharp\mu(x)dx\right\}^{\frac{1}{q}},
\end{split}
\end{equation}
where $\delta_{0}^u\sharp\mu$ and $\delta_{0}^l\sharp\mu$ are the pushforward measures of the uniform measure $\mu$ on $[0,1]^d$ induced by the mappings $\delta_{0}^u$ and $\delta_{0}^l$, respectively. The task is now to construct lower bounds for $G_{A,q}(g)$ from (\ref{eq:cat3miao}). The proof will be divided into two scenarios.

\textsc{The case where $0 < \beta \leq 1$.} The verification of $f_0 \in \mathcal{F}(\beta, C_{\beta}/2)$ will be given in the Appendix. We prove the inequality
\begin{equation}\label{eq:tongxun}
  G_{A,q}(f_{0}) \leq 3G_{A,q}(g) + 4 \left\| f_0  - g \right\|_q
\end{equation}
for any $1 \leq q < \infty$. In view of (\ref{eq:cat3miao}), it suffices to demonstrate that
\begin{equation}\label{eq:tongxun1}
  \delta_{0}^u\sharp\mu(x) \leq 2 \text{\quad and \quad } \delta_{0}^l\sharp\mu(x) \leq 2
\end{equation}
for any $x \in [0,1]^d$. Recall that the adversary $A$, as defined in Definition~\ref{def:adversarial}, presents the attack along a principal direction $v=(1,0,\ldots,0)^{\top}$. Hence
\begin{equation}\label{eq:fuzhan1}
    \min\left\{ x_1':x' \in A(x) \right\} = (x_1 - \underline{c}r)\vee 0
 \end{equation}
 and
 \begin{equation}\label{eq:fuzhan2}
    \max\left\{ x_1':x' \in A(x) \right\} = (x_1 + \bar{c}r) \wedge 1,
 \end{equation}
 where $x_1'$ and $x_1$ are the first coordinates of $x'$ and $x_1$, respectively, and $0\leq \underline{c},\bar{c} \leq 1$ are constants as defined in Definition~\ref{def:adversarial} (ii).

Note that the value of $f_0$ varies only along the first coordinate. Given the monotonicity of $f_0$ with respect to $x_1$, we define $\delta_0^l$ and $\delta_0^u$ as follows:
\begin{equation}\label{eq:lilajia}
  \delta_{0}^l(x) \triangleq \left\{\begin{array}{ll}
 x  &\quad a_{k-1}\leq x_1<a_{k-1}+2r \\
 x - (\underline{c}r,0,\ldots,0)^{\top} &\quad a_{k-1}+2r \leq  x_1 <a_{k-1}+ 5r  \text{\qquad for\quad} k=1,\ldots,K,\\
 x + (\underline{c}r,0,\ldots,0)^{\top} &\quad a_{k-1}+5r \leq  x_1 < a_{k-1}+ 8r\\
 x &\quad a_{K} \leq  x_1 \leq 1,\\
\end{array}\right.
\end{equation}
and
\begin{equation}\label{eq:lilajia2}
  \delta_{0}^u(x) \triangleq \left\{\begin{array}{ll}
 x + (\bar{c}r,0,\ldots,0)^{\top} &\quad a_{k-1}\leq x_1<a_{k-1}+4r \\
 x &\quad a_{k-1}+4r \leq  x_1 < a_{k-1}+6r \text{\qquad for\quad} k=1,\ldots,K,\\
 x - (\bar{c}r,0,\ldots,0)^{\top} &\quad a_{k-1}+6r \leq  x_1 < a_{k-1}+ 8r\\
 x - (\bar{c}r,0,\ldots,0)^{\top} &\quad a_{K} \leq x_1 \leq 1.\\
\end{array}\right.
\end{equation}
Consequently, (\ref{eq:tongxun1}) holds for any $x \in [0,1]^d$. Therefore, from (\ref{eq:tongxun}), we obtain
\begin{equation}\label{eq:diaodai1}
  G_{A,q}(g) \geq \frac{G_{A,q}(f_{0})}{3} - \frac{4}{3} \left\|g- f_0 \right\|_q.
\end{equation}

The task is now to lower bound $G_{A,q}(f_{0})$. For $1 \leq q < \infty$, we get
\begin{equation}\label{eq:diaodai2}
  \begin{split}
     G_{A,q}(f_{0}) & =  \frac{1}{2}\left\{\int_{[0,1]^d} \left[\sup_{x' \in A(x)}f_0(x') - \inf_{x' \in A(x)}f_0(x')\right]^q dx\right\}^{\frac{1}{q}} \\
       & \geq \frac{1}{2}\left\{2  K \int_{0}^{(\bar{c}+\underline{c})r} x_1^{\beta q} dx_1\right\}^{\frac{1}{q}}\\
       & \geq \kappa_1 r^{\beta},
  \end{split}
\end{equation}
where the first inequality follows from (\ref{eq:fuzhan1})--(\ref{eq:fuzhan2}) and the periodicity of $f_0$, the second inequality uses $K = \lfloor 1/(8r) \rfloor$, and $\kappa_1$ is a constant depending on $\beta$, $q$, $\bar{c}$, and $\underline{c}$. Thus, by defining $C_1 = \kappa_1/8$ and $C_2 =\kappa_1/6$, we conclude that for any $1 \leq q < \infty$,
\begin{equation*}
  \begin{split}
     G_{A,q}(g) & \geq \frac{\kappa_1 r^{\beta}}{3}- \frac{4}{3} \left\|g- f_0 \right\|_q \geq \frac{\kappa_1 r^{\beta}}{6} = C_2 r^{\beta},
  \end{split}
\end{equation*}
where the first inequality follows from (\ref{eq:diaodai1}), and the second inequality holds under the condition $\| g - f_0\|_q \leq C_1 r^{1\wedge \beta}$.

\textsc{The case where $\beta>1$.} Due to the condition $1/2 < 1-4r \leq 1$ and the fact that $\psi_0 \in \mathbb{C}^{\infty}$, we can select an appropriate $B_{\beta}$ such that $f_0 \in \mathcal{F}(\beta,C_{\beta}/2)$ for any $\beta>1$ and $C_{\beta}>0$, as noted on page 93 of \cite{Tsybakov2009book}. In this case, the mappings $\delta_{0}^l$ and $\delta_{0}^u$ can be defined as
\begin{equation*}\label{eq:lilajia55}
  \delta_{0}^l(x) \triangleq \left\{\begin{array}{ll}
 x  &\quad 0\leq x_1 < 2r, \\
 x - (\underline{c}r,0,\ldots,0)^{\top} &\quad 2r \leq  x_1 <1,  \\
\end{array}\right.
\end{equation*}
and
\begin{equation*}\label{eq:lilajia225}
  \delta_{0}^u(x) \triangleq \left\{\begin{array}{ll}
 x + (\bar{c}r,0,\ldots,0)^{\top}&\quad 0\leq x_1 < 1-2r, \\
 x &\quad 1-2r \leq  x_1 <1.  \\
\end{array}\right.
\end{equation*}
It is evident that for any $x \in [0,1]^d$, the pushforward measures satisfy the bounds $\delta_{0}^u\sharp\mu(x) \leq 2$ and $\delta_{0}^l\sharp\mu(x) \leq 2$. Therefore, the inequality in (\ref{eq:diaodai1}) holds in this case as well.

Moreover, there exist constants $1/4 < c_1 < c_2  < 3/4$ and $c_3>0$ such that, for any $x$ with $c_1 \leq x_1 \leq c_2$, the partial derivative satisfies $\partial f_0(x)/\partial x_1 >c_3$. Consequently, for $1 \leq q < \infty$, we have
\begin{equation*}
  \begin{split}
     G_{A,q}(f_{0}) & = \frac{1}{2}\left\{\int_{[0,1]^d} \left[\sup_{x' \in A(x)}f_0(x') - \inf_{x' \in A(x)}f_0(x')\right]^q dx\right\}^{\frac{1}{q}} \\
       & \geq \frac{1}{2}\left\{ \int_{c_1}^{c_2}\left[c_3  \left(\underline{c} \wedge \bar{c} \right) r \right]^q dx_1 \right\}^{\frac{1}{q}} \triangleq \kappa_3 r.
  \end{split}
\end{equation*}
By applying the similar argument used in the case where $0<\beta \leq 1$ again, with $C_1$ and $C_2$ redefined, we can establish the results for the case where $\beta > 1$.

\end{proof}

\subsection{Localized Fano's method}

Given the base distribution constructed in Section~\ref{sec:base}, we can construct a set of distributions in $\mathcal{P}(\beta)$ where $\mathbb{P}_X$ and $\xi$ are identical to those of the base distribution, and the regression functions are localized around the base regression function $f_0$.

\begin{lemma}\label{lem:bump}

For any $1 \leq q < \infty$, there exists a subset of functions $\mathcal{S} \triangleq \{f_1,\ldots,f_{J_{n}} \} \subseteq \mathcal{F}(\beta, C_{\beta})$ and two constants $c_1,c_2>0$ such that for any $1 \leq j \neq k \leq J_{n}$,
\begin{description}
  \item[(i)] $\| f_j - f_k \|_q >  c_1\epsilon_{n}$,
  \item[(ii)] $\| f_j - f_0 \|_q \leq c_2\epsilon_{n}$,
  \item[(iii)] $\| f_j - f_0 \|_2 \leq \epsilon_{n}$,
  \item[(iv)] $\log J_{n} \geq n \epsilon_{n}^2 + 2\log 2$,
\end{description}
where $\epsilon_{n} \triangleq C(n^{-\frac{\beta}{2\beta+d}} \wedge 1)$ with some $C>0$.

\end{lemma}

The specific examples of the set $\mathcal{S}$ in Lemma~\ref{lem:bump} are provided in Section~\ref{sec:proof_lower} of the Appendix. These examples are constructed by perturbing the base regression function $f_0$ with a collection of bump functions. Using the local metric method from Section 7 of \cite{Yang1999Information} and Section C.3 of \cite{Wang2014Adaptive} (also see Theorem 2.5 of \cite{Tsybakov2009book}), we observe that the conditions (i), (iii), and (iv) in Lemma~\ref{lem:bump} lead to the following minimax argument.

\begin{lemma}\label{eq:dayuechen1}
For any $1 \leq q < \infty$, we have
\begin{equation*}
  \inf_{\hat{f}} \max_{f_j \in \mathcal{S}}\mathbb{P}_{\otimes^n} \left( \|\hat{f}- f_{j} \|_q \geq \frac{c_1 \epsilon_{n}}{2}  \right) \geq \frac{1}{2},
\end{equation*}
where the infimum is over all estimators, and $c_1$ and $\epsilon_{n}$ are defined in Lemma~\ref{lem:bump}.
\end{lemma}

Our proof of the minimax lower bounds for the adversarial risks combines Lemmas~\ref{lem:mirror}--\ref{eq:dayuechen1}, as will be explained in detail below.

\begin{proof}[Proof of Theorem~\ref{theo:lower}]
  To construct the lower bound, it suffices to focus on the set of distributions where $\mathbb{P}_{X}$ is the uniform distribution on $[0,1]^d$, $\xi$ is a Gaussian random variable with mean zero and unit variance conditioned on $X$, and the adversarial attack is $(v,r)$-SODA with the principal direction $v = (1,0,\ldots,0)^{\top}$. Under this reduction scheme, the adversarial minimax risk is lower bounded by
\begin{equation}\label{eq:tou1}
\begin{split}
     \inf_{\hat{f}}\sup_{\mathbb{P}_{(X,Y)} \in \mathcal{P}(\beta)}R_{A,q}(\hat{f}, f)& \geq \inf_{\hat{f}}\sup_{f \in \mathcal{F}(\beta,C_{\beta})} \mathbb{E}_{\otimes^n}\left[ d_{A,q}^q(\hat{f},f) \right]\\
     & \geq \varpi_n^q  \inf_{\hat{f}} \max_{f_j \in \mathcal{S}}\mathbb{P}_{\otimes^n}\left[ d_{A,q}(\hat{f},f_j) \geq \varpi_n \right],\\
\end{split}
\end{equation}
where the first inequality follows from the definition in (\ref{eq:gunbal1}), the second inequality follows from Markov's inequality and the fact that $\mathcal{S} \subseteq \mathcal{F}(\beta, C_{\beta})$ with $\mathcal{S}$ constructed as per Lemma~\ref{lem:bump}, and $\varpi_n$ is a parameter to be specified later.

Define $\hat{u}(x)\triangleq\sup_{x'\in A(x)}\hat{f}(x')$ and $\hat{l}(x)\triangleq\inf_{x'\in A(x)}\hat{f}(x')$. Note that
\begin{equation}\label{eq:yanxi_1}
  \begin{split}
     d_{A,q}^q(\hat{f},f_j) & = \int_{[0,1]^d}\left[\sup_{x' \in A(x)}\left| \hat{f}(x') - f_{j}(x) \right|\right]^qdx \\
       & =  \int_{[0,1]^d}\left[\left| f_{j}(x) - \frac{\hat{u}(x)+\hat{l}(x)}{2} \right|+ \frac{\hat{u}(x)-\hat{l}(x)}{2}\right]^qdx\\
       & \geq \int_{[0,1]^d}\left[ \frac{\hat{u}(x)-\hat{l}(x)}{2}\right]^qdx = G_{A,q}^q(\hat{f}),
  \end{split}
\end{equation}
where the last equality follows from the definition in (\ref{eq:gunbal2}).
Based on the inequality (\ref{eq:yanxi_1}) and the condition $\{ x \} \subseteq A(x)$, we obtain
\begin{equation}\label{eq:dhaizai2}
  d_{A,q}(\hat{f},f_j) - d_{q}(\hat{f},f_j) \geq \left[G_{A,q}(\hat{f}) - \| \hat{f} - f_j \|_q \right]_+.
\end{equation}
Combining (\ref{eq:tou1}) with (\ref{eq:dhaizai2}), we have
\begin{equation}\label{eq:diaodaihai1}
  \begin{split}
    & \inf_{\hat{f}}\sup_{\mathbb{P}_{(X,Y)} \in \mathcal{P}(\beta)}R_{A,q}(\hat{f}, f)  \geq \varpi_n^q  \inf_{\hat{f}} \max_{f_j \in \mathcal{S}}\mathbb{P}_{\otimes^n}\left[ d_{A,q}(\hat{f},f_j) \geq \varpi_n \right] \\
       & \geq \varpi_n^q  \inf_{\hat{f}} \max_{f_j \in \mathcal{S}}\mathbb{P}_{\otimes^n}\left\{ \| \hat{f} - f_j \|_q +\left[G_{A,q}(\hat{f}) - \| \hat{f} - f_j \|_q \right]_+ \geq \varpi_n  \right\}.
  \end{split}
\end{equation}
We now divide our proof into three cases.

\textsc{The case where $ r^{1 \wedge \beta} \lesssim n^{-\frac{\beta}{2\beta+d}}$.} In this scenario, the perturbation magnitude is relatively small compared to the standard rate $n^{-\frac{\beta}{2\beta+d}}$. Given that $R_{A,q}(\hat{f},f) \geq R_{q}(\hat{f},f)$, the adversarial minimax risk is lower bounded by
\begin{equation*}
  \begin{split}
       \inf_{\hat{f}}\sup_{\mathbb{P}_{(X,Y)} \in \mathcal{P}(\beta)}R_{A,q}(\hat{f}, f) &\geq  \inf_{\hat{f}}\sup_{\mathbb{P}_{(X,Y)} \in \mathcal{P}(\beta)}R_{q}(\hat{f}, f)  \gtrsim n^{-\frac{q\beta}{2\beta+d}},
  \end{split}
\end{equation*}
where the lower bound for the standard minimax rate can be found in the literature such as \cite{Stone1982global, Gyorfi2002book, Tsybakov2009book}.

\textsc{The case where $ n^{-\frac{\beta}{2\beta+d}} =  o(r^{1 \wedge \beta})$ and $r < 1/8$.} We now proceed to lower bound the last term in equation (\ref{eq:diaodaihai1}). For this case, we set $\varpi_n$ as follows:
\begin{equation}\label{eq:yanxi3}
  \varpi_n = \frac{C_2r^{1 \wedge \beta}}{2} + \frac{c_1 \epsilon_{n}}{4},
\end{equation}
where $C_2$ is the constant from Lemma~\ref{lem:mirror}, and $c_1$ and $\epsilon_{n}$ are defined in Lemma~\ref{lem:bump}.

If the event $\{ \| \hat{f} - f_j \|_q +[G_{A,q}(\hat{f}) - \| \hat{f} - f_j \|_q ]_+ < \varpi_n \}$ occurs, it immediately follows that
\begin{equation}\label{eq:dhaizai3}
  \| \hat{f} - f_j \|_q < \frac{C_2r^{1 \wedge \beta}}{2} + \frac{c_1 \epsilon_{n}}{4}.
\end{equation}
Additionally, when $n$ is sufficiently large, we have
\begin{equation}\label{eq:yanxi2}
\begin{split}
   \|\hat{f}- f_{0} \|_q & \leq \|\hat{f}- f_{j} \|_q + \|f_{j}- f_{0} \|_q \leq \frac{C_2r^{1 \wedge \beta}}{2} + \frac{c_1 \epsilon_{n}}{4} + c_2\epsilon_{n} \\
     & = \frac{C_2 r^{1 \wedge \beta}}{2} + \frac{(4c_2+c_1) \epsilon_{n}}{4} < C_1r^{1 \wedge \beta},
\end{split}
\end{equation}
where the first inequality follows from the triangle inequality, the second from equation (\ref{eq:dhaizai3}) and Lemma~\ref{lem:bump}, and the last from the condition $C_2/2 < C_1$ in Lemma~\ref{lem:mirror}, together with $\epsilon_{n} \asymp n^{-\frac{\beta}{2\beta+d}} =  o(r^{1 \wedge \beta})$. By Lemma~\ref{lem:mirror}, equation (\ref{eq:yanxi2}) implies that $G_{A,q}(\hat{f}) \geq C_2r^{1 \wedge \beta}$. Therefore, under the event $\{ \| \hat{f} - f_j \|_q +[G_{A,q}(\hat{f}) - \| \hat{f} - f_j \|_q ]_+ < \varpi_n \}$, we have
\begin{equation*}
  \begin{split}
  \varpi_n= \frac{C_2r^{1 \wedge \beta}}{2} + \frac{c_1 \epsilon_{n}}{4} & > \|\hat{f}- f_{j} \|_q + \left[G_{A,q}(\hat{f}) - \| \hat{f} - f_j \|_q \right]_+\\
        & \geq \|\hat{f}- f_{j} \|_q + C_2r^{1 \wedge \beta} - \frac{C_2r^{1 \wedge \beta}}{2} - \frac{c_1 \epsilon_{n}}{4}\\
       & = \|\hat{f}- f_{j} \|_q +  \frac{C_2r^{1 \wedge \beta}}{2} - \frac{c_1 \epsilon_{n}}{4} ,
  \end{split}
\end{equation*}
which implies that $\|\hat{f}- f_{j} \|_q < c_1 \epsilon_{n}/2$. Consequently, from (\ref{eq:diaodaihai1}), we obtain
\begin{equation}\label{eq:xintai1}
  \begin{split}
     &\inf_{\hat{f}}\sup_{\mathbb{P}_{(X,Y)} \in \mathcal{P}(\beta)}R_{A,q}(\hat{f}, f) \\
      & \geq \varpi_n^q  \inf_{\hat{f}} \max_{f_j \in \mathcal{S}}\mathbb{P}_{\otimes^n}\left\{ \| \hat{f} - f_j \|_q +\left[G_{A,q}(\hat{f}) - \| \hat{f} - f_j \|_q \right]_+ \geq \varpi_n  \right\} \\
       & \geq \varpi_n^q \inf_{\hat{f}} \max_{f_j \in \mathcal{S}}\mathbb{P}_{\otimes^n} \left( \|\hat{f}- f_{j} \|_q \geq \frac{c_1 \epsilon_{n}}{2}  \right) \gtrsim r^{q(1 \wedge \beta)} ,
  \end{split}
\end{equation}
where the last step follows from Lemma~\ref{eq:dayuechen1} and the definition of $\varpi_n$ in (\ref{eq:yanxi3}).

\textsc{The case where $1/8 \leq r \leq 1 $.} When $r \geq 1/8$, for any $A(x)$, we only need to construct an $A'(x)$ such that $A'(x) \subseteq A(x)$ and $A'$ satisfies Definition~\ref{def:adversarial} with $r= 1/16$. It follows directly from Definition~\ref{def:adversarial} that such an $A'$ must exist. By combining the relation $R_{A,q}(\hat{f}, f) \geq R_{A',q}(\hat{f}, f) $ with (\ref{eq:xintai1}), we obtain the following bound:
\begin{equation*}
  \inf_{\hat{f}}\sup_{\mathbb{P}_{(X,Y)} \in \mathcal{P}(\beta)}R_{A,q}(\hat{f}, f) \gtrsim 1.
\end{equation*}
\end{proof}

\section{Discussion}\label{sec:discuss}

In this paper, we investigate the minimax rate for nonparametric regression estimation under adversarial risks, focusing on cases where the true regression function has H\"{o}lder smoothness and the adversarial attack meets a mild geometric requirement. Our findings have important practical and theoretical implications. Theoretically, we identify two key factors that collectively determine the adversarial minimax risk: the attack's magnitude (measured by the diameter of the perturbation set) and the smoothness level of the target regression function class. For an adversary, our results identify the critical magnitude needed for a successful attack. From a defender's perspective, the derived minimax rates offer a benchmark to evaluate any estimator in presence of the adversarial attacks. We propose several nonparametric estimators and establish their optimality and adaptivity under the future $X$-attacks.

The results and techniques developed in this paper may serve as a starting point for theoretical analysis of other general adversarial learning problems. For example, beyond the regression setting, it is valuable to assess the intrinsic difficulty of adversarial learning for tasks such as density estimation, classification, and quantile regression. Additionally, evaluating whether other popular defense strategies, like adversarial training \citep{Madry2018}, achieve the minimax optimality or not remains an important open question. Furthermore, understanding other intriguing phenomena in artificial intelligence from a minimax perspective, such as the trade-off between overparameterization and adversarial robustness \citep{Hassani2024}, or the issue of hallucination \citep{zhang2023siren} in generative artificial intelligence applications, would be highly valuable.

\newpage
\appendix

\section*{Appendix}\label{appendix}
\addcontentsline{toc}{section}{Appendix}
\renewcommand{\thesection}{\Alph{section}}
\numberwithin{equation}{section}

In this Appendix, we provide a detailed discussion of the adversarial risks defined in (\ref{eq:adv_risk}) and present complete proofs of the minimax upper bounds stated in Theorems~\ref{theo:upper}--\ref{theo:adaptive1} of the paper ``Adversarial learning for nonparametric regression: Minimax rate and adaptive estimation''. In addition, we establish the auxiliary technical lemmas and complete the proofs of the minimax lower bounds given in Theorem~\ref{theo:lower}.

\section{Adversarial risk}\label{sec:other_measure}

In this section, we provide further insight into the underlying assumption of the adversarial risk defined in (\ref{eq:adv_risk}). We then establish the connections between the adversarial risk (\ref{eq:adv_risk}) and the TRADES risk (\ref{eq:risk_zhang}) introduced in \cite{Zhang2019}. Finally, we review other commonly used adversarial measures and discuss their relationships and differences with (\ref{eq:adv_risk}).

\subsection{Further discussion on the adversarial risk (\ref{eq:adv_risk})}\label{subsec:other_1}

It is important to point out that (\ref{eq:adv_risk}) implies a sensible assumption about the adversary's behavior in adversarial learning. To see this, consider $q=2$. In this case, (\ref{eq:adv_risk}) takes the form
\begin{equation}\label{eq:offer11}
  R_{A,2}(\hat{f}, f) = \mathbb{E}\sup_{X' \in A(X)}\mathbb{E}_{Y|X}|\hat{f}(X') - Y|^2 - \mathbb{E}|f(X) - Y|^2,
\end{equation}
where $\mathbb{E}|f(X) - Y|^2 = \inf_{g}\mathbb{E}|g(X) - Y|^2 $ serves as a benchmark that only depends on $\mathbb{P}_{(X,Y)}$. From (\ref{eq:offer11}), we see that $R_{A,2}(\hat{f}, f)$ measures the adversarial prediction performance by perturbing $X'$ to maximize $\mathbb{E}_{Y|X}|\hat{f}(X') - Y|^2$. This implies that the adversary operates without prior knowledge of the future $Y$ value, and the worst-case perturbation corresponds to the one that maximizes the conditional expected prediction error given $X$.

Therefore, our formulation of adversarial learning addresses a practically prevalent scenario where neither the adversary nor the defender has access to the future $Y$ value. From a practical perspective, constructing worst possible attacks in our setting requires the adversary to have access to both $\hat{f}$ and $f$ for choosing the most damaging perturbation $x' \in A(x)$. If $f$ is agnostic to the adversary but additional information is available to train a model for the attack, the worst-case perturbation may still be estimated using an appropriate estimator of $f$. Otherwise, (\ref{eq:offer11}) serves more like a robustness measure and describes the possibly worst prediction degradation due to the future $X$-attacks.

\subsection{Connections between (\ref{eq:adv_risk}) and (\ref{eq:risk_zhang})}\label{subsec:other_2}

As stated in Section~\ref{sec:discuss_adve_risk} of the main text, when $\{x\} \subseteq A(x)$, $\phi(x,y) = (x-y)^2$, and $t = 1$, we have
\begin{equation}\label{eq:yueye1}
  \frac{\big[T_{\phi,t}(\hat{f}) - \mathbb{E}|f(X) - Y|^2\big]}{5} \leq  R_{A,2}(\hat{f}, f) \leq 2 \big[T_{\phi,t}(\hat{f}) - \mathbb{E}|f(X) - Y|^2\big].
\end{equation}
We begin by proving the upper bound in (\ref{eq:yueye1}). We have
\begin{equation*}
  \begin{split}
    R_{A,2}(\hat{f}, f)& = \mathbb{E}\sup_{X' \in A(X)}\left| \hat{f}(X') - f(X) \right|^2 =  \mathbb{E}\sup_{X' \in A(X)}\left| \hat{f}(X') - \hat{f}(X) + \hat{f}(X) - f(X) \right|^2\\
       & \leq 2\mathbb{E}\sup_{X' \in A(X)}\left| \hat{f}(X') - \hat{f}(X)\right|^2 + 2\mathbb{E}\left|\hat{f}(X) - f(X) \right|^2\\
       & = 2\mathbb{E}\sup_{X' \in A(X)}\left| \hat{f}(X') - \hat{f}(X)\right|^2 + 2\mathbb{E}\left|\hat{f}(X) - Y \right|^2 - 2\mathbb{E}\left| Y -  f(X) \right|^2\\
       & = 2 T_{\phi,t}(\hat{f}) - 2\mathbb{E}|f(X) - Y|^2.
  \end{split}
\end{equation*}
For the lower bound, observe:
\begin{equation}\label{eq:yueye2}
  \begin{split}
     R_{A,2}(\hat{f}, f) & = \mathbb{E}\sup_{X' \in A(X)}\left| \hat{f}(X') - f(X) \right|^2 \geq \mathbb{E}\left| \hat{f}(X) - f(X) \right|^2 \\
       & = \mathbb{E}\left|\hat{f}(X) - Y \right|^2 - \mathbb{E}\left|   f(X) -Y\right|^2,
  \end{split}
\end{equation}
where the inequality follows from $X \in A(X)$. Additionally, we have
\begin{equation}\label{eq:yueye3}
  \begin{split}
     R_{A,2}(\hat{f}, f) & = \mathbb{E}\sup_{X' \in A(X)}\left| \hat{f}(X') - f(X) \right|^2 \\
       & = \mathbb{E}\left[\left| \frac{\sup_{X' \in A(X)}\hat{f}(X')+\inf_{X' \in A(X)}\hat{f}(X')}{2} -f(X) \right|\right.\\
      & \left. \quad\quad+ \frac{\sup_{X' \in A(X)}\hat{f}(X') - \inf_{X' \in A(X)}\hat{f}(X')}{2} \right]^2\\
      & \geq \frac{1}{4}\mathbb{E} \left[ \sup_{X' \in A(X)}\hat{f}(X') - \inf_{X' \in A(X)}\hat{f}(X') \right]^2 \geq \frac{1}{4}\mathbb{E}\sup_{X' \in A(X)}\left| \hat{f}(X') - \hat{f}(X)\right|^2.
  \end{split}
\end{equation}
Combining (\ref{eq:yueye2}) and (\ref{eq:yueye3}), we obtain
\begin{equation*}
  \begin{split}
     R_{A,2}(\hat{f}, f) & \geq \frac{1}{5}\mathbb{E}\left|\hat{f}(X) - Y \right|^2 - \frac{1}{5}\mathbb{E}\left|   f(X) -Y\right|^2 + \frac{4}{5}\times \frac{1}{4}\times\mathbb{E}\sup_{X' \in A(X)}\left| \hat{f}(X') - \hat{f}(X)\right|^2\\
       & = \frac{1}{5} T_{\phi,t}(\hat{f}) - \frac{1}{5}\mathbb{E}|f(X) - Y|^2.
  \end{split}
\end{equation*}
This completes the proof of (\ref{eq:yueye1}).

\subsection{Other performance measures}\label{subsec:other_3}

Before delving into the discussion of other adversarial measures, it is helpful to recall that there are two typical approaches to evaluate the performance of statistical procedures.

\subsubsection{Two standard measures}

The first approach is grounded in the framework of statistical decision theory \citep[see, e.g.,][]{le2012asymptotic, Yang1999Information}, where the goal is to estimate a parameter $\theta$ or a functional $\theta(\mathbb{P})$ of the unknown distribution $\mathbb{P}$. The performance of an estimator $\hat{\theta}$ is measured by $d(\hat{\theta}, \theta)$, where $d$ is a semi-metric (see Chapter~2 of \cite{Tsybakov2009book} and Chapter~15 of \cite{Wainwright2019high} for a complete formulation). For instance, in regression problems, the target $\theta$ is the true regression function $f$, and $d$ is typically chosen as the $L_q$-norm, which leads to the standard $L_q$-risk discussed in Section~\ref{subsect:adv-risk}.

The second approach measures the performance of a statistical procedure from a prediction perspective. In the standard supervised learning setup without an adversary, the performance of an estimator $\hat{f}$ is evaluated using the loss $\ell(\hat{f},(X,Y))$, where $(X,Y)$ denotes a pair of test data independent of $\hat{f}$. This framework has been extensively studied \citep[see, e.g.,][]{vapnik2013nature, geer2000empirical, pollard1990empirical} and has found wide applications in modern statistical learning \citep[see, e.g.,][]{Bartlett2006, Koltchinskii2006, Lecue2009} and machine learning theory \cite[see, e.g.,][]{Goodfellow-et-al-2016, shalev2014understanding}. The advantages of adopting the prediction-based measures were discussed in \cite{geisser1993predictive}.

In most of the literature, both approaches are considered reasonable, and in some special cases, such as regression \citep{Gyorfi2002book} and classification \citep{devroye2013probabilistic, Yang1999classification, Audibert2007fast}, explicit links between the two approaches can be established when specific loss functions are adopted.

\subsubsection{Other adversarial measures}

Now, consider the adversarial setting. Our definition of the adversarial risk in (\ref{eq:adv_risk}) can be seen as an application of the statistical decision framework to adversarial learning, which focuses on the robust estimation of $f$. Given a loss function $\ell$, adversarial robustness have also been measured from the prediction view point.
\begin{itemize}

\item \emph{Output-aware $X$-attacks}: In this scenario, the output remains unchanged ($Y' = Y$), while the input vector is perturbed such that $X' \in B_p(X,r)$ with the observation of the true output $Y$ prior to selecting $X'$ \citep[see, e.g.,][]{Szegedy2014Intriguing, Goodfellow2014Explaining, Carlini2017Towards, Madry2018}. The adversarial loss is defined as
    \begin{equation}\label{eq:rucgun2}
      L_{OA}(\hat{f}) \triangleq \mathbb{E}_{(X,Y)} \left[\sup_{X' \in B_p(X,r)}\ell(\hat{f},(X',Y)) \right].
    \end{equation}
    The corresponding risk is $R_{OA}(\hat{f}) \triangleq \mathbb{E}_{\otimes^n}L_{OA}(\hat{f})$.

\item \emph{Marginal distribution perturbation}: In this situation, an adversary shifts the distribution of $X$, making it different from the original marginal distribution $\mathbb{P}_{X}$ \citep[see, e.g.,][]{staib2017distributionally, Pydi2020, Mehrabi21, Duchi2021, Duchi2023}. Specifically, the perturbed input $X'$ follows a distribution $Q \in \mathcal{P}_A$, where $\mathcal{P}_A$ denotes an adversarial distribution set around the marginal distribution $\mathbb{P}_{X}$. Fix the conditional distribution $\mathcal{P}_{Y'|X'} = \mathcal{P}_{Y|X}$. The distributionally adversarial loss of $\hat{f}$ is then defined as
    \begin{equation}\label{eq:rucgun1}
      L_{DA}(\hat{f}) \triangleq \sup_{Q \in \mathcal{P}_A}\mathbb{E}_{X'\sim Q}\mathbb{E}_{Y'|X'}\left[ \ell(\hat{f},(X',Y')) \right],
    \end{equation}
    where $\mathcal{P}_A$ is typically chosen as a set of distributions close to $\mathbb{P}_{X}$ in terms of the Wasserstein distance. The corresponding distributionally adversarial risk is given by $R_{DA}(\hat{f}) \triangleq \mathbb{E}_{\otimes^n}L_{DA}(\hat{f})$.

\end{itemize}

These two prediction-based adversarial models rely on significant different assumptions about the adversarial behavior. In the output-aware $X$-attack model, the perturbation is chosen after obtaining both the estimator $\hat{f}$ and the test data point $(X, Y)$. In the distribution shift model, the adversary perturbs the distribution with the knowledge of the true data distribution $\mathbb{P}_{(X, Y)}$. Both assumptions are reasonable for adversarial robustness evaluation in different setting and are applicable across various domains. \cite{staib2017distributionally, Pydi2020, Mehrabi21} have proved that for any output-aware $X$-attack with $B_p(X,r)$, one can construct a neighborhood set $\mathcal{P}_A$ based on the Wasserstein distance defined with respect to the $\ell_p$ distance, such that $L_{OA}(\hat{f})$ is upper bounded by $L_{DA}(\hat{f})$.

The assumptions underlying our adversarial measure (\ref{eq:adv_risk}) strike a balance between (\ref{eq:rucgun1}) and (\ref{eq:rucgun2}). Specifically, to construct a worst-case perturbation, the adversary is not required to have access to the future $Y$ value, but it is assumed that the adversary has partial knowledge of the conditional distribution $\mathbb{P}_{Y|X}$ (i.e., the regression function $f$) or additional information to approximate the true regression function. Furthermore, it is worth noting that when $\ell(\hat{f},(X,Y)) = |\hat{f}(X) - Y|^q$ with $1 \leq q < \infty$, the following order relationship holds between (\ref{eq:offer11}) and (\ref{eq:rucgun2}):
\begin{equation}\label{eq:qinghe1}
  R_{OA}(\hat{f}) \lesssim R_{A,q}(\hat{f}, f) + \mathbb{E}|f(X) - Y|^q.
\end{equation}
This inequality implies that convergence in (\ref{eq:offer11}) also guarantees robustness in $R_{OA}(\hat{f})$, up to an additional term that accounts for the random error. A similar bound to (\ref{eq:qinghe1}) has also been established between $R_{OA}(\hat{f})$ and $T_{\phi,t}(\hat{f})$ for the $0$-$1$ loss \citep{Zhang2019}.

\section{Proof of Theorem~\ref{theo:upper}}\label{sec:proof_upper}

\subsection{Notation}\label{sec:proof_notation_1}

We begin by introducing some necessary notation for the proofs. Using the set of discretization points $\Lambda_{M}$, we partition the domain $\Omega = [0, 1]^d$ into $M^d$ disjoint subsets, denoted by $\Omega_1, \ldots, \Omega_{M^d}$. Points $x$ and $y$ are assigned to the same subset if and only if $u_x = u_y$, where $u_x$ refers to the point in $\Lambda_{M}$ closest to $x$ in the $\ell_2$-norm. For each $k = 1, \ldots, M^d$, we define $u_k \triangleq u_x$ for any $x \in \Omega_k$, i.e., the discretization point corresponding to subset $\Omega_k$. For any function $f \in \mathcal{F}(\beta, C_{\beta})$, let $f_u(\cdot)$ denote the Taylor polynomial of degree $\ell \wedge \lfloor \beta \rfloor$ centered at $u$, as defined in (\ref{eq:taylor_poly}). Additionally, let $\lambda_{\min}(A)$ and $\lambda_{\max}(A)$ denote the smallest and largest eigenvalues of a square matrix $A$, respectively. For any event $\mathcal{A}$, its complement is denoted as $\mathcal{A}^c$.

Recall the matrix $B_{uh}$ and the vector $a_{uh}$ from (\ref{eq:B}) and (\ref{eq:a}), respectively. We now define the matrix $Q_{uh} \in \mathbb{R}^{N_{\ell,d}\times n}$ as
\begin{equation}\label{eq:Q_uh}
  Q_{uh} \triangleq \frac{1}{\sqrt{nh^d}}\left[
        U\left(\frac{X_1 - u}{h}\right)\sqrt{K\left( \frac{X_1-u}{h} \right)},\cdots , U\left(\frac{X_n - u}{h}\right)\sqrt{K\left( \frac{X_n-u}{h} \right)}
      \right].
\end{equation}
Consequently, the matrix $B_{uh}$ can be decomposed as $B_{uh} = Q_{uh} Q_{uh}^{\top}$.

Next, we define three vectors $e_{uh}, \eta_{uh}, t_{uh} \in \mathbb{R}^n$ as follows:
\begin{equation}\label{eq:e_uh}
  e_{uh} \triangleq \left[ \xi_1\sqrt{K\left( \frac{X_1-u}{h} \right)},\ldots, \xi_n\sqrt{K\left( \frac{X_n-u}{h} \right)} \right]^{\top},
\end{equation}
\begin{equation}\label{eq:eta_uh}
  \eta_{uh} \triangleq \left[ f(X_1)\sqrt{K\left( \frac{X_1-u}{h} \right)},\ldots, f(X_n)\sqrt{K\left( \frac{X_n-u}{h} \right)}\right]^{\top},
\end{equation}
\begin{equation}\label{eq:t_uh}
  t_{uh} \triangleq \left[ f_u(X_1)\sqrt{K\left( \frac{X_1-u}{h} \right)},\ldots, f_u(X_n)\sqrt{K\left( \frac{X_n-u}{h} \right)}\right]^{\top},
\end{equation}
respectively. It follows that the vector $a_{uh}$, as defined in (\ref{eq:a}), can be expressed as
\begin{equation}\label{eq:japan1}
  a_{uh} = \frac{1}{\sqrt{nh^d}} Q_{uh} \left( \eta_{uh} + e_{uh} \right).
\end{equation}
Recall the Taylor polynomial defined by (\ref{eq:taylor_poly}). Then, the vector $t_{uh}$ can be written as
\begin{equation}\label{eq:ljinv1}
  t_{uh} = \sqrt{nh^d} Q_{uh}^{\top} c_{uh},
\end{equation}
where
\begin{equation}\label{eq:ljinv2}
  c_{uh} \triangleq \left( \tilde{D}^sf(u)h^{|s|} \right)_{0 \leq |s| \leq \ell}.
\end{equation}

Define a diagonal matrix in $\mathbb{R}^{N_{\ell,d}\times N_{\ell,d}}$ as follows:
\begin{equation}\label{eq:D}
  D_{uh} \triangleq \diag\left[\left(\frac{1}{nh^d}\sum_{i=1}^{n} \frac{1}{(s!)^2}\left(\frac{X_i - u}{h} \right)^{2s} K\left( \frac{X_i-u}{h} \right)\right)_{0 \leq |s| \leq \ell}\right].
\end{equation}
Note that the diagonal entries of $D_{uh}$ are identical to those of the matrix $B_{uh}$. Similar to (\ref{eq:B_tilde}), we define
\begin{equation}\label{eq:tilde_D}
  \tilde{D}_{uh} \triangleq D_{uh} + \tau I_{N_{\ell,d}} 1_{\{ \lambda_{\min}(D_{uh}) < \tau \}}.
\end{equation}
Note that the matrixes (\ref{eq:D})--(\ref{eq:tilde_D}) have been defined in Section~\ref{sec:adaptive} with $u = u_x$. Here, we provide two definitions for general $u$.

\subsection{Technical lemma}

Before proving the theorem, we first state several useful lemmas. The first lemma provides an upper bound on the maximum deviation of the functions in $\mathcal{F}(\beta,C_{\beta})$ under the adversarial attack.

\begin{lemma}\label{lem:lip}
  Let $f \in \mathcal{F}(\beta,C_{\beta})$ and $A \in \mathcal{T}^{\prime}(r)$. Then for all $x \in [0,1]^d$, it holds that
  \begin{equation}\label{eq:lip}
    \sup_{x' \in A(x)}\left| f(x') - f(x) \right| \lesssim r^{1 \wedge \beta}.
  \end{equation}
\end{lemma}

\begin{proof}[Proof of Lemma~\ref{lem:lip}]
  We first prove that for any $x,u \in [0,1]^d$,
\begin{equation}\label{eq:poly}
\begin{split}
   \left|f(x) - f_{u}(x)\right| & \leq C_{\beta}\sum_{|s|=\lfloor \beta \rfloor} \frac{1}{s !}|x-u|^s \left\| x - u \right\|^{\beta - \lfloor \beta \rfloor}, \\
\end{split}
\end{equation}
where $f_{u}$ is the Taylor polynomial of degree $\lfloor \beta \rfloor$ at $u$.

The proof of the argument (\ref{eq:poly}) with $d=1$ can be found in Lemma~11.1 of \cite{Gyorfi2002book}. Now we complete the proof for general $d$. When $0 < \beta \leq 1$, we have $\lfloor \beta \rfloor = 0$ and $f_{u}(x) = f(u)$. Then (\ref{eq:poly}) follows from the definition (\ref{eq:holder_smooth}). In the case $\beta > 1$, we have
\begin{equation*}
  \begin{split}
       \left|f(x) - f_{u}(x)\right| & = \left|f(x) - \sum_{|s| \leq \lfloor \beta \rfloor-1} \frac{D^s f(u)}{s !}(x - u)^s - \sum_{|s| = \lfloor \beta \rfloor} \frac{D^s f(u)}{s !}(x - u)^s\right|\\
       & = \left| \sum_{|s|=\lfloor \beta \rfloor} \frac{\lfloor \beta \rfloor}{s !}(x-u)^s \int_0^1(1-z)^{\lfloor \beta \rfloor-1}D^s f[u+z(x-u)] d z \right.\\
       &\quad\left.- \sum_{|s| = \lfloor \beta \rfloor} \frac{\lfloor \beta \rfloor }{s !}(x-u)^s\int_0^1(1-z)^{\lfloor \beta \rfloor-1} D^s f(u) d z \right|\\
       & = \left| \sum_{|s|=\lfloor \beta \rfloor} \frac{\lfloor \beta \rfloor}{s !}(x-u)^s \int_0^1(1-z)^{\lfloor \beta \rfloor-1}\left\{D^s f[u+z(x-u)] - D^s f(u) \right\} d z \right|\\
       & \leq C_{\beta}\sum_{|s|=\lfloor \beta \rfloor} \frac{1}{s !}|x-u|^s \left\| x - u \right\|^{\beta - \lfloor \beta \rfloor},\\
  \end{split}
\end{equation*}
where the second equality follows from the integral form of the Taylor series remainder, and the last inequality follows from the definition of $\mathcal{F}(\beta,C_{\beta})$. Thus, (\ref{eq:poly}) is established.

By the multinomial theorem and the relation $\| x \|_1 \leq \sqrt{d}\| x \|_2$ for $x \in \mathbb{R}^d$, we see the right side of (\ref{eq:poly}) is upper bounded by
\begin{equation}\label{eq:poly2}
\begin{split}
   \left|f(x) - f_{u}(x)\right| & \leq \frac{C_{\beta}}{\lfloor \beta \rfloor !}\| x - u \|_1^{\lfloor \beta \rfloor}\left\| x - u \right\|^{\beta - \lfloor \beta \rfloor}\\
    &\leq \frac{C_{\beta}d^{\frac{\lfloor \beta \rfloor}{2}}}{\lfloor \beta \rfloor !}\left\| x - u \right\|^{\beta} \lesssim \left\| x - u \right\|^{\beta},
\end{split}
\end{equation}
where the last inequality follows from the assumption that $d$ is not related to $n$.

Note that
\begin{equation}\label{eq:key_0}
  \sup_{x' \in A(x)}\left| f(x') - f(x) \right| \leq \sup_{\|x' - x\| \leq r} \left| f(x') - f(x) \right|.
\end{equation}
To show (\ref{eq:lip}), it suffices to upper bound $\sup_{\|x' - x\| \leq r} | f(x') - f(x) |$. Define $\bar{x} = (x'+ x)/2$. Then we have
\begin{equation}\label{eq:key_1}
\begin{split}
  \sup_{\|x' - x\| \leq r}\left|  f(x') - f(x) \right| & = \sup_{\|x' - x\| \leq r}\left| f(x') - f_{\bar{x}}(x') + f_{\bar{x}}(x') - f_{\bar{x}}(x) + f_{\bar{x}}(x) -f(x)\right|\\
     & \leq \sup_{\| x' - x \| \leq r}\left| f(x') - f_{\bar{x}}(x') \right| + \sup_{\| x' - x \| \leq r}\left|  f_{\bar{x}}(x') - f_{\bar{x}}(x) \right|\\
     &\quad+ \sup_{\| x' - x \| \leq r}\left|  f_{\bar{x}}(x) -f(x)\right|.
\end{split}
\end{equation}
The first term at the right side of (\ref{eq:key_1}) is upper bounded by
\begin{equation}\label{eq:key_2}
\begin{split}
     \sup_{\| x' - x \| \leq r}\left| f(x') - f_{\bar{x}}(x') \right| & \lesssim \sup_{\| x' - x \| \leq r} \| x' - \bar{x} \|^{\beta} = \frac{1}{2^{\beta}}\sup_{\| x' - x \| \leq r}\| x' - x \|^{\beta} \lesssim r^{\beta}, \\
\end{split}
\end{equation}
where the first inequality follows from (\ref{eq:poly2}), and the equality follows from the definition of $\bar{x}$. Based on the same technique in (\ref{eq:key_2}), we see the third term of (\ref{eq:key_1}) is upper bounded by
\begin{equation}\label{eq:key_4}
  \sup_{\| x' - x \| \leq r}\left|  f_{\bar{x}}(x) -f(x)\right| \lesssim r^{\beta}.
\end{equation}
The task is now to find the upper bound on the second term of (\ref{eq:key_1}). We have
\begin{equation}\label{eq:key_3}
  \begin{split}
       \sup_{\| x' - x \| \leq r}\left|  f_{\bar{x}}(x') - f_{\bar{x}}(x) \right| &= \sup_{\| x' - x \| \leq r}\left|  \sum_{|s| \leq \lfloor \beta \rfloor} \frac{D^s f(\bar{x})}{s !}\left[(x' - \bar{x})^s - (x - \bar{x})^s\right] \right|\\
       & = \sup_{\| x' - x \| \leq r}\left|  \sum_{|s| \leq \lfloor \beta \rfloor} \frac{D^s f(\bar{x})}{2^{|s|} s !}\left[1+(-1)^{|s|+1}\right](x' - x)^s \right|\\
       & \leq C_{\beta} \sup_{\| x' - x \| \leq r} \sum_{1 \leq k \leq \lfloor \beta \rfloor} \sum_{|s| = k} \frac{2}{s !}|x' - x|^s\\
       & \leq  2C_{\beta} \sup_{\| x' - x \| \leq r} \sum_{1 \leq k \leq \lfloor \beta \rfloor} \frac{d^{\frac{k}{2}}}{k!}\| x' - x \|^{k}\\
       & \lesssim r,
  \end{split}
\end{equation}
where the first equality follows from (\ref{eq:taylor_poly}), the second equality is due to the definition of $\bar{x}$, the first inequality follows from $|D^s f(\bar{x})| \leq C_{\beta}$ and $1/2^{|s|} \leq 1$ for $0 \leq |s| \leq \lfloor \beta \rfloor$, and the second inequality follows the similar reasoning as in (\ref{eq:poly2}).
Combining (\ref{eq:key_0}) with (\ref{eq:key_1})--(\ref{eq:key_3}), we obtain (\ref{eq:lip}).

\end{proof}

The next lemma shows that within each subset $\Omega_k$, the functions in $\mathcal{F}(\beta,C_{\beta})$ can be approximated by its Taylor polynomial expanded at the corresponding discretization point $u_k$.

\begin{lemma}\label{lem:taylor}
  For any $f \in \mathcal{F}(\beta,C_{\beta})$ and $\Omega_k, k=1,\ldots, M^d$, we have
  \begin{equation}\label{eq:discre_error}
    \sup_{x' \in \Omega_k} \left| f_{u_k}(x') - f(x') \right| \lesssim \frac{1}{M^{\beta}}.
  \end{equation}
\end{lemma}

\begin{proof}[Proof of Lemma~\ref{lem:taylor}]

Based on the construction of the discretization points in (\ref{eq:discrete_set}), it is clear that $\sup_{x' \in \Omega_k}\|x' - u_k \| = \sqrt{d}/(2M) $. We conclude from (\ref{eq:poly2}) that
\begin{equation*}
  \sup_{x' \in \Omega_k} \left| f_{u_k}(x') - f(x') \right| \lesssim \sup_{x' \in \Omega_k} \| x' - u_k \|^{\beta} \lesssim \frac{1}{M^{\beta}},
\end{equation*}
and hence (\ref{eq:discre_error}).

\end{proof}

Let us denote by $\mathfrak{X}_n \triangleq \sigma(X_1,\ldots,X_n)$ the $\sigma$-algebra generated by the design. In the following lemma, we define several events on $\mathfrak{X}_n$ and analyze the probability that they occur. For $k=1,\ldots,M^d$, recall that the diagonal matrix $D_{u_kh}$ is defined in (\ref{eq:D}) with $u = u_k$.

\begin{lemma}\label{lem:event}
  For $k=1,\ldots,M^d$, define $\mathcal{A}_{u_kh} \triangleq \{ 0 <n_{u_kh} \leq c_1 n h^d\}$, $\mathcal{B}_{u_kh} \triangleq \{\lambda_{\min}(B_{u_kh}) \geq c_2 \}$, and $\mathcal{D}_{u_kh} \triangleq \{ c_3 \leq  \lambda_{\min}(D_{u_kh}) \leq  \lambda_{\max}(D_{u_kh}) \leq c_4\}$, where $c_i,i=1,\ldots,4$, are four positive constants that will be specified in the proof. Assume there exist two constants $\gamma_1,\gamma_2>0$ such that $M/n^{\gamma_1} \to 0$ and $nh^d/n^{\gamma_2} \to \infty$. Then when $n$ is large enough, there must exist an event $\mathcal{C}_h$ on $\mathfrak{X}_n$ such that
  \begin{equation*}
    \mathcal{C}_h \subseteq \bigcap_{k=1}^{M^d}\left(\mathcal{A}_{u_kh} \cap \mathcal{B}_{u_kh} \cap \mathcal{D}_{u_kh}\right)
  \end{equation*}
  and
  \begin{equation*}
    \mathbb{P}_{\otimes^ n}\left( \mathcal{C}_h \right) \geq 1 - \exp\left( - C nh^d \right),
  \end{equation*}
  where $C>0$ is a constant not related to $n$.
\end{lemma}

\begin{proof}[Proof of Lemma~\ref{lem:event}]

\underline{\textsc{Studying the event} $\mathcal{A}_{u_kh}$.} Define
$$
n_{u_kh}' \triangleq \mathrm{Card}[\{X_i:X_i \in B(u_k, \Delta h) \}]=\sum_{i=1}^{n}1_{\{ \| X_i - u_k \| \leq \Delta h \}},
$$
where $0 < \Delta < 1$ is the constant appearing in Assumption~\ref{ass:kernel}. Recalling the definition $n_{u_kh} = \mathrm{Card}[\{X_i:X_i \in B(u_k, h) \}]$, we thus have $n_{u_kh}' \leq n_{u_kh}$, and
\begin{equation}\label{eq:luo1}
  \left\{ n_{u_kh}' >0 \right\} \subseteq \left\{ n_{u_kh} >0 \right\}.
\end{equation}

Note that the random variable $1_{\{ \| X_i - u_k \| \leq \Delta h \}} - \mathbb{P}(\| X - u_k \| \leq \Delta h)$ has mean $0$, and is bounded since $|1_{\{ \| X_i - u_k \| \leq \Delta h \}} - \mathbb{P}(\| X - u_k \| \leq \Delta h) | \leq 1$. And its variance is upper bounded by $\mathbb{P}(\| X - u_k \| \leq \Delta h) [1 - \mathbb{P}(\| X - u_k \| \leq \Delta h)] \leq \mathbb{P}(\| X - u_k \| \leq \Delta h)$. Based on the Bernstein inequality, for any $\varepsilon >0$, we have
\begin{equation}\label{eq:luoluo2}
\begin{split}
   &\mathbb{P}_{\otimes^ n}\left[ \left|  n_{u_kh}' - n \mathbb{P}(\| X - u_k \| \leq \Delta h) \right|\geq\varepsilon  \right] \\
    &= \mathbb{P}_{\otimes^ n}\left\{ \left| \sum_{i=1}^{n}\left[ 1_{\{ \| X_i - u_k \| \leq \Delta h \}} - \mathbb{P}(\| X - u_k \| \leq \Delta h) \right] \right| \geq\varepsilon  \right\} \\
   & \leq 2\exp\left( -\frac{\varepsilon^2/2}{n\mathbb{P}(\| X - u_k \| \leq \Delta h) [1 - \mathbb{P}(\| X - u_k \| \leq \Delta h)]+ \varepsilon/3} \right)\\
     & \leq 2\exp\left( -\frac{\varepsilon^2/2}{n\mathbb{P}\left(\| X - u_k \| \leq \Delta h\right)+ \varepsilon/3} \right)\\
     &  \leq 2\exp\left( -\frac{\varepsilon^2/2}{n\kappa_1 h^d+ \varepsilon/3} \right),
\end{split}
\end{equation}
where the last inequality follows from
\begin{equation}\label{eq:luoluo0}
\begin{split}
   \mathbb{P}\left(\| X - u_k \| \leq \Delta h \right) & = \int_{\Omega \cap B(u_k,\Delta h)}\mu(x)dx  \leq \mu_{\max}\lambda\left[ B(u_k,\Delta h)\right]\\
   & = \mu_{\max}v_d (\Delta h)^d \triangleq \kappa_1 h^d,
\end{split}
\end{equation}
and here $v_d = \lambda[B(0,1)]$ denotes the volume of the unit ball. In addition, recall that the support $\Omega = [0,1]^d$ satisfies the $(c_\mu, r_\mu)$-regularity condition introduced in \cite{Audibert2007fast}, i.e., there exist constants $c_\mu, r_\mu > 0$ such that
\begin{equation}\label{eq:support_regular}
  \lambda\left[\Omega \cap B(x, r)\right] \geq c_\mu \lambda\left[B(x, r)\right], \quad \forall 0<r \leq r_\mu, \forall x \in \Omega.
\end{equation}
We have
\begin{equation}\label{eq:luoluo}
\begin{split}
   \mathbb{P}\left(\| X - u_k \| \leq \Delta h \right) & = \int_{\Omega \cap B(u_k,\Delta h)}\mu(x)dx \geq \mu_{\min}\lambda\left[ \Omega \cap B(u_k,\Delta h)\right] \\
     & \geq \mu_{\min}c_{\mu} \lambda\left[ B(u_k,\Delta h)\right] = \mu_{\min}c_{\mu} v_d (\Delta h)^d \triangleq \kappa_2 h^d,
\end{split}
\end{equation}
where the second inequality follows from (\ref{eq:support_regular}) and the condition $h \leq r_{\mu}/\Delta$.

We now consider the event $\{  n_{u_kh}' \geq \kappa_2 n h^d/2\}$. From (\ref{eq:luo1}), we see $\{  n_{u_kh}' \geq \kappa_2 n h^d/2\} \subseteq \{ n_{u_kh} >0 \}$. In addition, the probability of the complement of this event is upper bounded by
\begin{equation}\label{eq:luoxin1}
\begin{split}
    \mathbb{P}_{\otimes^ n}\left( n_{u_kh}'\leq \frac{\kappa_2 n h^d}{2}\right)& = \mathbb{P}_{\otimes^ n}\left(   n_{u_kh}' \leq \kappa_2 n h^d - \frac{\kappa_2 n h^d}{2} \right)  \\
     & \leq \mathbb{P}_{\otimes^ n}\left[   n_{u_kh}' \leq  n \mathbb{P}(\| X - u_k \| \leq \Delta h) - \frac{\kappa_2 n h^d}{2} \right]\\
     & \leq \mathbb{P}_{\otimes^ n}\left[ \left|  n_{u_kh}' - n \mathbb{P}(\| X - u_k \| \leq \Delta h) \right|\geq \frac{\kappa_2 n h^d}{2}  \right]\\
       & \leq 2\exp\left( - Cnh^d \right),
\end{split}
\end{equation}
where the second inequality follows from (\ref{eq:luoluo}), and the last inequality follows from (\ref{eq:luoluo2}).

Similar to (\ref{eq:luoluo0}) and (\ref{eq:luoluo}), we see
\begin{equation}\label{eq:wo2}
  \kappa_4 h^d \leq \mathbb{P}\left(\| X - u_k \| \leq  h \right) \leq \kappa_3 h^d.
\end{equation}
Define $c_1 \triangleq \kappa_3 + \kappa_4/2$ and consider the event $\{ n_{u_kh} \geq c_1 n h^d \} $. Using the Bernstein inequality again, we obtain
\begin{equation}\label{eq:wo1}
\begin{split}
   \mathbb{P}_{\otimes^ n}\left( n_{u_kh} \geq c_1 n h^d \right) & = \mathbb{P}_{\otimes^ n}\left[ n_{u_kh} \geq  \left( \kappa_3 + \frac{\kappa_4}{2} \right)n h^d \right]  \\
     & \leq \mathbb{P}_{\otimes^ n}\left[   n_{u_kh} \geq  n \mathbb{P}(\| X - u_k \| \leq  h) + \frac{\kappa_4 n h^d}{2} \right]\\
     & \leq \mathbb{P}_{\otimes^ n}\left[ \left|  n_{u_kh} - n \mathbb{P}(\| X - u_k \| \leq  h) \right|\geq \frac{\kappa_4 n h^d}{2}  \right]\\
       & \leq 2\exp\left( - Cnh^d \right),
\end{split}
\end{equation}
where the first inequality follows from (\ref{eq:wo2}), and the last inequality can be obtained by the similar reasoning as that in (\ref{eq:luoluo2}).

Therefore, we define
\begin{equation*}
  \mathcal{A}_{u_kh}' \triangleq \left\{ n_{u_kh}'\geq \frac{\kappa_2 n h^d}{2} \right\} \cap \left\{ n_{u_kh} \leq c_1 n h^d \right\}.
\end{equation*}
We see $\mathcal{A}_{u_kh}' \subseteq \mathcal{A}_{u_kh}$, and from (\ref{eq:luoxin1}) and (\ref{eq:wo1}), we have
\begin{equation}\label{eq:wo3}
  \mathbb{P}_{\otimes^ n}\left( \mathcal{A}_{u_kh}'^c\right) \leq 4\exp\left( - Cnh^d \right).
\end{equation}

\underline{\textsc{Studying the event} $\mathcal{B}_{u_kh}$.} In this part, we bound the probability of $\mathcal{B}_{u_kh}$ based on the technique in \cite{Audibert2007fast}. Define a matrix $\bar{B}_{u_kh} \triangleq \mathbb{E}(B_{u_kh}) \in \mathbb{R}^{N_{\ell,d} \times N_{\ell,d}}$, where $\bar{B}_{u_kh}(s,r)$ and $B_{u_kh}(s,r)$ denote $sr$-th component of the matrixes $\bar{B}_{u_kh}$ and $B_{u_kh}$, respectively. Indeed, we have
\begin{equation}\label{eq:zhuzhu3}
\begin{split}
   B_{u_kh}(s,r) & = \frac{1}{nh^d}\sum_{i=1}^{n}\frac{(X_i - u_k)^s}{h^{|s|}s!}\frac{(X_i - u_k)^r}{h^{|r|}r!}K\left( \frac{X_i-u_k}{h}\right) \\
     & = \frac{1}{nh^d} \sum_{i=1}^{n} \frac{1}{s!r!}\left( \frac{X_i - u_k}{h} \right)^{s+r}K\left( \frac{X_i-u_k}{h}\right),
\end{split}
\end{equation}
and
\begin{equation}\label{eq:zhuzhu1}
  \begin{split}
     \bar{B}_{u_kh}(s,r) & =\mathbb{E}\left[\frac{1}{h^d}\frac{1}{s!r!}\left( \frac{X - u_k}{h} \right)^{s+r}K\left( \frac{X-u_k}{h}\right)\right]\\
       &= \int_{\Omega}\frac{1}{h^d}\frac{1}{s!r!}\left( \frac{x - u_k}{h} \right)^{s+r}K\left( \frac{x-u_k}{h}\right) \mu(x) dx \\
     & =\int_{\mathbb{R}^d} \frac{ t^{s+r}}{s!r!}K(t)\mu(u_k+ht)dt.
  \end{split}
\end{equation}
Then, the smallest eigenvalue $\lambda_{\min}(B_{u_kh})$ can be lower bounded by
\begin{equation}\label{eq:zhuzhu4}
  \begin{split}
     \lambda_{\min}(B_{u_kh}) & = \inf_{\| w \|=1}w^{\top}B_{u_kh}w = \inf_{\| w \|=1}w^{\top}\left(\bar{B}_{u_kh}+B_{u_kh} -\bar{B}_{u_kh} \right)w\\
       & \geq  \inf_{\| w \|=1}w^{\top}\bar{B}_{u_kh}w + \inf_{\| w \|=1}w^{\top}\left(B_{u_kh} -\bar{B}_{u_kh} \right)w\\
       & \geq \inf_{\| w \|=1}w^{\top}\bar{B}_{u_kh}w - \sum_{0 \leq |s|,|r| \leq \ell}\left| B_{u_kh}(s,r) -  \bar{B}_{u_kh}(s,r) \right|.
  \end{split}
\end{equation}

We first show that $\inf_{\| w \|=1}w^{\top}\bar{B}_{u_kh}w$ is lower bounded by a positive constant. Recalling $\Omega$ is the support of $\mu(x)$, we define $\Psi_{u_k}\triangleq \{ t\in \mathbb{R}^d: \| t \| \leq \Delta, u_k+ht \in \Omega \}$. On account of Assumption~\ref{ass:kernel} and Definition~\ref{def:strong_density}, we have $K(t) \geq K_{\min}$ and $\mu(u_k+ht) \geq \mu_{\min}$ when $t \in \Psi_{u_k}$, respectively, and
\begin{equation*}
  \lambda(\Psi_{u_k}) = h^{-d}\lambda\left[ B(u_k,\Delta h) \cap \Omega \right] \geq c_{\mu} h^{-d}  \lambda\left[ B(u_k,\Delta h) \right] = c_{\mu}v_d \Delta^d>0.
\end{equation*}
Thus, we see
\begin{equation*}
  \begin{split}
     w^{\top}\bar{B}_{u_kh}w & =  \int_{\mathbb{R}^d}\left( \sum_{0 \leq |s| \leq \ell}w_s\frac{t^s}{s!} \right)^2K(t)\mu(u_k+ht)dt \geq  K_{\min} \mu_{\min} \int_{\Psi_{u_k}}\left( \sum_{0 \leq |s| \leq \ell}w_s\frac{t^s}{s!} \right)^2dt.
  \end{split}
\end{equation*}
Now we argue that
\begin{equation*}
  \inf_{\| w \|=1} \int_{\Psi_{u_k}}\left( \sum_{0 \leq |s| \leq \ell}w_s\frac{t^s}{s!} \right)^2dt
\end{equation*}
must be positive. Suppose its minimum is attained at $w^*$. Then the relation
\begin{equation}\label{eq:zhuzhu2}
   \sum_{0 \leq |s| \leq \ell}w_s^*\frac{t^s}{s!} = 0
\end{equation}
holds for almost all $t \in \Psi_{u_k}$. As observed above, $\Psi_{u_k}$ is a set with positive measure. On the other hand, the expression in (\ref{eq:zhuzhu2}) is a polynomial in $t$. Thus, it is impossible there are infinite number of roots for a polynomial, which results in a contradiction. Thus, there must exist a constant $c_2>0$ such that
\begin{equation}\label{eq:zhuzhu5}
  \inf_{\| w \|=1}w^{\top}\bar{B}_{u_kh}w \geq 2c_2
\end{equation}
for all $k=1,\ldots,M^d$.

In view of (\ref{eq:zhuzhu4}), to establish the event $\mathcal{B}_{u_kh}$, we need to control the term $| B_{u_kh}(s,r) -  \bar{B}_{u_kh}(s,r) |$ for $0 \leq |s|,|r| \leq \ell$. Combining (\ref{eq:zhuzhu3}) with (\ref{eq:zhuzhu1}), we see
\begin{equation}\label{eq:xiao1}
\begin{split}
     & B_{u_kh}(s,r)-\bar{B}_{u_kh}(s,r) \\
     & = \frac{1}{nh^d} \sum_{i=1}^{n} \frac{1}{s!r!} \left( \frac{X_i - u_k}{h} \right)^{s+r}K\left( \frac{X_i-u_k}{h}\right) - \int_{\mathbb{R}^d} \frac{t^{s}t^{r}}{s!r!}K(t)\mu(u_k+ht)dt\\
     & =  \sum_{i=1}^{n}T_i,
\end{split}
\end{equation}
where
\begin{equation*}
  T_i \triangleq \frac{1}{nh^d}\frac{1}{s!r!}\left( \frac{X_i - u_k}{h} \right)^{s+r}K\left( \frac{X_i-u_k}{h}\right) -  \frac{1}{n}\int_{\mathbb{R}^d} \frac{t^{s}t^{r}}{s!r!}K(t)\mu(u_k+ht)dt.
\end{equation*}
Note that $\mathbb{E}T_i = 0$, and
\begin{equation*}
\begin{split}
|T_i|& \leq 2\left|\frac{1}{nh^d}\frac{1}{s!r!} \left( \frac{X_i - u_k}{h} \right)^{s+r}K\left( \frac{X_i-u_k}{h}\right) \right|\\
     & \leq \frac{2 K_{\max}}{nh^d} \left|t^{s+r}\right|  = \frac{2K_{\max}}{nh^d} \left|t_1\right|^{s_1+r_1}\cdots \left|t_d\right|^{s_d+r_d}\\
     & \leq \frac{2K_{\max}}{nh^d}\triangleq\frac{\kappa_3}{nh^d},
\end{split}
\end{equation*}
where $t$ is a $d$-dimensional vector satisfying $\| t\| \leq 1$ due to the support of $K$. The variance of $T_i$ is upper bounded by
\begin{equation*}
  \begin{split}
     \Var(T_i) & \leq \frac{1}{n^2h^{2d}}\left(\frac{1}{s!r!}\right)^2\mathbb{E}\left( \frac{X_i - u_k}{h} \right)^{2s+2r}K^2\left( \frac{X_i-u_k}{h}\right) \\
       & \leq \frac{1}{n^2 h^d} \int_{\mathbb{R}^d}t^{2s+2r}K^2(t)\mu(u_k+h t)dt\\
       & \leq \frac{K_{\max}^2 \mu_{\max}}{n^2 h^d}  \int_{B(0,1)}t^{2s+2r} dt \leq \frac{K_{\max}^2 \mu_{\max} v_d}{n^2 h^d } \triangleq  \frac{\kappa_4}{n^2 h^d },
  \end{split}
\end{equation*}
where the second inequality follows from that the kernel function is supported on $B(0,1)$, and the last inequality is due to $|t^{2s+2r}| \leq 1$ when $t \in B(0,1)$. Using (\ref{eq:xiao1}) and the Bernstein inequality, we have for any $\varepsilon>0$,
\begin{equation}\label{eq:xiao3}
  \begin{split}
&\mathbb{P}_{\otimes^ n}\left( \left| B_{u_kh}(s,r)-\bar{B}_{u_kh}(s,r) \right| \geq \varepsilon \right)  = \mathbb{P}_{\otimes^ n}\left( \left| \sum_{i=1}^{n}T_i \right| \geq \varepsilon \right) \leq 2 \exp\left( - \frac{\frac{1}{2}\varepsilon^2}{n\mathbb{E}T_i^2+\frac{1}{3}\frac{\kappa_3}{nh^d} \varepsilon} \right)\\
        &\leq 2 \exp\left( - \frac{\frac{1}{2}\varepsilon^2}{ \frac{n\kappa_4}{n^2 h^d }+\frac{1}{3}\frac{\kappa_3}{nh^d} \varepsilon} \right) = 2 \exp\left( - \frac{\varepsilon^2}{\frac{1}{nh^d}(2\kappa_4+2\kappa_3\varepsilon/3)} \right).
  \end{split}
\end{equation}

We now define an event
\begin{equation}\label{eq:event2}
  \mathcal{B}_{u_kh}' \triangleq \bigcap_{0 \leq |s|,|r| \leq \ell} \left \{ \left| B_{u_kh}(s,r)-\bar{B}_{u_kh}(s,r) \right| \leq \frac{c_2}{N_{\ell,d}^2} \right\},
\end{equation}
where $c_2$ is the constant given in (\ref{eq:zhuzhu5}). From (\ref{eq:xiao3}), we obtain
\begin{equation}\label{eq:xiao4}
  \begin{split}
     \mathbb{P}_{\otimes^ n}\left( \mathcal{B}_{u_kh}'^c \right) & \leq \sum_{0 \leq |s|,|r| \leq \ell}\mathbb{P}_{\otimes^ n}\left( \left| B_{u_kh}(s,r)-\bar{B}_{u_kh}(s,r) \right| \geq \frac{c_2}{N_{\ell,d}^2} \right) \\
       & \leq 2N_{\ell,d}^2\exp\left( - C nh^d  \right).
  \end{split}
\end{equation}
Furthermore, when the event $\mathcal{B}_{u_kh}'$ holds, it follows from (\ref{eq:zhuzhu4}) and (\ref{eq:zhuzhu5}) that
\begin{equation}\label{eq:xiaoxiao1}
  \lambda_{\min}(B_{u_kh}) \geq 2c_2 - N_{\ell,d}^2 \frac{c_2}{N_{\ell,d}^2} = c_2 >0,
\end{equation}
which implies $\mathcal{B}_{u_kh}' \subseteq \mathcal{B}_{u_kh}$.

\underline{\textsc{Studying the event} $\mathcal{D}_{u_kh}$.} Recall that the entries of $D_{u_kh}$ are just the diagonal entries of $B_{u_kh}$. Thus, the analysis of the event $\mathcal{D}_{u_kh}$ is similar to that of $\mathcal{B}_{u_kh}$. Define $\bar{D}_{u_kh}= \mathbb{E}(D_{u_kh})$. There exists a $c_3>0$ such that $\inf_{\| w \|=1}w^{\top}\bar{D}_{u_kh}w \geq 2c_3$, and
\begin{equation*}
  \begin{split}
     \lambda_{\min}(D_{u_kh}) & \geq \inf_{\| w \|=1}w^{\top}\bar{D}_{u_kh}w - \sum_{0 \leq |s| \leq \ell}\left| D_{u_kh}(s,s) -  \bar{D}_{u_kh}(s,s) \right| \\
       & \geq c_3 >0
  \end{split}
\end{equation*}
under the event
\begin{equation}\label{eq:event3}
  \mathcal{D}_{u_kh}' \triangleq \bigcap_{0 \leq |s| \leq \ell} \left \{ \left| D_{u_kh}(s,s)-\bar{D}_{u_kh}(s,s) \right| \leq \frac{c_3}{N_{\ell,d}} \right\}.
\end{equation}
In addition, if the event $\mathcal{D}_{u_kh}'$ holds, we also see that
\begin{equation*}
  \begin{split}
     \lambda_{\max}(D_{u_kh}) & \leq \sup_{\| w \|=1}w^{\top}\bar{D}_{u_kh}w + \sum_{0 \leq |s| \leq \ell}\left| D_{u_kh}(s,s) -  \bar{D}_{u_kh}(s,s) \right| \\
       & \leq K_{\max}\mu_{\max}v_d + c_3 \triangleq c_4,
  \end{split}
\end{equation*}
where the second inequality follows from the fact that the largest eigenvalues of the diagonal matrix $\bar{D}_{u_kh}$ equal to its largest entries. Therefore, we obtain $\mathcal{D}_{u_kh}' \subseteq \mathcal{D}_{u_kh}$. As in the proof of (\ref{eq:xiao4}), we have
\begin{equation}\label{eq:xiao44}
  \mathbb{P}_{\otimes^ n}\left( \mathcal{D}_{u_kh}'^c \right) \leq 2N_{\ell,d}\exp\left( - C nh^d  \right).
\end{equation}

\underline{\textsc{Constructing the event} $\mathcal{C}_h$.} Define the event $\mathcal{C}_h$ as
\begin{equation*}
  \mathcal{C}_h \triangleq \bigcap_{k=1}^{M^d}(\mathcal{A}_{u_kh}' \cap \mathcal{B}_{u_kh}' \cap \mathcal{D}_{u_kh}' ).
\end{equation*}
Combining the results from the above three parts, we obtain
\begin{equation*}
  \mathcal{C}_h \subseteq \bigcap_{k=1}^{M^d}(\mathcal{A}_{u_kh} \cap \mathcal{B}_{u_kh} \cap \mathcal{D}_{u_kh}).
\end{equation*}
And from (\ref{eq:luoxin1}), (\ref{eq:xiao4}), and (\ref{eq:xiao44}), we see that when $n$ is large enough,
\begin{equation*}
\begin{split}
    \mathbb{P}_{\otimes^ n}\left(\mathcal{C}_h\right) &\geq 1- \sum_{k=1}^{M^d}\left[ \mathbb{P}_{\otimes^ n}(\mathcal{A}_{u_kh}'^c) + \mathbb{P}_{\otimes^ n}(\mathcal{B}_{u_kh}'^c) + \mathbb{P}_{\otimes^ n}(\mathcal{D}_{u_kh}'^c)\right]  \geq 1 - C_1M^d\exp\left( - C_2 nh^d \right) \\
     & \geq 1 - \exp\left( - C nh^d \right),
\end{split}
\end{equation*}
where the last inequality follows from that $M$ increases to $\infty$ no faster than a polynomial rate, and $nh^d$ increases to $\infty$ no slower than a polynomial rate. Thus, the results in Lemma~\ref{lem:event} are established.

\end{proof}

The following lemma provides an upper bound for the expected maximum of powers of sub-Gaussian variables. Its proof applies Lemma 2 from \cite{peng2023optimality} specifically to sub-Gaussian variables.

\begin{lemma}\label{lem:sub_Gaussian_bound}
  Let $\xi_1,\ldots,\xi_M$ be zero-mean $\sigma$-sub-Gaussian random variables. Then, we have
  \begin{equation*}
    \mathbb{E} \left( \max_{1 \leq m \leq M} \left| \xi_m \right|^q \right) \leq C \left( \log M \right)^{\frac{q}{2}},
  \end{equation*}
  where $C>0$ is an absolute constant that depends on $\sigma$ and $q$.
\end{lemma}

\begin{proof}
  Recalling the definition of $\sigma$-sub-Gaussian variables (see, e.g., Definition~\ref{def:error_moment}), we have $\mathbb{E} \exp(\lambda \xi_m)  \leq \exp(\lambda^2\sigma^2/2)$ for all $\lambda \in \mathbb{R}$. It follows that
  \begin{equation}\label{eq:huiyigun2}
    \mathbb{E} \exp\left(\lambda \left|\xi_m \right| \right) \leq \mathbb{E} \exp\left(\lambda \xi_m\right) + \mathbb{E} \exp\left(-\lambda \xi_m\right) \leq 2\exp\left(\frac{\lambda^2\sigma^2}{2}\right).
  \end{equation}
  Using similar reasoning to that in Lemma~2 of \cite{peng2023optimality}, we obtain
  \begin{equation}\label{eq:huiyigun1}
    \begin{split}
       \left[\mathbb{E} \left( \max_{1 \leq m \leq M} \left| \xi_m \right|^q \right) \right]^{\frac{1}{q}} & \leq  \frac{\log\left[ 2M\exp\left(\lambda^2\sigma^2/2\right) +\exp\left( q - 1 \right) \right]}{\lambda} \\
         & \leq C \left( \frac{\log M}{\lambda} + \lambda \right).
    \end{split}
  \end{equation}
  Optimizing the right side of (\ref{eq:huiyigun1}) yields $\lambda^* = \sqrt{\log M}$. Thus, the upper bound (\ref{eq:huiyigun2}) holds.
\end{proof}

\subsection{Proof of the upper bound for $1 \leq q < \infty$}\label{sec:proof:buxiang2}

We are now in a position to prove Theorem~\ref{theo:upper}. Note that the adversarial risk (\ref{eq:adv_risk}) can be written as
\begin{equation}\label{eq:0}
  R_{A,q}(\hat{f}_{\mathrm{PP}}, f) = \mathbb{E}\sup_{X' \in A(X)}\left| \hat{f}_{\mathrm{PP}}(X') - f(X) \right|^q = \mathbb{E}_X \mathbb{E}_{\otimes^n}\sup_{X' \in A(X)}\left| \hat{f}_{\mathrm{PP}}(X') - f(X) \right|^q.
\end{equation}
Given that the marginal distribution $\mathbb{P}_X$ satisfies the bounded density assumption as outlined in Definition~\ref{def:strong_density}, constructing upper bounds for (\ref{eq:0}) can be reduced to establishing upper bounds for the pointwise risk. For a given $x \in \Omega$, we have
\begin{equation}\label{eq:1}
\begin{split}
\mathbb{E}_{\otimes^n}\sup_{x' \in A(x)}\left| \hat{f}_{\mathrm{PP}}(x') - f(x) \right|^q & \leq 2^{q-1} \mathbb{E}_{\otimes^n}\sup_{x' \in A(x)}\left| \hat{f}_{\mathrm{PP}}(x') - f(x') \right|^q \\
&+ 2^{q-1} \sup_{x' \in A(x)}\left| f(x') - f(x) \right|^q.
\end{split}
\end{equation}
According to Lemma~\ref{lem:lip}, the second term on the right-hand side of (\ref{eq:1}) is upper bounded by
\begin{equation}\label{eq:2}
  2^{q-1}\sup_{x' \in A(x)}\left| f(x') - f(x) \right|^q \lesssim  r^{q(1 \wedge \beta)}.
\end{equation}
Our primary objective in the subsequent analysis is to upper bound $\mathbb{E}_{\otimes^n}\sup_{x' \in A(x)}| \hat{f}_{\mathrm{PP}}(x') - f(x') |^q$ uniformly for each $x \in \Omega$.

Recall that the magnitude of the adversarial attack is characterized by
$$
r = \sup_{x \in [0,1]^d,x' \in A(x)}\| x' - x \|.
$$
Based on the construction of the adversarial sets in Definition~\ref{def:adversarial}, for any $x \in \Omega$, there exists an index set $I_x \in \{1,\ldots,M^d \}$ such that
\begin{equation}\label{eq:rela1}
  A(x) \subseteq \Omega \cap B(x, r) \subseteq \Omega \cap B_{\infty}(x, r) \subseteq \cup_{k \in I_x}\Omega_k
\end{equation}
with $\mathrm{Card}(I_x) \leq ( 2rM +2)^d$. Recall also that for $x' \in \Omega_k$, we write $u_{x'} = u_k$ and $\hat{f}_{\mathrm{PP}}(x') = \tilde{f}_{u_kh}(x')$. Combining this with the relation in (\ref{eq:rela1}), we can derive the following upper bound:
\begin{equation}\label{eq:rela2}
\begin{split}
   &\mathbb{E}_{\otimes^n}\sup_{x' \in A(x)}\left| \hat{f}_{\mathrm{PP}}(x') - f(x') \right|^q \leq \mathbb{E}_{\otimes^n}\sup_{k \in I_x}\sup_{x' \in \Omega_k}\left| \hat{f}_{\mathrm{PP}}(x') - f(x') \right|^q \\
     & \leq 2^{q-1} \mathbb{E}_{\otimes^n}\sup_{k \in I_x}\sup_{x' \in \Omega_k}\left| \tilde{f}_{u_kh}(x') - f_{u_k}(x') \right|^q + 2^{q-1} \sup_{k \in I_x}\sup_{x' \in \Omega_k} \left| f_{u_k}(x') - f(x') \right|^q\\
     & \lesssim \mathbb{E}_{\otimes^n}\sup_{k \in I_x}\sup_{x' \in \Omega_k}\left| \tilde{f}_{u_kh}(x') - f_{u_k}(x') \right|^q + \frac{1}{M^{q\beta}},
\end{split}
\end{equation}
where the last inequality follows from the results in Lemma~\ref{lem:taylor}.

We now proceed to upper bound the term $\mathbb{E}_{\otimes^n}\sup_{k \in I_x}\sup_{x' \in \Omega_k}| \tilde{f}_{u_kh}(x') - f_{u_k}(x') |^q$ in the last line of (\ref{eq:rela2}). Note that this term can be decomposed as
\begin{subequations}
\begin{align}
&\mathbb{E}_{\otimes^n}\sup_{k \in I_x}\sup_{x' \in \Omega_k}\left| \tilde{f}_{u_kh}(x') - f_{u_k}(x') \right|^q \nonumber \\
&= \mathbb{E}_{\otimes^n}\sup_{k \in I_x}\sup_{x' \in \Omega_k}\left| \tilde{f}_{u_kh}(x') - f_{u_k}(x') \right|^q1_{\{ \mathcal{C}_h \}} \label{eqn:line-1} \\
&+ \mathbb{E}_{\otimes^n}\sup_{k \in I_x}\sup_{x' \in \Omega_k}\left| \tilde{f}_{u_kh}(x') - f_{u_k}(x') \right|^q1_{\{ \mathcal{C}_h^c  \}}, \label{eqn:line-2}
\end{align}
\end{subequations}
where the event $\mathcal{C}_h$ is defined in Lemma~\ref{lem:event}.

\subsubsection{Upper bounding~(\ref{eqn:line-1})}

Define the vector $c_{u_kh}$ in $\mathbb{R}^{N_{\ell,d}}$:
\begin{equation}\label{eq:chule1}
  c_{u_kh} = \left( \tilde{D}^sf(u_k)h^{|s|} \right)_{0 \leq |s| \leq \ell},
\end{equation}
which equals to the vector $c_{uh}$ defined in (\ref{eq:ljinv2}) with $u = u_k$ and $\ell \geq \lfloor \beta \rfloor$. In addition, define
\begin{equation}\label{eq:d}
  \begin{split}
     d_{u_kh} \triangleq B_{u_kh} c_{u_kh} & = \frac{1}{nh^d}\sum_{i=1}^{n} U\left( \frac{X_i - u_k}{h} \right) U^{\top}\left( \frac{X_i - u_k}{h} \right)c_{u_kh} K\left( \frac{X_i-u_k}{h} \right)\\
       & = \frac{1}{nh^d} \sum_{i=1}^{n} f_{u_k}(X_i) U\left( \frac{X_i - u_k}{h} \right) K\left( \frac{X_i-u_k}{h} \right),
  \end{split}
\end{equation}
where the second equality follows from the definition of $f_{u_k}$. Since $\mathcal{C}_h \subseteq \mathcal{A}_{u_kh} \cap \mathcal{B}_{u_kh} \cap \mathcal{D}_{u_kh}$, when the event $\mathcal{C}_h$ holds, we have $\tilde{f}_{u_kh}(x') = \hat{f}_{u_kh}(x')$, and the matrix $D_{u_kh}$ is invertible since $\lambda_{\min}(D_{u_kh})\geq c_3$. Thus, when $\mathcal{C}_h$ holds, for $1 \leq k \leq M^d$, we have
\begin{equation}\label{eq:4}
  \begin{split}
     \tilde{f}_{u_kh}(x') - f_{u_k}(x') & = \hat{f}_{u_kh}(x') - f_{u_k}(x')\\
      &= U^{\top}\left(\frac{x' - u_k}{h}\right)B_{u_kh}^{-1}a_{u_kh} - U^{\top}\left(\frac{x' - u_k}{h}\right)c_{u_kh}\\
       & = U^{\top}\left(\frac{x' - u_k}{h}\right)B_{u_kh}^{-1} \left( a_{u_kh} - d_{u_kh}\right)\\
       & = U^{\top}\left(\frac{x' - u_k}{h}\right)D_{u_kh}^{-\frac{1}{2}} (D_{u_kh}^{-\frac{1}{2}}B_{u_kh}D_{u_kh}^{-\frac{1}{2}})^{-1} D_{u_kh}^{-\frac{1}{2}}\left( a_{u_kh} - d_{u_kh}\right),
  \end{split}
\end{equation}
where the second equality follows from the definitions of $\hat{f}_{u_kh}(x')$ and $f_{u_k}(x')$, and the third equality follows from the definition of $d_{u_kh}$. Define $\Delta_{u_kh}(s)$ as the $s$-th element of the vector $\Delta_{u_kh}\triangleq D_{u_kh}^{-1/2} (D_{u_kh}^{-1/2}B_{u_kh}D_{u_kh}^{-1/2})^{-1} D_{u_kh}^{-1/2}\left( a_{u_kh} - d_{u_kh}\right)$. Therefore, when the event $\mathcal{C}_h$ holds, we have
\begin{equation}\label{eq:tus2}
  \begin{split}
     \sup_{k \in I_x}\sup_{x' \in \Omega_k}\left| \hat{f}_{u_kh}(x') - f_{u_k}(x') \right| & = \sup_{k \in I_x}\sup_{x' \in \Omega_k}\left|  \sum_{0 \leq |s| \leq \ell}\frac{1}{s!}\left( \frac{x' - u_k}{h} \right)^s \Delta_{u_kh}(s) \right|\\
       & = \sup_{k \in I_x}\sup_{x' \in \Omega_k}\left|  \sum_{0 \leq |s| \leq \ell}\frac{1}{s!h^{|s|}}\left( x' - u_k \right)^s \Delta_{u_kh}(s) \right|\\
       & \leq \sup_{k \in I_x}\sup_{x' \in \Omega_k}  \sum_{0 \leq |s| \leq \ell}\frac{1}{s!h^{|s|}}\left|\left( x' - u_k \right)^s\right| \left| \Delta_{u_kh}(s) \right|\\
       & \leq \sup_{k \in I_x}\sum_{0 \leq |s| \leq \ell} \frac{1}{s!(2M)^{|s|}h^{|s|}}  \left| \Delta_{u_kh}(s) \right|,
  \end{split}
\end{equation}
where the last inequality follows from $x' \in \Omega_k$ and the fact that $\left| \Delta_{u_kh}(s) \right|$ does not depend on $x'$.

Our following goal is to upper bound $\sup_{k \in I_x}| \Delta_{u_kh}(s) |$. In view of (\ref{eq:d}), the $s$-th component of the vector $ a_{u_kh} - d_{u_kh}$ can be expressed as
\begin{equation}\label{eq:k1}
\begin{split}
    \left( a_{u_kh} - d_{u_kh}\right)_s & = \frac{1}{nh^d} \sum_{i=1}^{n} \left[Y_i - f_{u_k}(X_i)\right] \frac{1}{s!}\left( \frac{X_i - u_k}{h} \right)^s K\left( \frac{X_i-u_k}{h} \right)\\
     & = \frac{1}{nh^d}\sum_{i=1}^{n} \left[f(X_i) - f_{u_k}(X_i)\right] \frac{1}{s!}\left( \frac{X_i - u_k}{h} \right)^s K\left( \frac{X_i-u_k}{h} \right)\\
     &\quad+ \frac{1}{nh^d}\sum_{i=1}^{n} \xi_i \frac{1}{s!}\left( \frac{X_i - u_k}{h} \right)^s K\left( \frac{X_i-u_k}{h} \right)\\
     & \triangleq b_{u_kh}(s) + v_{u_kh}(s),
\end{split}
\end{equation}
which corresponds respectively to the bias and estimation error terms. In matrix notation, this decomposition can also be written as
\begin{equation}\label{eq:part0}
  \begin{split}
     a_{u_kh} - d_{u_kh} & = \frac{1}{nh^d}\sum_{i=1}^{n} \left[f(X_i) - f_{u_k}(X_i)\right] U\left( \frac{X_i - u_k}{h} \right) K\left( \frac{X_i-u_k}{h} \right) \\
       & \quad+ \frac{1}{nh^d}\sum_{i=1}^{n} \xi_i U\left( \frac{X_i - u_k}{h} \right) K\left( \frac{X_i-u_k}{h} \right)\\
       & \triangleq b_{u_kh} + v_{u_kh}.
  \end{split}
\end{equation}
Using the Cauchy–Schwarz inequality, the $s$-th component of the bias term $b_{u_kh}$ can be upper bounded by
\begin{equation}\label{eq:5}
  \begin{split}
b_{u_kh}(s)& = \frac{1}{nh^d} \sum_{i=1}^{n} \left[f(X_i) - f_{u_k}(X_i)\right] \sqrt{K\left( \frac{X_i-u_k}{h} \right)} \frac{1}{s!} \left( \frac{X_i - u_k}{h} \right)^s \sqrt{K\left( \frac{X_i-u_k}{h} \right)} \\
       & \leq \left\{\frac{1}{nh^d}\sum_{i=1}^{n} \left[f(X_i) - f_{u_k}(X_i)\right]^2 K\left( \frac{X_i-u_k}{h} \right)\right\}^{\frac{1}{2}}\\
       &\quad\times\left\{\frac{1}{nh^d}\sum_{i=1}^{n} \frac{1}{(s!)^2} \left( \frac{X_i - u_k}{h} \right)^{2s} K\left( \frac{X_i-u_k}{h} \right)\right\}^{\frac{1}{2}}.
  \end{split}
\end{equation}
Recalling the definition (\ref{eq:D}), we have
\begin{equation*}
  \begin{split}
      \left(D_{u_kh}^{-\frac{1}{2}}b_{u_kh} \right)_s  & \leq  \left\{\frac{1}{nh^d}\sum_{i=1}^{n} \left[f(X_i) - f_{u_k}(X_i)\right]^2 K\left( \frac{X_i-u_k}{h} \right)\right\}^{\frac{1}{2}}\\
       & \lesssim \left\{\frac{1}{nh^d}\sum_{i=1}^{n} \| X_i - u_k \|^{2\beta} K\left( \frac{X_i-u_k}{h} \right)\right\}^{\frac{1}{2}}\\
       & \lesssim \frac{n_{u_kh}^{\frac{1}{2}}}{(nh^d)^{\frac{1}{2}}} h^{\beta} \lesssim h^{\beta}
  \end{split}
\end{equation*}
where the first inequality follows from the definition of $D_{u_kh}$ and (\ref{eq:5}), the second inequality follows from (\ref{eq:poly2}), and the third inequality follows from Assumption~\ref{ass:kernel}, and the last inequality follows from $\mathcal{C}_h\subseteq \mathcal{A}_{u_kh}$. Hence, the vector $(D_{u_kh}^{-1/2}B_{u_kh}D_{u_kh}^{-1/2})^{-1} D_{u_kh}^{-1/2}b_{u_kh}$ can be written as
\begin{equation}\label{eq:part1}
  \left(D_{u_kh}^{-\frac{1}{2}}B_{u_kh}D_{u_kh}^{-\frac{1}{2}}\right)^{-1}  h^{\beta} \alpha_{u_kh},
\end{equation}
where the vector $\alpha_{u_kh} \in \mathbb{R}^{N_{\ell,d}}$ satisfies $\| \alpha_{u_kh} \|_{\infty} \lesssim 1$.

We now present an equivalent expression for the vector $(D_{u_kh}^{-1/2}B_{u_kh}D_{u_kh}^{-1/2})^{-1} D_{u_kh}^{-1/2}v_{u_kh}$. Note that
\begin{equation}\label{eq:huai1}
  \begin{split}
v_{u_kh}& = \frac{1}{nh^d}\left[
        U\left(\frac{X_1 - u_k}{h}\right)\sqrt{K\left( \frac{X_1-u_k}{h} \right)},\cdots , U\left(\frac{X_n - u_k}{h}\right)\sqrt{K\left( \frac{X_n-u_k}{h} \right)}
      \right]\\
      &\quad \times \begin{pmatrix}
\xi_1\sqrt{K\left( \frac{X_1-u_k}{h} \right)} \\
\vdots \\
\xi_n\sqrt{K\left( \frac{X_n-u_k}{h} \right)}
\end{pmatrix} \\
       &  \triangleq \frac{1}{\sqrt{nh^d}} Q_{u_kh} e_{u_kh},
  \end{split}
\end{equation}
where $Q_{u_kh} \in \mathbb{R}^{N_{\ell,d}\times n}$ satisfies $Q_{u_kh}Q_{u_kh}^{\top}=B_{u_kh}$, and the elements of $e_{u_kh} \in \mathbb{R}^n$ are independent $(K_{\max}\sigma)$-sub-Gaussian random variables conditioned on $\mathfrak{X}_n$. Define the random vector $\gamma_{u_kh} \triangleq (D_{u_kh}^{-1/2}B_{u_kh}D_{u_kh}^{-1/2})^{-1/2} D_{u_kh}^{-1/2}Q_{u_kh} e_{u_kh}$. Since the property of sub-Gaussianity is preserved by linear operations, and
\begin{equation*}
\begin{split}
     &   \left(D_{u_kh}^{-\frac{1}{2}}B_{u_kh}D_{u_kh}^{-\frac{1}{2}}\right)^{-\frac{1}{2}} D_{u_kh}^{-\frac{1}{2}} Q_{u_kh}  Q_{u_kh}^{\top} D_{u_kh}^{-\frac{1}{2}} \left(D_{u_kh}^{-\frac{1}{2}}B_{u_kh}D_{u_kh}^{-\frac{1}{2}}\right)^{-\frac{1}{2}} = I_{N_{\ell,d}},
\end{split}
\end{equation*}
we see that the elements of $\gamma_{u_kh}$ are $(K_{\max}\sigma)$-sub-Gaussian random variables. Thus, we obtain
\begin{equation}\label{eq:part3}
  \left(D_{u_kh}^{-\frac{1}{2}}B_{u_kh}D_{u_kh}^{-\frac{1}{2}}\right)^{-1} D_{u_kh}^{-\frac{1}{2}}v_{u_kh} = \frac{1}{\sqrt{nh^d}} \left(D_{u_kh}^{-\frac{1}{2}}B_{u_kh}D_{u_kh}^{-\frac{1}{2}}\right)^{-\frac{1}{2}}\gamma_{u_kh}.
\end{equation}
Combining (\ref{eq:part0}) with (\ref{eq:part1}) and (\ref{eq:part3}), we have
\begin{equation}\label{eq:important1}
\begin{split}
     & \left(D_{u_kh}^{-\frac{1}{2}}B_{u_kh}D_{u_kh}^{-\frac{1}{2}}\right)^{-1} D_{u_kh}^{-\frac{1}{2}}\left( a_{u_kh} - d_{u_kh}\right) \\
     & = \left(D_{u_kh}^{-\frac{1}{2}}B_{u_kh}D_{u_kh}^{-\frac{1}{2}}\right)^{-1} D_{u_kh}^{-\frac{1}{2}}\left( b_{u_kh} + v_{u_kh}\right)\\
     & \lesssim \left(D_{u_kh}^{-\frac{1}{2}}B_{u_kh}D_{u_kh}^{-\frac{1}{2}}\right)^{-1}  h^{\beta} \alpha_{u_kh} + \frac{1}{\sqrt{nh^d}} \left(D_{u_kh}^{-\frac{1}{2}}B_{u_kh}D_{u_kh}^{-\frac{1}{2}}\right)^{-\frac{1}{2}}\gamma_{u_kh},
\end{split}
\end{equation}
where we emphasize again that $\| \alpha_{u_kh} \|_{\infty} \lesssim 1$, and the elements of $\gamma_{u_kh}$ are centered $(K_{\max}\sigma)$-sub-Gaussian random variables conditioned on $\mathfrak{X}_n$.

From (\ref{eq:important1}), we see that each element of $(D_{u_kh}^{-1/2}B_{u_kh}D_{u_kh}^{-1/2})^{-1} D_{u_kh}^{-1/2}( a_{u_kh} - d_{u_kh})$ can be upper bounded by
\begin{equation}\label{eq:important2}
  \begin{split}
     & \left\| \left(D_{u_kh}^{-\frac{1}{2}}B_{u_kh}D_{u_kh}^{-\frac{1}{2}}\right)^{-1} D_{u_kh}^{-\frac{1}{2}}\left( a_{u_kh} - d_{u_kh}\right) \right\|_{\infty} \\
       & \lesssim \left\| \left(D_{u_kh}^{-\frac{1}{2}}B_{u_kh}D_{u_kh}^{-\frac{1}{2}}\right)^{-1} h^{\beta} \alpha_{u_kh}\right\|_{\infty} + \frac{1}{\sqrt{nh^d}} \left\|\left(D_{u_kh}^{-\frac{1}{2}}B_{u_kh}D_{u_kh}^{-\frac{1}{2}}\right)^{-\frac{1}{2}}\gamma_{u_kh}\right\|_{\infty}\\
       & \leq \left\| \left(D_{u_kh}^{-\frac{1}{2}}B_{u_kh}D_{u_kh}^{-\frac{1}{2}}\right)^{-1} h^{\beta} \alpha_{u_kh}\right\| + \frac{1}{\sqrt{nh^d}} \left\|\left(D_{u_kh}^{-\frac{1}{2}}B_{u_kh}D_{u_kh}^{-\frac{1}{2}}\right)^{-\frac{1}{2}}\gamma_{u_kh}\right\|\\
       & \leq \lambda_{\max}\left[ \left(D_{u_kh}^{-\frac{1}{2}}B_{u_kh}D_{u_kh}^{-\frac{1}{2}}\right)^{-1}  \right]  h^{\beta}\left\| \alpha_{u_kh}\right\|\\
        &\quad+ \frac{1}{\sqrt{nh^d}}\lambda_{\max}\left[ \left(D_{u_kh}^{-\frac{1}{2}}B_{u_kh}D_{u_kh}^{-\frac{1}{2}}\right)^{-\frac{1}{2}}  \right] \left\| \gamma_{u_kh}\right\|\\
       & \leq \lambda_{\min}^{-1} \left(D_{u_kh}^{-\frac{1}{2}}B_{u_kh}D_{u_kh}^{-\frac{1}{2}}\right) h^{\beta}\sqrt{N_{\ell,d}}\| \alpha_{u_kh} \|_{\infty}\\
        &\quad+ \frac{1}{\sqrt{nh^d}}\sqrt{N_{\ell,d}}\lambda_{\max}\left[ \left(D_{u_kh}^{-\frac{1}{2}}B_{u_kh}D_{u_kh}^{-\frac{1}{2}}\right)^{-1}  \right] \sqrt{N_{\ell,d}}\left\| \gamma_{u_kh}\right\|_{\infty}\\
       & \lesssim \lambda_{\min}^{-1} \left(D_{u_kh}^{-\frac{1}{2}}B_{u_kh}D_{u_kh}^{-\frac{1}{2}}\right) \left( h^{\beta} + \frac{1}{\sqrt{nh^d}}\left\| \gamma_{u_kh}\right\|_{\infty} \right),
  \end{split}
\end{equation}
where the second inequality follows from $\| \cdot \|_{\infty} \leq \| \cdot \|$, and the forth step follows from $ \| a \| \leq \sqrt{N_{\ell,d}}\| a \|_{\infty}$ for any vector $a \in \mathbb{R}^{N_{\ell,d}}$. The second term in the forth step of (\ref{eq:important2}) follows from
\begin{equation*}
  \begin{split}
     \lambda_{\max}\left[ \left(D_{u_kh}^{-\frac{1}{2}}B_{u_kh}D_{u_kh}^{-\frac{1}{2}}\right)^{-\frac{1}{2}}  \right] & = \lambda_{\max}\left[ \left(D_{u_kh}^{-\frac{1}{2}}B_{u_kh}D_{u_kh}^{-\frac{1}{2}}\right)^{-1}  \left(D_{u_kh}^{-\frac{1}{2}}B_{u_kh}D_{u_kh}^{-\frac{1}{2}}\right)^{\frac{1}{2}} \right] \\
       & \leq \lambda_{\max}\left[ \left(D_{u_kh}^{-\frac{1}{2}}B_{u_kh}D_{u_kh}^{-\frac{1}{2}}\right)^{-1}  \right] \lambda_{\max}\left[ \left(D_{u_kh}^{-\frac{1}{2}}B_{u_kh}D_{u_kh}^{-\frac{1}{2}}\right)^{\frac{1}{2}}  \right]\\
       & \leq \sqrt{N_{\ell,d}}\lambda_{\max}\left[ \left(D_{u_kh}^{-\frac{1}{2}}B_{u_kh}D_{u_kh}^{-\frac{1}{2}}\right)^{-1}  \right],
  \end{split}
\end{equation*}
where the last inequality is due to each entries of the matrix $D_{u_kh}^{-1/2}B_{u_kh}D_{u_kh}^{-1/2}$ is smaller than $1$ in absolute value by the Cauchy–Schwarz inequality. Since
\begin{equation*}
\begin{split}
     \lambda_{\min} \left(D_{u_kh}^{-\frac{1}{2}}B_{u_kh}D_{u_kh}^{-\frac{1}{2}}\right) &\geq \lambda_{\min} ^2 \left(D_{u_kh}^{-\frac{1}{2}} \right) \lambda_{\min}\left( B_{u_kh} \right) \\
     & = \lambda_{\max}^{-1}\left( D_{u_kh} \right) \lambda_{\min}\left( B_{u_kh} \right) \geq C
\end{split}
\end{equation*}
and $\lambda_{\min}(D_{u_kh}) \geq c_3$ under the event $\mathcal{C}_h$, we thus have
\begin{equation*}
  \left\| \left(D_{u_kh}^{-\frac{1}{2}}B_{u_kh}D_{u_kh}^{-\frac{1}{2}}\right)^{-1} D_{u_kh}^{-\frac{1}{2}}\left( a_{u_kh} - d_{u_kh}\right) \right\|_{\infty} \lesssim h^{\beta} + \frac{1}{\sqrt{nh^d}}\left\| \gamma_{u_kh}\right\|_{\infty},
\end{equation*}
and
\begin{equation}\label{eq:tus1}
  \left\| \Delta_{u_kh}\right\|_{\infty}= \left\|D_{u_kh}^{-\frac{1}{2}} \left(D_{u_kh}^{-\frac{1}{2}}B_{u_kh}D_{u_kh}^{-\frac{1}{2}}\right)^{-1} D_{u_kh}^{-\frac{1}{2}}\left( a_{u_kh} - d_{u_kh}\right) \right\|_{\infty} \lesssim h^{\beta} + \frac{1}{\sqrt{nh^d}}\left\| \gamma_{u_kh}\right\|_{\infty}.
\end{equation}

Now define a random vector $\gamma_{nx} = (\gamma_{u_kh}^{\top})^{\top}_{k \in I_{x}}$. Then combining (\ref{eq:tus2}) with (\ref{eq:tus1}), we have
\begin{equation*}
\begin{split}
    \sup_{k \in I_x}\sup_{x' \in \Omega_k}\left| \hat{f}_{u_kh}(x') - f_{u_k}(x') \right| & \lesssim \sum_{0 \leq |s| \leq \ell} \frac{1}{(M h)^{|s|}} \left(h^{\beta} + \frac{1}{\sqrt{nh^d}}\left\| \gamma_{nx}\right\|_{\infty} \right) \\
     &  \lesssim \left(h^{\beta} + \frac{1}{\sqrt{nh^d}}\left\| \gamma_{nx}\right\|_{\infty} \right),
\end{split}
\end{equation*}
where the second inequality follows from the condition $M h \gtrsim 1$. Thus, we have
\begin{equation*}
  \mathbb{E}_{\otimes^n}\sup_{k \in I_x}\sup_{x' \in \Omega_k}\left| \tilde{f}_{u_kh}(x') - f_{u_k}(x') \right|^q1_{\{ \mathcal{C}_h \}} \lesssim h^{q\beta} + \frac{\mathbb{E}_{\otimes^n}\left(\left\| \gamma_{nx}\right\|_{\infty}^q\right)}{(nh^d)^{q/2}}.
\end{equation*}
Note that
\begin{equation*}
\begin{split}
   \mathbb{E}_{\otimes^n}\left(\left\| \gamma_{nx}\right\|_{\infty}^q\right) & = \mathbb{E}_{\otimes^n} \left[ \mathbb{E}_{\otimes^n} \left( \left\| \gamma_{nx}\right\|_{\infty}^q | \mathfrak{X}_n \right)\right]\\
     & = \mathbb{E}_{\otimes^n} \left\{ \mathbb{E}_{\otimes^n} \left[ \left(\max_{1 \leq i \leq N_{\ell,d}\mathrm{Card}(I_x)}\left| W_i \right|\right) ^q | \mathfrak{X}_n \right]\right\},
\end{split}
\end{equation*}
where $W_i$ are the centered $(K_{\max}\sigma)$-sub-Gaussian random variables conditioned on $\mathfrak{X}_n$. Then, based on Lemma~\ref{lem:sub_Gaussian_bound}, we know that
\begin{equation*}
\begin{split}
   \mathbb{E}_{\otimes^n} \left[ \left(\max_{1 \leq i \leq N_{\ell,d}\mathrm{Card}(I_x)}\left| W_i \right|\right) ^q | \mathfrak{X}_n \right] & \lesssim \left[  \log (N_{\ell,d}\mathrm{Card}(I_x) ) \right]^{\frac{q}{2}} \\
     & \lesssim  \left[\log N_{\ell,d} + d \log (2rM +2)\right]^{\frac{q}{2}},
\end{split}
\end{equation*}
where the last inequality follows from $\mathrm{Card}(I_x) \leq ( 2rM +2)^d$.
Thus, we conclude that
\begin{equation}\label{eq:xiayule3}
  \mathbb{E}_{\otimes^n}\sup_{k \in I_x}\sup_{x' \in \Omega_k}\left| \tilde{f}_{u_kh}(x') - f_{u_k}(x') \right|^q 1_{\{ \mathcal{C}_h \}} \lesssim h^{q\beta} +\left[ \frac{\log (rM +1) }{nh^d}\right]^{q/2}
\end{equation}
under the condition $M h \gtrsim 1$.

\subsubsection{Upper bounding~(\ref{eqn:line-2})}\label{sec:exinnv11}

Since $f \in \mathcal{F}(\beta,C_{\beta})$, there exists a constant $C$ such that $|f_{u_k}(x')| \leq C$. Thus,
\begin{equation}\label{eq:gift1}
  \begin{split}
     \mathbb{E}_{\otimes^n}\sup_{k \in I_x}\sup_{x' \in \Omega_k}\left| \tilde{f}_{u_kh}(x') - f_{u_k}(x') \right|^q1_{\{ \mathcal{C}_h^c  \}} & \lesssim \mathbb{E}_{\otimes^n}\left[\sup_{k \in I_x}\sup_{x' \in \Omega_k}\left| \tilde{f}_{u_kh}(x') \right|\right]^q1_{\{ \mathcal{C}_h^c \}}\\
      &+ C\mathbb{P}_{\otimes^n}\left( \mathcal{C}_h^c \right). \\
  \end{split}
\end{equation}
The second term, $C\mathbb{P}_{\otimes^n}( \mathcal{C}_h^c )$, can be upper bounded by an exponentially vanishing probability as given in Lemma~\ref{lem:event}. Our primary goal now is to bound the first term in (\ref{eq:gift1}).

First, if $n_{u_kh}=0$, then from (\ref{eq:lpe_modified}), we have $\tilde{f}_{u_kh} = 0$, implying
\begin{equation}\label{eq:gift2}
  \sup_{x' \in \Omega_k}\left| \tilde{f}_{u_kh}(x') \right|=0.
\end{equation}
In this case, the first term in (\ref{eq:gift1}) is upper bounded by $0$.

Next, we assume $n_{u_kh}>0$. In this case, the matrixes $\tilde{B}_{u_kh}$ and $\tilde{D}_{u_kh}$, defined in (\ref{eq:B_tilde}) and (\ref{eq:tilde_D}), are invertible, as $\lambda_{\min}(\tilde{B}_{u_kh}) \geq \tau$ and $\lambda_{\min}(\tilde{D}_{u_kh}) \geq \tau$. By (\ref{eq:lpe_modified}), the estimator $\tilde{f}_{u_kh}$ can be expressed as
\begin{equation}\label{eq:pp1}
  \begin{split}
   \tilde{f}_{u_kh}(x')   &= U^{\top}\left(\frac{x' - u_k}{h}\right)\tilde{B}_{u_kh}^{-1}a_{u_kh} \\
       & = U^{\top}\left(\frac{x' - u_k}{h}\right)\tilde{D}_{u_kh}^{-\frac{1}{2}} (\tilde{D}_{u_kh}^{-\frac{1}{2}}\tilde{B}_{u_kh}\tilde{D}_{u_kh}^{-\frac{1}{2}})^{-1} \tilde{D}_{u_kh}^{-\frac{1}{2}} a_{u_kh}\\
       & \triangleq U^{\top}\left(\frac{x' - u_k}{h}\right)\tilde{\theta}_{u_kh},
  \end{split}
\end{equation}
where $\tilde{\theta}_{u_kh}\triangleq \tilde{D}_{u_kh}^{-1/2} (\tilde{D}_{u_kh}^{-1/2}\tilde{B}_{u_kh}\tilde{D}_{u_kh}^{-1/2})^{-1} \tilde{D}_{u_kh}^{-1/2} a_{u_kh} $, and $\tilde{\theta}_{u_kh}(s)$ denotes the $s$-th element of the vector $\tilde{\theta}_{u_kh}$. Thus, we have
\begin{equation}\label{eq:qian1}
  \begin{split}
     \sup_{k \in I_x}\sup_{x' \in \Omega_k}\left| \tilde{f}_{u_kh}(x') \right| & = \sup_{k \in I_x}\sup_{x' \in \Omega_k}\left|  \sum_{0 \leq |s| \leq \ell}\frac{1}{s!}\left( \frac{x' - u_k}{h} \right)^s \tilde{\theta}_{u_kh}(s) \right|\\
       & \leq \sup_{k \in I_x}\sum_{0 \leq |s| \leq \ell} \frac{1}{s!(2M h)^{|s|}}  \left| \tilde{\theta}_{u_kh}(s) \right|\\
       & \lesssim \sup_{k \in I_x}\sup_{0 \leq |s| \leq \ell}\left| \tilde{\theta}_{u_kh}(s) \right| = \sup_{k \in I_x}\left\| \tilde{\theta}_{u_kh} \right\|_{\infty},
  \end{split}
\end{equation}
where the first inequality follows from $x' \in \Omega_k$, and the second inequality follows from the condition $Mh \gtrsim  1$.

Our goal now is to upper bound $\| \tilde{\theta}_{u_kh} \|_{\infty}$. Using the relationship in (\ref{eq:japan1}), we have
\begin{subequations}
\begin{align}
\left\| \tilde{\theta}_{u_kh} \right\|_{\infty}   & = \left\|\tilde{D}_{u_kh}^{-\frac{1}{2}} (\tilde{D}_{u_kh}^{-\frac{1}{2}}\tilde{B}_{u_kh}\tilde{D}_{u_kh}^{-\frac{1}{2}})^{-1} \tilde{D}_{u_kh}^{-\frac{1}{2}} a_{u_kh} \right\|_{\infty}\nonumber \\
       & = \left\|\frac{1}{\sqrt{nh^d}}\tilde{D}_{u_kh}^{-\frac{1}{2}} (\tilde{D}_{u_kh}^{-\frac{1}{2}}\tilde{B}_{u_kh}\tilde{D}_{u_kh}^{-\frac{1}{2}})^{-1} \tilde{D}_{u_kh}^{-\frac{1}{2}} Q_{u_kh} \left( \eta_{u_kh} + e_{u_kh} \right) \right\|_{\infty} \nonumber \\
       & \leq \frac{1}{\sqrt{nh^d}} \left\| \tilde{D}_{u_kh}^{-\frac{1}{2}} (\tilde{D}_{u_kh}^{-\frac{1}{2}}\tilde{B}_{u_kh}\tilde{D}_{u_kh}^{-\frac{1}{2}})^{-1} \tilde{D}_{u_kh}^{-\frac{1}{2}} Q_{u_kh} \eta_{u_kh} \right\|_{\infty} \label{eq:huanle1}\\
       & + \frac{1}{\sqrt{nh^d}} \left\| \tilde{D}_{u_kh}^{-\frac{1}{2}} (\tilde{D}_{u_kh}^{-\frac{1}{2}}\tilde{B}_{u_kh}\tilde{D}_{u_kh}^{-\frac{1}{2}})^{-1} \tilde{D}_{u_kh}^{-\frac{1}{2}} Q_{u_kh} e_{u_kh} \right\|_{\infty}. \label{eq:huanle2}
\end{align}
\end{subequations}

\underline{\textsc{Upper bounding}~(\ref{eq:huanle1}).} The first term (\ref{eq:huanle1}) can be upper bounded by
\begin{equation}\label{eq:xianbai3}
  \begin{split}
  \mathrm{(\ref{eq:huanle1})} & \leq \frac{1}{\sqrt{nh^d}} \left\| \tilde{D}_{u_kh}^{-\frac{1}{2}} (\tilde{D}_{u_kh}^{-\frac{1}{2}}\tilde{B}_{u_kh}\tilde{D}_{u_kh}^{-\frac{1}{2}})^{-1} \tilde{D}_{u_kh}^{-\frac{1}{2}} Q_{u_kh} \eta_{u_kh} \right\| \\
     & \leq \frac{1}{\sqrt{nh^d}} \lambda_{\max}\left[ \tilde{D}_{u_kh}^{-\frac{1}{2}} (\tilde{D}_{u_kh}^{-\frac{1}{2}}\tilde{B}_{u_kh}\tilde{D}_{u_kh}^{-\frac{1}{2}})^{-1}   \right] \left\| \tilde{D}_{u_kh}^{-\frac{1}{2}} Q_{u_kh} \eta_{u_kh} \right\|\\
     & \leq \sqrt{\frac{N_{\ell,d}}{nh^d}} \lambda_{\max}\left( \tilde{B}_{u_kh}^{-1} \tilde{D}_{u_kh}^{\frac{1}{2}} \right) \left\| \tilde{D}_{u_kh}^{-\frac{1}{2}} Q_{u_kh} \eta_{u_kh} \right\|_{\infty},\\
  \end{split}
\end{equation}
where the first inequality uses $\| \cdot \|_{\infty} \leq \| \cdot \| $, and the last inequality follows from $\| \cdot \| \leq \sqrt{N_{\ell,d}} \| \cdot \|_{\infty}$. Recall that the sum of the squares of the $s$-th row of $Q_{u_kh}$ equals to the $ss$-th diagonal element of $D_{u_kh}$. By the definition of $\tilde{D}_{u_kh}$ and the Cauchy–Schwarz inequality, we have
\begin{equation}\label{eq:xianbai1}
  \left\| \tilde{D}_{u_kh}^{-\frac{1}{2}} Q_{u_kh} \eta_{u_kh} \right\|_{\infty} \leq \left\|  \eta_{u_kh}  \right\| = \left[\sum_{i=1}^{n} f^2(X_i)  K\left( \frac{X_i-u_k}{h} \right)\right]^{\frac{1}{2}} \lesssim \sqrt{n},
\end{equation}
where the second inequality holds since $f \in \mathcal{F}(\beta, C_{\beta})$ and $K$ satisfies Assumption~\ref{ass:kernel}. Furthermore, using the fact that $\lambda_{\max}(AB) \leq \lambda_{\max}(A)\lambda_{\max}(B)$ for two symmetric positive definite matrices, we obtain
\begin{equation}\label{eq:xianbai2}
  \begin{split}
     \lambda_{\max}\left( \tilde{B}_{u_kh}^{-1} \tilde{D}_{u_kh}^{\frac{1}{2}} \right) & \leq \lambda_{\max}\left( \tilde{B}_{u_kh}^{-1} \right) \lambda_{\max}\left(\tilde{D}_{u_kh}^{\frac{1}{2}} \right)  = \lambda_{\min}^{-1}\left( \tilde{B}_{u_kh} \right) \lambda_{\max}^{\frac{1}{2}}\left(\tilde{D}_{u_kh} \right) \\
       & \leq \tau^{-1} \left( \frac{K_{\max}}{h^d} + \tau \right)^{\frac{1}{2}},\\
  \end{split}
\end{equation}
where the equality holds because $\tilde{B}_{u_kh}$ is positive definite and $\tilde{D}_{u_kh}$ is a diagonal matrix. The last inequality in (\ref{eq:xianbai2}) follows from the fact that the diagonal elements of $\tilde{D}_{u_kh}$ are no larger than $K_{\max}/h^d + \tau$. Substituting (\ref{eq:xianbai1})--(\ref{eq:xianbai2}) into (\ref{eq:xianbai3}), we obtain
\begin{equation}\label{eq:xianbai4}
  \mathrm{(\ref{eq:huanle1})} \lesssim  \frac{1}{\sqrt{nh^d}}\times\frac{1}{\tau} \left( \frac{K_{\max}}{h^d} + \tau \right)^{\frac{1}{2}}\times\sqrt{n} \lesssim \frac{1}{\tau h^d},
\end{equation}
where the last step follows from $h=O(1)$ and $\tau = O(1)$.

\underline{\textsc{Upper bounding}~(\ref{eq:huanle2}).} Recall that the elements of $e_{u_kh}$ are $n$ independent $(K_{\max}\sigma)$-sub-Gaussian random variables. Define the random vector
\begin{equation}\label{eq:juqi1}
  \tilde{\gamma}_{u_kh} \triangleq \left(\tilde{D}_{u_kh}^{-\frac{1}{2}}\tilde{B}_{u_kh}\tilde{D}_{u_kh}^{-\frac{1}{2}}\right)^{-\frac{1}{2}} \tilde{D}_{u_kh}^{-\frac{1}{2}}Q_{u_kh} e_{u_kh} \in \mathbb{R}^{N_{\ell,d}}.
\end{equation}
Using the relation $B_{u_kh} = Q_{u_kh} Q_{u_kh}^{\top}$, we observe that the covariance matrix of  $\tilde{\gamma}_{u_kh}$ is given by
\begin{equation*}
\begin{split}
     & (\tilde{D}_{u_kh}^{-\frac{1}{2}}\tilde{B}_{u_kh}\tilde{D}_{u_kh}^{-\frac{1}{2}})^{-\frac{1}{2}} \tilde{D}_{u_kh}^{-\frac{1}{2}} B_{u_kh} \tilde{D}_{u_kh}^{-\frac{1}{2}} (\tilde{D}_{u_kh}^{-\frac{1}{2}}\tilde{B}_{u_kh}\tilde{D}_{u_kh}^{-\frac{1}{2}})^{-\frac{1}{2}}\\
     & =  (\tilde{D}_{u_kh}^{-\frac{1}{2}}\tilde{B}_{u_kh}\tilde{D}_{u_kh}^{-\frac{1}{2}})^{-\frac{1}{2}} \tilde{D}_{u_kh}^{-\frac{1}{2}} \tilde{B}_{u_kh} \tilde{D}_{u_kh}^{-\frac{1}{2}} (\tilde{D}_{u_kh}^{-\frac{1}{2}}\tilde{B}_{u_kh}\tilde{D}_{u_kh}^{-\frac{1}{2}})^{-\frac{1}{2}}\\
      &\quad+ (\tilde{D}_{u_kh}^{-\frac{1}{2}}\tilde{B}_{u_kh}\tilde{D}_{u_kh}^{-\frac{1}{2}})^{-\frac{1}{2}} \tilde{D}_{u_kh}^{-\frac{1}{2}} \left(B_{u_kh} - \tilde{B}_{u_kh}\right) \tilde{D}_{u_kh}^{-\frac{1}{2}} (\tilde{D}_{u_kh}^{-\frac{1}{2}}\tilde{B}_{u_kh}\tilde{D}_{u_kh}^{-\frac{1}{2}})^{-\frac{1}{2}}\\
      & = I_{N_{\ell,d}} + (\tilde{D}_{u_kh}^{-\frac{1}{2}}\tilde{B}_{u_kh}\tilde{D}_{u_kh}^{-\frac{1}{2}})^{-\frac{1}{2}} \tilde{D}_{u_kh}^{-\frac{1}{2}} \left(B_{u_kh} - \tilde{B}_{u_kh}\right) \tilde{D}_{u_kh}^{-\frac{1}{2}} (\tilde{D}_{u_kh}^{-\frac{1}{2}}\tilde{B}_{u_kh}\tilde{D}_{u_kh}^{-\frac{1}{2}})^{-\frac{1}{2}}.
\end{split}
\end{equation*}
Since $B_{u_kh} - \tilde{B}_{u_kh}$ is a diagonal matrix with non-positive elements, the diagonal entries of
$$
(\tilde{D}_{u_kh}^{-\frac{1}{2}}\tilde{B}_{u_kh}\tilde{D}_{u_kh}^{-\frac{1}{2}})^{-\frac{1}{2}} \tilde{D}_{u_kh}^{-\frac{1}{2}} \left(B_{u_kh} - \tilde{B}_{u_kh}\right) \tilde{D}_{u_kh}^{-\frac{1}{2}} (\tilde{D}_{u_kh}^{-\frac{1}{2}}\tilde{B}_{u_kh}\tilde{D}_{u_kh}^{-\frac{1}{2}})^{-\frac{1}{2}}
$$
are also non-positive. Consequently, the diagonal entries of the covariance matrix of $\tilde{\gamma}_{u_kh}$ do not exceed 1. Therefore, the random vector $\tilde{\gamma}_{u_kh}$ consists of $N_{\ell,d}$ $(K_{\max}\sigma)$-sub-Gaussian variables.

Based on (\ref{eq:juqi1}), the term (\ref{eq:huanle2}) can be rewritten and upper bounded as follows:
\begin{equation}\label{eq:xianbai5}
  \begin{split}
     \mathrm{(\ref{eq:huanle2})} & = \frac{1}{\sqrt{nh^d}} \left\| \tilde{D}_{u_kh}^{-\frac{1}{2}} (\tilde{D}_{u_kh}^{-\frac{1}{2}}\tilde{B}_{u_kh}\tilde{D}_{u_kh}^{-\frac{1}{2}})^{-\frac{1}{2}} \tilde{\gamma}_{u_kh} \right\|_{\infty}\\
       & \leq \frac{1}{\sqrt{nh^d}} \left\| \tilde{D}_{u_kh}^{-\frac{1}{2}} (\tilde{D}_{u_kh}^{-\frac{1}{2}}\tilde{B}_{u_kh}\tilde{D}_{u_kh}^{-\frac{1}{2}})^{-\frac{1}{2}} \tilde{\gamma}_{u_kh} \right\|\\
       & \leq\frac{1}{\sqrt{nh^d}} \lambda_{\max}\left[ \tilde{D}_{u_kh}^{-\frac{1}{2}} (\tilde{D}_{u_kh}^{-\frac{1}{2}}\tilde{B}_{u_kh}\tilde{D}_{u_kh}^{-\frac{1}{2}})^{-\frac{1}{2}} \right] \left\| \tilde{\gamma}_{u_kh} \right\|\\
       & \leq\sqrt{\frac{N_{\ell,d}}{nh^d}} \lambda_{\max}\left[ \tilde{D}_{u_kh}^{-\frac{1}{2}} (\tilde{D}_{u_kh}^{-\frac{1}{2}}\tilde{B}_{u_kh}\tilde{D}_{u_kh}^{-\frac{1}{2}})^{-\frac{1}{2}} \right] \left\| \tilde{\gamma}_{u_kh} \right\|_{\infty},
  \end{split}
\end{equation}
where the first inequality follows from $\| \cdot \|_{\infty} \leq \| \cdot \|$, and the last inequality follows from $\| \cdot \| \leq \sqrt{N_{\ell,d}} \| \cdot \|_{\infty}$.

The main challenge here is to upper bound $\lambda_{\max}[ \tilde{D}_{u_kh}^{-1/2} (\tilde{D}_{u_kh}^{-1/2}\tilde{B}_{u_kh}\tilde{D}_{u_kh}^{-1/2})^{-1/2} ]$ in (\ref{eq:xianbai5}). We observe that
\begin{equation}\label{eq:wuban1}
  \begin{split}
     &\lambda_{\max}\left[ \tilde{D}_{u_kh}^{-\frac{1}{2}} (\tilde{D}_{u_kh}^{-\frac{1}{2}}\tilde{B}_{u_kh}\tilde{D}_{u_kh}^{-\frac{1}{2}})^{-\frac{1}{2}} \right]    \leq \lambda_{\max} \left( \tilde{D}_{u_kh}^{-\frac{1}{2}} \right) \lambda_{\max} \left[ \left(\tilde{D}_{u_kh}^{-\frac{1}{2}}\tilde{B}_{u_kh}\tilde{D}_{u_kh}^{-\frac{1}{2}}\right)^{-\frac{1}{2}} \right]\\
       & = \lambda_{\min}^{-\frac{1}{2}}\left( \tilde{D}_{u_kh} \right) \lambda_{\max}\left[ \left(\tilde{D}_{u_kh}^{-\frac{1}{2}}\tilde{B}_{u_kh}\tilde{D}_{u_kh}^{-\frac{1}{2}}\right)^{-1}  \left(\tilde{D}_{u_kh}^{-\frac{1}{2}}\tilde{B}_{u_kh}\tilde{D}_{u_kh}^{-\frac{1}{2}}\right)^{\frac{1}{2}} \right]\\
       & \leq \tau^{-\frac{1}{2}} \lambda_{\max}\left[ \left(\tilde{D}_{u_kh}^{-\frac{1}{2}}\tilde{B}_{u_kh}\tilde{D}_{u_kh}^{-\frac{1}{2}}\right)^{-1}  \right] \lambda_{\max}\left[   \left(\tilde{D}_{u_kh}^{-\frac{1}{2}}\tilde{B}_{u_kh}\tilde{D}_{u_kh}^{-\frac{1}{2}}\right)^{\frac{1}{2}}  \right],
  \end{split}
\end{equation}
where the last inequality follows from the fact $\lambda_{\min}( \tilde{D}_{u_kh} ) \geq \tau$. To proceed, we establish bounds for each term in (\ref{eq:wuban1}) individually. First,
\begin{equation}\label{eq:kaoya1}
\begin{split}
\lambda_{\max}\left[ \left(\tilde{D}_{u_kh}^{-\frac{1}{2}}\tilde{B}_{u_kh}\tilde{D}_{u_kh}^{-\frac{1}{2}}\right)^{-1}  \right] & \leq \lambda_{\max}\left( \tilde{D}_{u_kh} \right) \lambda_{\min}^{-1}\left( \tilde{B}_{u_kh} \right) \leq \left( \frac{K_{\max}}{h^d} + \tau \right)\frac{1}{\tau}.
\end{split}
\end{equation}
The proof of (\ref{eq:kaoya1}) is straightforward. Next, we bound
\begin{equation}\label{eq:kaoya2}
  \begin{split}
     \lambda_{\max}\left[   \left(\tilde{D}_{u_kh}^{-\frac{1}{2}}\tilde{B}_{u_kh}\tilde{D}_{u_kh}^{-\frac{1}{2}}\right)^{\frac{1}{2}}  \right] & = \lambda_{\max}^{\frac{1}{2}}    \left(\tilde{D}_{u_kh}^{-\frac{1}{2}}\tilde{B}_{u_kh}\tilde{D}_{u_kh}^{-\frac{1}{2}}\right)  \leq \sqrt{2N_{\ell,d}},
  \end{split}
\end{equation}
The proof of (\ref{eq:kaoya2}) uses the fact that each entry of  $\tilde{D}_{u_kh}^{-1/2}\tilde{B}_{u_kh}\tilde{D}_{u_kh}^{-1/2}$ has absolute value less than $2$. To see this, note that $D_{u_kh}$ is a diagonal matrix, and its smallest eigenvalue is equal to its smallest diagonal entry. Therefore,
$$
\tilde{D}_{u_kh}(s,s) \geq D_{u_kh}(s,s) \vee \tau \geq \frac{D_{u_kh}(s,s) + \tau}{2},
$$
and we also have
$$
\tilde{B}_{u_kh}(s,s) \leq B_{u_kh}(s,s) + \tau = D_{u_kh}(s,s) + \tau.
$$
Thus, the diagonal entries of $\tilde{D}_{u_kh}^{-1/2}\tilde{B}_{u_kh}\tilde{D}_{u_kh}^{-1/2}$ are bounded by $2$. Additionally, the $st$-th off-diagonal entry of $\tilde{D}_{u_kh}^{-1/2}\tilde{B}_{u_kh}\tilde{D}_{u_kh}^{-1/2}$ is
$$
\frac{\tilde{B}_{u_kh}(s,t)}{\tilde{D}_{u_kh}^{1/2}(s,s)\tilde{D}_{u_kh}^{1/2}(t,t)} \leq \frac{B_{u_kh}(s,t)}{D_{u_kh}^{1/2}(s,s)D_{u_kh}^{1/2}(t,t)},
$$
whose absolute value is less than $1$ by the Cauchy–Schwarz inequality.

Inserting (\ref{eq:kaoya1})--(\ref{eq:kaoya2}) into the right-hand side of (\ref{eq:wuban1}), we obtain
\begin{equation}\label{eq:kaoya3}
  \begin{split}
     \lambda_{\max}\left[ \tilde{D}_{u_kh}^{-\frac{1}{2}} (\tilde{D}_{u_kh}^{-\frac{1}{2}}\tilde{B}_{u_kh}\tilde{D}_{u_kh}^{-\frac{1}{2}})^{-\frac{1}{2}} \right] & \lesssim \tau^{-\frac{1}{2}} \left( \frac{K_{\max}}{h^d} + \tau \right)\frac{1}{\tau} \lesssim \frac{1}{\tau^{\frac{3}{2}}h^d}.\\
  \end{split}
\end{equation}
Combining (\ref{eq:xianbai5}) with (\ref{eq:kaoya3}), we get
\begin{equation}\label{eq:xianbai444}
  \mathrm{(\ref{eq:huanle2})} \lesssim \frac{\left\| \tilde{\gamma}_{u_kh} \right\|_{\infty}}{n^{\frac{1}{2}}\tau^{\frac{3}{2}}h^\frac{3d}{2}} = \frac{\left\| \tilde{\gamma}_{u_kh} \right\|_{\infty}}{(\tau nh^d)^{\frac{1}{2}}\tau h^d}.
\end{equation}

\underline{\textsc{Constructing the final bound}.} The upper bounds for the terms (\ref{eq:huanle1}) and (\ref{eq:huanle2}) are established in (\ref{eq:xianbai4})--(\ref{eq:xianbai444}), respectively. We define $\tilde{\gamma}_{nx} \triangleq (\tilde{\gamma}_{u_kh}^{\top})^{\top}_{k \in I_{x}}$, which contains $\mathrm{Card}(I_{x})N_{\ell,d}$ $(K_{\max}\sigma)$-sub-Gaussian random variables. By combining the bounds (\ref{eq:xianbai4})--(\ref{eq:xianbai444}) with (\ref{eq:qian1}), we immediately find that
\begin{equation}\label{eq:xianbai4444}
  \sup_{k \in I_x}\sup_{x' \in \Omega_k}\left| \tilde{f}_{u_kh}(x') \right| \lesssim \sup_{k \in I_x}\left\| \tilde{\theta}_{u_kh} \right\|_{\infty} \lesssim \frac{1}{\tau h^d} + \frac{\left\| \tilde{\gamma}_{nx} \right\|_{\infty}}{(\tau nh^d)^{\frac{1}{2}}\tau h^d}.
\end{equation}
Thus, the first term in (\ref{eq:gift1}) can be upper bounded by
\begin{equation}\label{eq:xiayule1}
  \begin{split}
    & \mathbb{E}_{\otimes^n}\left[\sup_{k \in I_x}\sup_{x' \in \Omega_k}\left| \tilde{f}_{u_kh}(x') \right|\right]^q1_{\{ \mathcal{C}_h^c \}} \lesssim \mathbb{E}_{\otimes^n}\left[\frac{1}{(\tau h^d)^{q}} + \frac{\| \tilde{\gamma}_{nx}\|_{\infty}^q}{(\tau nh^d)^{\frac{q}{2}}(\tau h^d)^q}\right]1_{\{ \mathcal{C}_h^c \}}\\
     & =\mathbb{E}_{\otimes^n} \left\{ \mathbb{E}_{\otimes^n}\left[\frac{1}{(\tau h^d)^{q}} + \frac{\| \tilde{\gamma}_{nx}\|_{\infty}^q}{(\tau nh^d)^{\frac{q}{2}}(\tau h^d)^q}\right]1_{\{ \mathcal{C}_h^c \}} \mid \mathfrak{X}_n \right\} \\
     & \leq \mathbb{E}_{\otimes^n}  \left[\frac{1}{(\tau h^d)^{q}} + \frac{\mathbb{E}_{\otimes^n}\left(\| \tilde{\gamma}_{nx}\|_{\infty}^q\mid \mathfrak{X}_n\right)}{(\tau nh^d)^{\frac{q}{2}}(\tau h^d)^q}\right]  \mathbb{P}_{\otimes^n}\left( \mathcal{C}_h^c \right) \\
       & \lesssim \left[\frac{1}{(\tau h^d)^{q}} + \frac{\left[\log (rM +1) \right]^{\frac{q}{2}}}{(\tau nh^d)^{\frac{q}{2}}(\tau h^d)^q}\right]\exp\left( - C nh^d \right)\\
       & \lesssim \exp\left( - C nh^d \right),
  \end{split}
\end{equation}
where the third inequality follows from Lemma~\ref{lem:sub_Gaussian_bound} and the fact that $\mathrm{Card}(I_{x})N_{\ell,d} \lesssim ( rM +1)^d$. The last inequality holds because, under the conditions $M/n^{\gamma_1} \to 0$, $nh^d/n^{\gamma_2} \to \infty$, and $\tau n^{\gamma_3} \to \infty$ (with $0<\gamma_1,\gamma_2,\gamma_3 <\infty$), the penultimate term in (\ref{eq:xiayule1}) is dominated by the exponentially small probability.

By substituting (\ref{eq:xiayule1}) into the right-hand side of (\ref{eq:gift1}), we obtain
\begin{equation}\label{eq:xiayule2}
  \mathbb{E}_{\otimes^n}\sup_{k \in I_x}\sup_{x' \in \Omega_k}\left| \tilde{f}_{u_kh}(x') - f_{u_k}(x') \right|^q1_{\{ \mathcal{C}_h^c  \}} \lesssim \exp\left( - C nh^d \right).
\end{equation}

\subsubsection{Completing the proof}

Inserting (\ref{eq:xiayule3}) and (\ref{eq:xiayule2}) into (\ref{eqn:line-1})--(\ref{eqn:line-2}) gives
\begin{equation}\label{eq:xiayule4}
  \mathbb{E}_{\otimes^n}\sup_{k \in I_x}\sup_{x' \in \Omega_k}\left| \tilde{f}_{u_kh}(x') - f_{u_k}(x') \right|^q \lesssim h^{q\beta} +\left[ \frac{\log (rM +1) }{nh^d}\right]^{\frac{q}{2}}.
\end{equation}
Combining (\ref{eq:0}), (\ref{eq:1}), (\ref{eq:2}), and (\ref{eq:rela2}) with (\ref{eq:xiayule4}), we obtain
\begin{equation}\label{eq:xiongxiong1}
  R_{A,q}(\hat{f}_{\mathrm{PP}}, f) \lesssim r^{q(1 \wedge \beta)} + \frac{1}{M^{q\beta}} + h^{q\beta} + \left[ \frac{\log (rM +1) }{nh^d}\right]^{\frac{q}{2}}
\end{equation}
under the conditions $M h \gtrsim 1$, $M/n^{\gamma_1} \to 0$, $nh^d/n^{\gamma_2} \to \infty$, and $\tau n^{\gamma_3} \to \infty$, where $0<\gamma_1,\gamma_2,\gamma_3 <\infty$. Given the choices $h \asymp r \vee n^{-\frac{1}{2\beta+d}}$ and $M \asymp 1/h \asymp r^{-1} \wedge n^{\frac{1}{2\beta+d}}$, we find
\begin{equation}
\begin{split}
   R_{A,q}(\hat{f}_{\mathrm{PP}}, f) & \lesssim r^{q(1 \wedge \beta)} + r^{q\beta} + n^{-\frac{q\beta}{2\beta+d}}  + \left(\frac{1}{nh^d}\right)^{\frac{q}{2}}\\
     & \lesssim r^{q(1 \wedge \beta)} + r^{q\beta} + n^{-\frac{q\beta}{2\beta+d}} + n^{-\frac{q\beta}{2\beta+d}} \\
     & \lesssim r^{q(1 \wedge \beta)}  + n^{-\frac{q\beta}{2\beta+d}}. \\
\end{split}
\end{equation}
This completes the proof of Theorem~\ref{theo:upper} for $1 \leq q <\infty$.

\subsection{Proof of the upper bound for $q=\infty$}\label{sec:proof_upper_2}

In this subsection, we complete the proof of Theorem~\ref{theo:upper} for the adversarial sup-norm risk. This proof relies on techniques similar to those developed in Section~\ref{sec:proof:buxiang2}. The adversarial sup-norm risk can be upper bounded by
\begin{equation}\label{eq:exinnv1}
  \begin{split}
     R_{A,\infty}(\hat{f}_{\mathrm{PP}}, f) & = \mathbb{E}_{\otimes^n}\sup_{x \in \Omega}\sup_{x' \in A(x)}\left| \hat{f}_{\mathrm{PP}}(x') - f(x) \right| \\
       & \leq \mathbb{E}_{\otimes^n}\sup_{x \in \Omega}\sup_{x' \in A(x)}\left| \hat{f}_{\mathrm{PP}}(x') -f(x')\right|+ \sup_{x \in \Omega}\sup_{x' \in A(x)}\left| f(x') - f(x) \right|\\
       & \lesssim \mathbb{E}_{\otimes^n}\sup_{x \in \Omega}\sup_{x' \in A(x)}\left| \hat{f}_{\mathrm{PP}}(x') -f(x')\right| + r^{1 \wedge \beta},
  \end{split}
\end{equation}
where the last inequality follows from Lemma~\ref{lem:lip}. Thus, it remains to upper bound the localized sup-norm risk $\mathbb{E}_{\otimes^n}\sup_{x \in \Omega}\sup_{x' \in A(x)}| \hat{f}_{\mathrm{PP}}(x') -f(x')|$.

Following a similar approach to (\ref{eq:rela2}), we have
\begin{equation}\label{eq:exinnv2}
  \begin{split}
     &\mathbb{E}_{\otimes^n}\sup_{x \in \Omega}\sup_{x' \in A(x)}\left| \hat{f}_{\mathrm{PP}}(x') -f(x')\right| \leq \mathbb{E}_{\otimes^n}\sup_{k \in \{1,\ldots,M^d \}}\sup_{x' \in \Omega_k}\left| \hat{f}_{\mathrm{PP}}(x') -f(x')\right|\\
       & \leq \mathbb{E}_{\otimes^n}\sup_{k \in \{1,\ldots,M^d \}}\sup_{x' \in \Omega_k}\left| \hat{f}_{\mathrm{PP}}(x') -f_{u_k}(x')\right| + \sup_{k \in \{1,\ldots,M^d \}}\sup_{x' \in \Omega_k}\left| f_{u_k}(x') - f(x')\right|\\
       & \lesssim \mathbb{E}_{\otimes^n}\sup_{k \in \{1,\ldots,M^d \}}\sup_{x' \in \Omega_k}\left| \hat{f}_{\mathrm{PP}}(x') -f_{u_k}(x')\right| + \frac{1}{M^{\beta}},
  \end{split}
\end{equation}
where the last inequality follows from Lemma~\ref{lem:taylor}. Then, using the same technique as in Section~\ref{sec:proof:buxiang2}, we can show that
\begin{equation}\label{eq:exinnv3}
  \mathbb{E}_{\otimes^n}\sup_{k \in \{1,\ldots,M^d \}}\sup_{x' \in \Omega_k}\left| \hat{f}_{\mathrm{PP}}(x') -f_{u_k}(x')\right| \leq h^{\beta} + \left[\frac{\log (M +1)}{nh^d}\right]^{\frac{1}{2}}
\end{equation}
when $Mh \gtrsim 1$ and the same conditions used in Section~\ref{sec:proof:buxiang2} hold. Combining (\ref{eq:exinnv1}), (\ref{eq:exinnv2}), and (\ref{eq:exinnv3}), we derive the following bound:
\begin{equation}\label{eq:exinnv4}
\begin{split}
   R_{A,\infty}(\hat{f}_{\mathrm{PP}}, f) & \lesssim r^{1 \wedge \beta} + \frac{1}{M^{\beta}} + h^{\beta} + \left[\frac{\log (M +1)}{nh^d}\right]^{\frac{1}{2}} \\
     & \lesssim r^{1 \wedge \beta} + h^{\beta} + \left[\frac{\log (M +1)}{nh^d}\right]^{\frac{1}{2}}\\
     & \lesssim r^{1 \wedge \beta} + r^{\beta} + \left(\frac{n}{\log n}\right)^{-\frac{\beta}{2\beta+d}},
\end{split}
\end{equation}
where the inequalities follow from $M \asymp 1/h$ and $h \asymp r \vee (n/\log n)^{-\frac{1}{2\beta+d}}$.
Thus, the results in Theorem~\ref{theo:upper} follow immediately.

\section{Proof of Theorem~\ref{theo:adaptive1}}\label{sec:proof:buxiang}

\subsection{Notation}

The proof in this section will make repeated use of the notations introduced in Section~\ref{sec:proof_notation_1}, particularly the matrix notations (\ref{eq:Q_uh})--(\ref{eq:eta_uh}). Additionally, we recall the diagonal matrix $R_{h'h} \in \mathbb{R}^{N_{\ell,d}}$ defined in (\ref{eq:jiaquan3}). Given a point $u_k$ and bandwidths $h' \leq h$, we define another diagonal matrix $K_{h'h} \in \mathbb{R}^n$ as
\begin{equation}\label{eq:dongdu1}
  K_{h'h} \triangleq \mathrm{diag}\left[\left(\frac{\sqrt{K\left( \frac{X_i-u_k}{h'} \right)}}{\sqrt{K\left( \frac{X_i-u_k}{h} \right)}}\right)_{1 \leq i \leq n} \right],
\end{equation}
where the ratios involving $0/0$ are defined to be $0$ in the above matrix. Note that under Assumption~\ref{ass:kernel_2}, $K_{h'h}$ is well-defined since for any $h' \leq h$, we have $\| \frac{X_i-u_k}{h} \| \leq \| \frac{X_i-u_k}{h'} \|$, implying $0 \leq K( \frac{X_i-u_k}{h'} ) \leq K( \frac{X_i-u_k}{h} )$. Therefore, dividing a non-zero term by zero will never happen.

We further define
\begin{equation}\label{eq:oracle_h}
  \bar{h} \triangleq \argmax_{h \in H}\left\{ h: h^{2\beta} \leq \frac{\log n}{nh^d}  \right\},
\end{equation}
as the oracle bandwidth within the grid $H$, where $H$ is defined by (\ref{eq:H_set}).

\subsection{Technical lemma}

Recall that $\mathfrak{X}_n$ denotes the $\sigma$-algebra generated by the design data. The following lemma restates Lemma~\ref{lem:event} specifically for $h = \bar{h}$.

\begin{lemma}\label{lem:C_n_bar}
  For $k=1,\ldots,M^d$, define the events $\mathcal{A}_{u_k\bar{h}} \triangleq \{ 0 <n_{u_k\bar{h}} \leq c_1 n \bar{h}^d\}$, $\mathcal{B}_{u_k\bar{h}} \triangleq \{\lambda_{\min}(B_{u_k\bar{h}}) \geq c_2 \}$, and $\mathcal{D}_{u_k\bar{h}} \triangleq \{ c_3 \leq  \lambda_{\min}(D_{u_k\bar{h}}) \leq  \lambda_{\max}(D_{u_k\bar{h}}) \leq c_4\}$, where $c_i,i=1,\ldots,4$, are positive constants. If $M \asymp n$, then for sufficiently large $n$, there exists an event $\mathcal{C}_{\bar{h}}$ measurable with respect to $\mathfrak{X}_n$ such that
  \begin{equation*}
    \mathcal{C}_{\bar{h}} \subseteq \bigcap_{k=1}^{M^d}\left(\mathcal{A}_{u_k\bar{h}} \cap \mathcal{B}_{u_k\bar{h}} \cap \mathcal{D}_{u_k\bar{h}}\right)
  \end{equation*}
  and
  \begin{equation*}
    \mathbb{P}_{\otimes^ n}\left( \mathcal{C}_{\bar{h}} \right) \geq 1 - \exp( - C n\bar{h}^d ),
  \end{equation*}
  where $C>0$ is a constant not related to $n$.
\end{lemma}

\begin{proof}[Proof of Lemma~\ref{lem:C_n_bar}]
  To prove this lemma, it suffices to verify that $n \bar{h}^d$ grows faster than a polynomial function of $n$. We begin by showing that for sufficiently large $n$, the oracle bandwidth $\bar{h}$ exists. To prove this existence, it is enough to show that $h_{-J}^{2\beta+d} = o(\log n / n)$. Based on the definitions in (\ref{eq:exinnv5}), we have
  \begin{equation}\label{eq:exinnv6}
    \beta_{-J} \sim \left( 1 + \frac{1}{\log n} \right)^{-2  \log n \log\log n  } \sim \left( \frac{1}{\log n} \right)^2.
  \end{equation}
  Thus,
  \begin{equation*}
    \begin{split}
       \frac{h_{-J}^{2\beta+d}}{\log n/ n} & = \frac{n^{-\frac{2\beta+d}{2\beta_{-J}+d}}}{\log n/ n} = \frac{n^{\frac{2\beta_{-J} - 2\beta}{2\beta_{-J}+d}}}{\log n} \to 0, \\
    \end{split}
  \end{equation*}
  where the last approximation follows from (\ref{eq:exinnv6}) and holds for any fixed $\beta > 0$. This confirms the existence of $\bar{h}$ when $n$ is sufficiently large.

  We now assume that $\bar{h}$ can be expressed as
  \begin{equation*}
    \bar{h} = h_{\bar{j}} = n^{-\frac{1}{2\beta_{\bar{j}}+d}},
  \end{equation*}
  where $\bar{j} \in \{ - J,\ldots, J_{\max} \}$ and may depend on the sample size $n$. In the case where $\bar{j} = J_{\max}$, we have
  \begin{equation*}
    \beta_{J_{\max}} \gtrsim \left( 1 + \frac{1}{\log n} \right)^{\log n \log \beta_{\max}} \sim \beta_{\max},
  \end{equation*}
  and
  \begin{equation*}
    n \bar{h}^d = n \times n^{-\frac{d}{2\beta_{J_{\max}}+d}} \gtrsim n^{1-\frac{d}{2\beta_{\max}+d}} = n^{\frac{2\beta_{\max}}{2\beta_{\max}+d}}.
  \end{equation*}
  This yields the desired result.

  When $\bar{j} \in \{ - J,\ldots, J_{\max}-1 \}$, we have the following relation according to the definition in (\ref{eq:oracle_h}):
  \begin{equation}\label{eq:exinnv7}
    h_{\bar{j}}^{d} = n^{-\frac{d}{2\beta_{\bar{j}}+d}} \leq \left(\frac{\log n}{n} \right)^{\frac{d}{2\beta+d}} < n^{-\frac{d}{2\beta_{\bar{j}+1}+d}}.
  \end{equation}
    Note that
  \begin{equation*}
    \begin{split}
       \frac{d}{2\beta_{\bar{j}}+d} - \frac{d}{2\beta_{\bar{j}+1}+d} & = \frac{2d(\beta_{\bar{j}+1} - \beta_{\bar{j}} )}{(2\beta_{\bar{j}}+d)(2\beta_{\bar{j}+1}+d)} = \frac{2d\beta_{\bar{j}}( 1 + \frac{1}{\log n} - 1 )}{(2\beta_{\bar{j}}+d)(2\beta_{\bar{j}+1}+d)}\\
         & = \frac{2d\beta_{\bar{j}}}{(2\beta_{\bar{j}}+d)(2\beta_{\bar{j}+1}+d)\log n } \leq \frac{1}{\log n},
    \end{split}
  \end{equation*}
  where the second equality follows from the definition in (\ref{eq:exinnv5}). Then, we have
  \begin{equation}\label{eq:exinnv8}
    \begin{split}
       \frac{n^{-\frac{d}{2\beta_{\bar{j}}+d}}}{n^{-\frac{d}{2\beta_{\bar{j}+1}+d}}} & = n^{-\left( \frac{d}{2\beta_{\bar{j}}+d} - \frac{d}{2\beta_{\bar{j}+1}+d} \right)} \geq n^{-\frac{1}{\log n}} = \exp(-1).\\
    \end{split}
  \end{equation}
  Thus, we obtain
  \begin{equation*}
    n h_{\bar{j}}^{d} = n \times n^{-\frac{d}{2\beta_{\bar{j}}+d}} \geq n \times n^{-\frac{d}{2\beta_{\bar{j}+1}+d}} \times \exp(-1) \gtrsim n \times \left(\frac{\log n}{n} \right)^{\frac{d}{2\beta+d}} \gtrsim n^{\frac{2\beta}{2\beta+d}},
  \end{equation*}
  where the first inequality follows from (\ref{eq:exinnv8}) and the second from (\ref{eq:exinnv7}). This verifies that $n \bar{h}^d$ grows faster than a polynomial function of $n$, thus proving the lemma.

\end{proof}

We now turn to investigating the probabilities associated with the bandwidth selection procedure in (\ref{eq:select:band}). Recall that $\Omega_1, \ldots, \Omega_{M^d}$ form a partition of $\Omega = [0, 1]^d$ based on the set of discretization points $\Lambda_{M}$. Note that for any $x \in \Omega_k$, the selected bandwidth $\hat{h}_x$ remains constant across $x$. Thus, we denote $\hat{h}_k$ as $\hat{h}_x$ for any $x \in \Omega_k$.

In the following lemma, we analyze the probability of the event
\begin{equation}\label{eq:xiongxiong2}
  \mathcal{E}_{u_k\bar{h}} \triangleq \left\{ \hat{h}_{k} \geq \bar{h} \right\},
\end{equation}
where $\bar{h}$ is the oracle bandwidth defined in (\ref{eq:oracle_h}).

\begin{lemma}\label{lem:E_n}

Suppose that $M \asymp n$. Define the event
  \begin{equation}\label{eq:yeyan1}
    \mathcal{E}_{\bar{h}} \triangleq \bigcap_{k=1}^{M^d} \mathcal{E}_{u_k\bar{h}}.
  \end{equation}
  Then, when the constant $C$ in (\ref{eq:select:band}) is chosen to be sufficiently large, there exists a large enough constant $C'$ such that
  \begin{equation}\label{eq:dongdu4}
    \mathbb{P}_{\otimes^n} \left( \mathcal{E}_{\bar{h}}^c  | \mathfrak{X}_n \right) \leq n^{-C'}
  \end{equation}
  if the event $\mathcal{C}_{\bar{h}}$ holds.

\end{lemma}

\begin{proof}

Note that
\begin{equation}\label{eq:exinnv9}
  \mathbb{P}_{\otimes^n} \left( \mathcal{E}_{\bar{h}}^c  | \mathfrak{X}_n \right) = \mathbb{P}_{\otimes^n} \left( \cup_{k=1}^{M^d} \mathcal{E}_{u_k\bar{h}}^c  | \mathfrak{X}_n \right) \leq M^d \max_{1 \leq k \leq M^d} \mathbb{P}_{\otimes^n} \left(  \mathcal{E}_{u_k\bar{h}}^c  | \mathfrak{X}_n \right).
\end{equation}
To establish this lemma, it is sufficient to bound the probability of the event $\mathcal{E}_{u_k\bar{h}}^c$ when $\mathcal{C}_{\bar{h}}$ holds. By the definition of the bandwidth selection criterion (\ref{eq:select:band}), $\hat{h}_{k}$ is the largest bandwidth satisfying (\ref{eq:select:band}). When the event $\mathcal{E}_{u_k\bar{h}}^c$ occurs, we have $\{ \hat{h}_{k} < \bar{h} \}$. This implies that $\bar{h}$ is not the largest bandwidth satisfying (\ref{eq:select:band}), leading to the occurrence of the following event $\mathcal{G}_{u_k}$:
  \begin{equation}\label{eq:event_G}
    \mathcal{G}_{u_k} \triangleq \left\{ \exists h' \in H, h' \leq \bar{h}, \left\| \tilde{D}_{u_kh'}^{-\frac{1}{2}}\left(  a_{u_kh'} - B_{u_kh'}R_{h'\bar{h}}\tilde{B}_{u_k\bar{h}}^{-1}a_{u_k\bar{h}}\right) \right\|_{\infty}  >  C\sqrt{\frac{\log n}{nh'^d}} \right\}.
  \end{equation}
Thus, it suffices to bound the probability of $\mathcal{G}_{u_k}$.

When the event $\mathcal{C}_{\bar{h}}$ holds, we have $\tilde{B}_{u_k\bar{h}} = B_{u_k\bar{h}}$. Using the matrix notations introduced in (\ref{eq:Q_uh}), (\ref{eq:e_uh}), and (\ref{eq:eta_uh}), the vector $a_{u_kh'} - B_{u_kh'}R_{h'\bar{h}}B_{u_k\bar{h}}^{-1}a_{u_k\bar{h}}$ can be expressed as
\begin{equation}\label{eq:dongdu6}
  \begin{split}
  & a_{u_k h'} - B_{u_kh'}R_{h'\bar{h}}B_{u_k\bar{h}}^{-1}a_{u_k\bar{h}}\\
       & =\frac{1}{\sqrt{nh'^d}} Q_{u_kh'} \left( \eta_{u_kh'} + e_{u_kh'} \right)\\
        &\quad- \frac{1}{\sqrt{n\bar{h}^d}}Q_{u_kh'}Q_{u_kh'}^{\top} R_{h'\bar{h}}\left( Q_{u_k\bar{h}}Q_{u_k\bar{h}}^{\top} \right)^{-1}Q_{u_k\bar{h}} \left( \eta_{u_k\bar{h}} + e_{u_k\bar{h}} \right) \\
       & =\frac{1}{\sqrt{nh'^d}} Q_{u_kh'} \left( \eta_{u_kh'} + e_{u_kh'} \right)\\
        &\quad- \frac{1}{\sqrt{nh'^d}}Q_{u_kh'}K_{h'\bar{h}}Q_{u_k\bar{h}}^{\top}\left( Q_{u_k\bar{h}}Q_{u_k\bar{h}}^{\top} \right)^{-1} Q_{u_k\bar{h}} \left( \eta_{u_k\bar{h}} + e_{u_k\bar{h}} \right)\\
       & = \frac{1}{\sqrt{nh'^d}} Q_{u_kh'} K_{h'\bar{h}} \left[ \eta_{u_k\bar{h}} - Q_{u_k\bar{h}}^{\top}\left( Q_{u_k\bar{h}}Q_{u_k\bar{h}}^{\top} \right)^{-1} Q_{u_k\bar{h}} \eta_{u_k\bar{h}} \right]\\
       &\quad + \frac{1}{\sqrt{nh'^d}} Q_{u_kh'} e_{u_kh'} - \frac{1}{\sqrt{nh'^d}}Q_{u_kh'}K_{h'\bar{h}}Q_{u_k\bar{h}}^{\top}\left( Q_{u_k\bar{h}}Q_{u_k\bar{h}}^{\top} \right)^{-1} Q_{u_k\bar{h}} e_{u_k\bar{h}},
  \end{split}
\end{equation}
where the first equality follows from $B_{uh} = Q_{uh} Q_{uh}^{\top}$ and (\ref{eq:japan1}), the second equality follows from
\begin{equation*}
  \begin{split}
     &  Q_{u_kh'}^{\top}R_{h'\bar{h}} = \frac{1}{\sqrt{nh^{'d}}} \begin{pmatrix}
U^{\top}\left(\frac{X_1 - u_k}{h'}\right)\sqrt{K\left( \frac{X_1-u_k}{h'} \right)} \\
\vdots \\
U^{\top}\left(\frac{X_n - u_k}{h'}\right)\sqrt{K\left( \frac{X_n-u_k}{h'} \right)}
\end{pmatrix} \diag\left[\left(\frac{h^{'|s|}}{\bar{h}^{|s|}}\right)_{0 \leq |s| \leq \ell}\right]\\
& = \frac{1}{\sqrt{nh^{'d}}} \begin{pmatrix}
U^{\top}\left(\frac{X_1 - u_k}{\bar{h}}\right)\sqrt{K\left( \frac{X_1-u_k}{h'} \right)} \\
\vdots \\
U^{\top}\left(\frac{X_n - u_k}{\bar{h}}\right)\sqrt{K\left( \frac{X_n-u_k}{h'} \right)}
\end{pmatrix}\\
& =\frac{\sqrt{n\bar{h}^d}}{\sqrt{nh^{'d}}}  \mathrm{diag}\left[\left(\frac{\sqrt{K\left( \frac{X_i-u_k}{h'} \right)}}{\sqrt{K\left( \frac{X_i-u_k}{\bar{h}} \right)}}\right)_{1 \leq i \leq n} \right] \frac{1}{\sqrt{n\bar{h}^d}} \begin{pmatrix}
U^{\top}\left(\frac{X_1 - u_k}{\bar{h}}\right)\sqrt{K\left( \frac{X_1-u_k}{\bar{h}} \right)} \\
\vdots \\
U^{\top}\left(\frac{X_n - u_k}{\bar{h}}\right)\sqrt{K\left( \frac{X_n-u_k}{\bar{h}} \right)}
\end{pmatrix}\\
& = \frac{\sqrt{n\bar{h}^d}}{\sqrt{nh^{'d}}}  K_{h'\bar{h}}Q_{u_k\bar{h}}^{\top},
  \end{split}
\end{equation*}
and the third equality follows from
\begin{equation*}
  \begin{split}
     K_{h'\bar{h}} \eta_{u_k\bar{h}}& = \mathrm{diag}\left[\left(\frac{\sqrt{K\left( \frac{X_i-u_k}{h'} \right)}}{\sqrt{K\left( \frac{X_i-u_k}{\bar{h}} \right)}}\right)_{1 \leq i \leq n} \right] \begin{pmatrix}
f(X_1)\sqrt{K\left( \frac{X_1-u_k}{h} \right)} \\
\vdots \\
f(X_n)\sqrt{K\left( \frac{X_n-u_k}{h} \right)}
\end{pmatrix}  = \eta_{u_kh'}.
  \end{split}
\end{equation*}
From (\ref{eq:dongdu6}) and the triangle inequality, we obtain
\begin{equation}\label{eq:dongdu7}
  \begin{split}
       & \left\| \tilde{D}_{u_kh'}^{-\frac{1}{2}}\left(  a_{u_kh'} - B_{u_kh'}R_{h'\bar{h}}\tilde{B}_{u_k\bar{h}}^{-1}a_{u_k\bar{h}}\right) \right\|_{\infty} \\
       & \leq \left\| \frac{1}{\sqrt{nh'^d}} \tilde{D}_{u_kh'}^{-\frac{1}{2}} Q_{u_kh'} K_{h'\bar{h}} \left[ \eta_{u_k\bar{h}} - Q_{u_k\bar{h}}^{\top}\left( Q_{u_k\bar{h}}Q_{u_k\bar{h}}^{\top} \right)^{-1} Q_{u_k\bar{h}} \eta_{u_k\bar{h}} \right] \right\|_{\infty}\\
       & + \left\| \frac{1}{\sqrt{nh'^d}} \tilde{D}_{u_kh'}^{-\frac{1}{2}} Q_{u_kh'} e_{u_kh'} - \frac{1}{\sqrt{nh'^d}}\tilde{D}_{u_kh'}^{-\frac{1}{2}}Q_{u_kh'}K_{h'\bar{h}}Q_{u_k\bar{h}}^{\top}\left( Q_{u_k\bar{h}}Q_{u_k\bar{h}}^{\top} \right)^{-1} Q_{u_k\bar{h}} e_{u_k\bar{h}} \right\|_{\infty}
  \end{split}
\end{equation}

We upper bound the first term on the right-hand side of (\ref{eq:dongdu7}). Using the Cauchy–Schwarz inequality, absolute value of the $s$-th element of the vector involved in the first term can be upper bounded by
\begin{equation*}
  \frac{1}{\sqrt{nh'^d}} \left\{ \tilde{D}_{u_kh'}^{-1}(s,s)  \sum_{j=1}^{n} \left[Q_{u_kh'}(s,j)\right]^2 \right\}^{\frac{1}{2}}\left\| K_{h'\bar{h}} \left[ \eta_{u_k\bar{h}} - Q_{u_k\bar{h}}^{\top}\left( Q_{u_k\bar{h}}Q_{u_k\bar{h}}^{\top} \right)^{-1} Q_{u_k\bar{h}} \eta_{u_k\bar{h}} \right] \right\|.
\end{equation*}
Combining the definitions in (\ref{eq:Q_uh}), (\ref{eq:D}), and (\ref{eq:tilde_D}), we know
\begin{equation*}
  \sum_{j=1}^{n} \left[Q_{u_kh'}(s,j)\right]^2 = \frac{1}{nh'^d}\sum_{i=1}^{n} \frac{1}{(s!)^2}\left(\frac{X_i - u_k}{h'} \right)^{2s} K\left( \frac{X_i-u_k}{h'} \right)
\end{equation*}
and
\begin{equation*}
  \tilde{D}_{u_kh'}^{-1}(s,s)  \sum_{j=1}^{n} \left[Q_{u_kh'}(s,j)\right]^2 \leq 1.
\end{equation*}
For another term, recalling the vector defined in (\ref{eq:t_uh}), we have
\begin{equation*}
  \begin{split}
       & \left\| K_{h'\bar{h}} \left[ \eta_{u_k\bar{h}} - Q_{u_k\bar{h}}^{\top}\left( Q_{u_k\bar{h}}Q_{u_k\bar{h}}^{\top} \right)^{-1} Q_{u_k\bar{h}} \eta_{u_k\bar{h}} \right] \right\| \\
       & \leq \left\|   \eta_{u_k\bar{h}} - Q_{u_k\bar{h}}^{\top}\left( Q_{u_k\bar{h}}Q_{u_k\bar{h}}^{\top} \right)^{-1} Q_{u_k\bar{h}} \eta_{u_k\bar{h}} \right\|\\
       & = \left\| \eta_{u_k\bar{h}}- t_{u_k\bar{h}} - Q_{u_k\bar{h}}^{\top} \left(Q_{u_k\bar{h}}Q_{u_k\bar{h}}^{\top}\right)^{-1} Q_{u_k\bar{h}}\left( \eta_{u_k\bar{h}}- t_{u_k\bar{h}}\right) \right\|\\
       & \leq \left\| \eta_{u_k\bar{h}}- t_{u_k\bar{h}} \right\|
  \end{split}
\end{equation*}
where the first inequality follows from that the elements of $K_{h'\bar{h}}$ are less than one due to Assumption~\ref{ass:kernel_2}, the first equality follows from (\ref{eq:ljinv1}), and the last inequality follows from that $Q_{u_k\bar{h}}^{\top} (Q_{u_k\bar{h}}Q_{u_k\bar{h}}^{\top})^{-1} Q_{u_k\bar{h}}$ is an orthogonal projection matrix. Therefore, we have proved that
\begin{equation*}
  \begin{split}
       \left\| \eta_{u_k\bar{h}}- t_{u_k\bar{h}} \right\|& = \sqrt{\sum_{i=1}^{n}\left[f(X_i) - f_{u_k}(X_i)\right]^2K\left( \frac{X_i-u_k}{\bar{h}} \right)}\\
       & \lesssim \sqrt{\bar{h}^{2\beta} n_{u_k\bar{h}} }  \lesssim \bar{h}^{\beta} \sqrt{ n \bar{h}^d},
  \end{split}
\end{equation*}
where the equality follows from the definitions (\ref{eq:eta_uh})--(\ref{eq:t_uh}), the first inequality follows from (\ref{eq:poly2}), and the second inequality follows from Lemma~\ref{lem:C_n_bar}. Combining the above results, we have proved that the first term on the right-hand side of (\ref{eq:dongdu7}) is upper bounded by
\begin{equation*}
  \frac{\sqrt{n \bar{h}^d}}{\sqrt{nh'^d}}  \bar{h}^{\beta} \leq \sqrt{\frac{\log n}{nh'^d}},
\end{equation*}
where the inequality follows from (\ref{eq:oracle_h}).

Thus, it remains to upper bound the probability of the following event:
\begin{equation}\label{eq:ljinv3}
\begin{split}
  &\mathcal{G}_{u_k}' \triangleq \left\{ \exists h' \in H, h' \leq \bar{h},\left\| \frac{1}{\sqrt{nh'^d}} \tilde{D}_{u_kh'}^{-\frac{1}{2}} Q_{u_kh'} e_{u_kh'} \right.\right.\\
   &\qquad\qquad\left.\left.- \frac{1}{\sqrt{nh'^d}}\tilde{D}_{u_kh'}^{-\frac{1}{2}}Q_{u_kh'}K_{h'\bar{h}}Q_{u_k\bar{h}}^{\top}\left( Q_{u_k\bar{h}}Q_{u_k\bar{h}}^{\top} \right)^{-1} Q_{u_k\bar{h}} e_{u_k\bar{h}} \right\|_{\infty}  >  C_1\sqrt{\frac{\log n}{nh'^d}} \right\},
   \end{split}
\end{equation}
where $C_1$ is a constant that depends on the difference between the constant $C$ in (\ref{eq:select:band}) and the leading constant in the upper bound of the first term in (\ref{eq:dongdu7}). Hence, $C_1$ can be chosen sufficiently large if the constant $C$ in (\ref{eq:select:band}) is sufficiently large.

Recalling the definition in (\ref{eq:e_uh}), we know that the elements of $e_{u_kh'}$ are i.i.d. $K_{\max}\sigma$-sub-Gaussian random variables when conditioned on $\mathfrak{X}_n$. Consequently, the $s$-th element of the vector $\frac{1}{\sqrt{nh'^d}} \tilde{D}_{u_kh'}^{-1/2} Q_{u_kh'} e_{u_kh'}$ is a sub-Gaussian variable with parameter
\begin{equation*}
  \frac{1}{\sqrt{nh'^d}} K_{\max}\sigma.
\end{equation*}
Using the Cauchy–Schwarz inequality, we know that the variance of each element in the vector  $\frac{1}{\sqrt{nh'^d}}\tilde{D}_{u_kh'}^{-1/2}Q_{u_kh'}K_{h'\bar{h}}Q_{u_k\bar{h}}^{\top}( Q_{u_k\bar{h}}Q_{u_k\bar{h}}^{\top} )^{-1} Q_{u_k\bar{h}} e_{u_k\bar{h}}$ is upper bounded by
\begin{equation*}
\begin{split}
     &   \frac{1}{nh'^d}\mathbb{E}_{\otimes}\left[ \left\| K_{h'\bar{h}}Q_{u_k\bar{h}}^{\top}\left( Q_{u_k\bar{h}}Q_{u_k\bar{h}}^{\top} \right)^{-1} Q_{u_k\bar{h}} e_{u_k\bar{h}} \right\|^2 | \mathfrak{X}_n \right] \\
     & \leq \frac{1}{nh'^d}\mathbb{E}_{\otimes}\left[ \left\| Q_{u_k\bar{h}}^{\top}\left( Q_{u_k\bar{h}}Q_{u_k\bar{h}}^{\top} \right)^{-1} Q_{u_k\bar{h}} e_{u_k\bar{h}} \right\|^2 | \mathfrak{X}_n \right]\\
     & \leq \frac{N_{\ell,d}K_{\max}^2\sigma^2}{nh'^d}
\end{split}
\end{equation*}
Let $\zeta_{u_kh'}(s)$ denote the $s$-th entry of the vector
$$
\frac{1}{\sqrt{nh'^d}} \tilde{D}_{u_kh'}^{-\frac{1}{2}} Q_{u_kh'} e_{u_kh'}-\frac{1}{\sqrt{nh'^d}}\tilde{D}_{u_kh'}^{-\frac{1}{2}}Q_{u_kh'}K_{h'\bar{h}}Q_{u_k\bar{h}}^{\top}\left( Q_{u_k\bar{h}}Q_{u_k\bar{h}}^{\top} \right)^{-1} Q_{u_k\bar{h}} e_{u_k\bar{h}}.
$$
Each $\zeta_{u_kh'}(s)$ is a sub-Gaussian random variable with parameter
\begin{equation*}
  \frac{(\sqrt{2N_{\ell,d}}+\sqrt{2})}{\sqrt{nh'^d}} K_{\max}\sigma \triangleq \frac{C_2}{\sqrt{nh'^d}}.
\end{equation*}
To upper bound the probability of the event in (\ref{eq:ljinv3}), we only need to bound the probability of the following event:
\begin{equation*}
  \mathcal{G}_{u_k}'' \triangleq\left\{ \exists h' \in H, h' \leq \bar{h}, \exists 0 \leq |s| \leq N_{\ell,d}, \mathrm{ s.t. } |\zeta_{u_kh'}(s)| \geq \frac{C_1 C_2\sqrt{\log n}}{\sqrt{nh'^d}} \right\}.
\end{equation*}

The probability of $\mathcal{G}_{u_k}''$ can be upper bounded as follows:
\begin{equation*}
  \mathbb{P}_{\otimes^n} \left(  \mathcal{G}_{u_k}''  | \mathfrak{X}_n \right) \leq  \mathrm{Card}(H)N_{\ell,d} \exp\left( - \frac{C^2_1}{2} \log n \right) \lesssim (\log n)^2\left(\frac{1}{n}\right)^{\frac{C^2_1}{2}},
\end{equation*}
where the first inequality comes from the tail behavior of sub-Gaussian variables, and the last inequality uses $\mathrm{Card}(H)\lesssim  (\log n)^2$. Combining this with (\ref{eq:exinnv9}), we get
\begin{equation*}
  \mathbb{P}_{\otimes^n} \left( \mathcal{E}_{\bar{h}}^c  | \mathfrak{X}_n \right) \lesssim n^d(\log n)^2\left(\frac{1}{n}\right)^{\frac{C^2_1}{2}} \lesssim (\log n)^2\left(\frac{1}{n}\right)^{\frac{C^2_1}{2}-d} \leq n^{-C'}
\end{equation*}
as long as the constant $C$ in (\ref{eq:select:band}) is chosen to be sufficiently large.

\end{proof}

\subsection{Proof of the upper bound for $1 \leq q < \infty$}\label{eq:meichifan1}

Following a similar technique as in Section~\ref{sec:proof:buxiang2}, we express the adversarial $L_q$-risk of $\hat{f}_{\mathrm{AP}}$ as
\begin{equation}\label{eq:buxian0}
  R_{A,q}(\hat{f}_{\mathrm{AP}}, f) = \mathbb{E}\sup_{X' \in A(X)}\left| \hat{f}_{\mathrm{AP}}(X') - f(X) \right|^q = \mathbb{E}_X \mathbb{E}_{\otimes^n}\sup_{X' \in A(X)}\left| \hat{f}_{\mathrm{AP}}(X') - f(X) \right|^q.
\end{equation}
Since the probability density function of $X$ is bounded above, to construct an upper bound on (\ref{eq:buxian0}), it is sufficient to bound the pointwise adversarial risk. For any given $x \in \Omega$, we apply the triangle inequality and Jensen's inequality to obtain
\begin{equation}\label{eq:buxian1}
\begin{split}
\mathbb{E}_{\otimes^n}\sup_{x' \in A(x)}\left| \hat{f}_{\mathrm{AP}}(x') - f(x) \right|^q &\leq 2^{q-1} \mathbb{E}_{\otimes^n}\sup_{x' \in A(x)}\left| \hat{f}_{\mathrm{AP}}(x') - f(x') \right|^q\\
     &+ 2^{q-1} \sup_{x' \in A(x)}\left| f(x') - f(x) \right|^q.
\end{split}
\end{equation}
Using the result from Lemma~\ref{lem:lip}, we can bound the second term on the right-hand side of (\ref{eq:buxian1}) as follows:
\begin{equation*}\label{eq:buxian2}
  2^{q-1} \sup_{x' \in A(x)}\left| f(x') - f(x) \right|^q \lesssim  r^{q(1 \wedge \beta)}.
\end{equation*}
 Thus, our main goal in the subsequent analysis is to derive an upper bound for the localized sup-norm risk $\mathbb{E}_{\otimes^n}\sup_{x' \in A(x)}| \hat{f}_{\mathrm{AP}}(x') - f(x') |^q$.

Similar to the argument in (\ref{eq:rela1}), there exists an index set $I_x$ such that $A(x) \subseteq \cup_{k \in I_x}\Omega_k$ and $|I_x|\leq ( 2rM +2)^d$. Therefore, the localized sup-norm risk can be upper bounded as follows:
\begin{equation}\label{eq:buxianrela2}
\begin{split}
&\mathbb{E}_{\otimes^n}\sup_{x' \in A(x)}\left| \hat{f}_{\mathrm{AP}}(x') - f(x') \right|^q \leq \mathbb{E}_{\otimes^n}\sup_{k \in I_x}\sup_{x' \in \Omega_k}\left| \hat{f}_{\mathrm{AP}}(x') - f(x') \right|^q \\
     & \lesssim  \mathbb{E}_{\otimes^n}\sup_{k \in I_x}\sup_{x' \in \Omega_k}\left| \tilde{f}_{u_{x'}\hat{h}_{x'}}(x') - f_{u_k}(x') \right|^q + \sup_{k \in I_x}\sup_{x' \in \Omega_k} \left| f_{u_k}(x') - f(x') \right|^q\\
     & \lesssim \mathbb{E}_{\otimes^n}\sup_{k \in I_x}\sup_{x' \in \Omega_k}\left| \tilde{f}_{u_k\hat{h}_k}(x') - f_{u_k}(x') \right|^q + \frac{1}{M^{q\beta}},
\end{split}
\end{equation}
where $f_{u_k}$ is the Taylor polynomial of $f$ of degree $\lfloor \beta \rfloor$ at the point $u_k$ (see the definition of (\ref{eq:taylor_poly})), and $u_k$ is the center of $\Omega_k$. The first inequality follows from $A(x) \subseteq \cup_{k \in I_x}\Omega_k$, the second inequality follows from the definition of $\hat{f}_{\mathrm{AP}}$ in (\ref{eq:buxianliang1}), the triangle inequality, and Jensen's inequality, and the last inequality stems from the definition $u_{x'} = u_k$ and $\hat{h}_{x'} = \hat{h}_{k}$ when $x' \in \Omega_k$, as well as the result from Lemma~\ref{lem:taylor}.

Our main goal is now to upper bound $\mathbb{E}_{\otimes^n}\sup_{k \in I_x}\sup_{x' \in \Omega_k}| \tilde{f}_{u_k\hat{h}_k}(x') - f_{u_k}(x') |^q$. We proceed as follows:
\begin{subequations}
\begin{align}
&\mathbb{E}_{\otimes^n}\sup_{k \in I_x}\sup_{x' \in \Omega_k}\left| \tilde{f}_{u_k\hat{h}_k}(x') - f_{u_k}(x') \right|^q \nonumber\\
& = \mathbb{E}_{\otimes^n}\sup_{k \in I_x}\sup_{x' \in \Omega_k} \left| \tilde{f}_{u_k\hat{h}_k}(x') - \tilde{f}_{u_k\bar{h}}(x') + \tilde{f}_{u_k\bar{h}}(x') - f_{u_k}(x') \right|^q \nonumber \\
&\lesssim \mathbb{E}_{\otimes^n}\sup_{k \in I_x}\sup_{x' \in \Omega_k} \left| \tilde{f}_{u_k\hat{h}_k}(x') - \tilde{f}_{u_k\bar{h}}(x') \right|^q \label{eqn:line-3} \\
&+ \mathbb{E}_{\otimes^n}\sup_{k \in I_x}\sup_{x' \in \Omega_k} \left| \tilde{f}_{u_k\bar{h}}(x') - f_{u_k}(x') \right|^q, \label{eqn:line-4}
\end{align}
\end{subequations}
where $\bar{h}$ is defined in (\ref{eq:oracle_h}).

\subsubsection{Upper bounding~(\ref{eqn:line-3})}

Recalling the events $\mathcal{C}_{\bar{h}}$ and $\mathcal{E}_{u_k\bar{h}}$ defined in Lemmas~\ref{lem:C_n_bar}--\ref{lem:E_n}, we decompose the term (\ref{eqn:line-3}) as follows:
\begin{subequations}
\begin{align}
&\mathbb{E}_{\otimes^n}\sup_{k \in I_x}\sup_{x' \in \Omega_k} \left| \tilde{f}_{u_k\hat{h}_k}(x') - \tilde{f}_{u_k\bar{h}}(x') \right|^q \nonumber\\
& = \mathbb{E}_{\otimes^n}\sup_{k \in I_x}\sup_{x' \in \Omega_k} \left| \tilde{f}_{u_k\hat{h}_k}(x') - \tilde{f}_{u_k\bar{h}}(x') \right|^q 1_{\{ \mathcal{C}_{\bar{h}} \cap \mathcal{E}_{u_k\bar{h}} \}} \label{eqn:line-5}\\
&+ \mathbb{E}_{\otimes^n}\sup_{k \in I_x}\sup_{x' \in \Omega_k} \left| \tilde{f}_{u_k\hat{h}_k}(x') - \tilde{f}_{u_k\bar{h}}(x') \right|^q 1_{\{ \mathcal{C}_{\bar{h}} \cap \mathcal{E}_{u_k\bar{h}}^c \}} \label{eqn:line-6} \\
&+ \mathbb{E}_{\otimes^n}\sup_{k \in I_x}\sup_{x' \in \Omega_k} \left| \tilde{f}_{u_k\hat{h}_k}(x') - \tilde{f}_{u_k\bar{h}}(x') \right|^q 1_{\{ \mathcal{C}_{\bar{h}}^{c} \}}. \label{eqn:line-7}
\end{align}
\end{subequations}
Note that $\mathcal{C}_{\bar{h}}$ is an event on $\mathfrak{X}_n$, while the event $\mathcal{E}_{u_k\bar{h}}$ depends on both the design data and the response vector.

\underline{\textsc{Upper bounding}~(\ref{eqn:line-5}).} Let us now state several consequences when the events $\mathcal{C}_{\bar{h}}$ and $\mathcal{E}_{u_k\bar{h}}$ hold simultaneously. By Lemma~\ref{lem:C_n_bar}, we know that $n_{u_k \bar{h}} >0$, $\tilde{B}_{u_k\bar{h}} = B_{u_k\bar{h}}$, and $\tilde{D}_{u_k\bar{h}} = D_{u_k\bar{h}}$. Furthermore, by the definition in (\ref{eq:xiongxiong2}), we have $ \bar{h} \leq \hat{h}_k$ and $n_{u_k \hat{h}_k} \geq n_{u_k \bar{h}} >0$. Since $\bar{h}, \hat{h}_k \in H$ and $\bar{h} \leq \hat{h}_k$, the definition of $\hat{h}_k$ in (\ref{eq:select:band}) implies
\begin{equation}\label{eq:jiaquan2}
  \left\| D_{u_k\bar{h}}^{-\frac{1}{2}}\left(  a_{u_k \bar{h}} - B_{u_k\bar{h}}R_{\bar{h}\hat{h}_k}\tilde{B}_{u_k \hat{h}_k}^{-1}a_{u_k \hat{h}_k}\right) \right\|_{\infty}  \leq  \frac{C\sqrt{\log n}}{\sqrt{n\bar{h}^d}}.
\end{equation}
Thus, when $\mathcal{C}_{\bar{h}} \cap \mathcal{E}_{u_k\bar{h}}$ holds, for any $x' \in  \Omega_k$, we have
\begin{equation}\label{eq:shabinv1}
  \begin{split}
     &\left|\tilde{f}_{u_k\hat{h}_k}(x') - \tilde{f}_{u_k \bar{h}}(x') \right|= \left|\tilde{f}_{u_k\hat{h}_k}(x') - \hat{f}_{u_k \bar{h}}(x') \right| \\
       & = \left| U^{\top}\left(\frac{x' - u_k}{\hat{h}_k}\right)\tilde{B}_{u_k \hat{h}_k}^{-1}a_{u_k \hat{h}_k} - U^{\top}\left(\frac{x' - u_k}{\bar{h}}\right)B_{u_k \bar{h}}^{-1}a_{u_k\bar{h}}\right|\\
       & = \left| \sum_{0 \leq |s| \leq \ell}\frac{1}{s!}\left( \frac{x' - u_k}{\hat{h}_k} \right)^s \left( \tilde{B}_{u_k \hat{h}_k}^{-1}a_{u_k \hat{h}_k} \right)_s - \sum_{0 \leq |s| \leq \ell}\frac{1}{s!}\left( \frac{x' - u_k}{\bar{h}} \right)^s \left( B_{u_k \bar{h}}^{-1}a_{u_k \bar{h}} \right)_s \right|\\
       & = \left| \sum_{0 \leq |s| \leq \ell}\frac{1}{s!}\frac{\left( x' - u_k \right)^s}{{\bar{h}}^{|s|}} \left[ \frac{{\bar{h}}^{|s|}}{\hat{h}_k^{|s|}}\left( \tilde{B}_{u_k \hat{h}_k}^{-1}a_{u_k \hat{h}_k} \right)_s -  \left( B_{u_k \bar{h}}^{-1}a_{u_k \bar{h}} \right)_s \right] \right|\\
       & \leq \sum_{0 \leq |s| \leq \ell}\frac{1}{s!}\frac{\left|\left( x' - u_k \right)^s\right|}{{\bar{h}}^{|s|}} \left| \frac{{\bar{h}}^{|s|}}{\hat{h}_k^{|s|}}\left( \tilde{B}_{u_k \hat{h}_k}^{-1}a_{u_k \hat{h}_k} \right)_s -  \left( B_{u_k \bar{h}}^{-1}a_{u_k \bar{h}} \right)_s \right|\\
       & \leq \sum_{0 \leq |s| \leq \ell}\frac{1}{s!}\frac{1}{(2M\bar{h})^{|s|}} \left| \frac{{\bar{h}}^{|s|}}{\hat{h}_k^{|s|}}\left( \tilde{B}_{u_k \hat{h}_k}^{-1}a_{u_k \hat{h}_k} \right)_s -  \left( B_{u_k \bar{h}}^{-1}a_{u_k \bar{h}} \right)_s \right|,
  \end{split}
\end{equation}
where the first equality follows from (\ref{eq:lpe_modified}) and the fact that $\tilde{B}_{u_k\bar{h}} = B_{u_k\bar{h}}$, the second equality follows from (\ref{eq:lpe_modified}) and $n_{u_k \hat{h}_k}  >0$, and the last inequality follows from $x' \in \Omega_k$.

Recalling the definition of the diagonal matrix $R_{\bar{h}\hat{h}_k}$ in (\ref{eq:jiaquan3}), we obtain the following bound for any $0 \leq |s| \leq \ell$:
\begin{equation}\label{eq:shabinv2}
  \begin{split}
  &\left| \frac{{\bar{h}}^{|s|}}{\hat{h}_k^{|s|}}\left( \tilde{B}_{u_k \hat{h}_k}^{-1}a_{u_k \hat{h}_k} \right)_s -  \left( B_{u_k \bar{h}}^{-1}a_{u_k \bar{h}} \right)_s \right|\leq \left\| R_{\bar{h}\hat{h}_k}\tilde{B}_{u_k \hat{h}_k}^{-1}a_{u_k \hat{h}_k} - B_{u_k \bar{h}}^{-1}a_{u_k \bar{h}} \right\|_{\infty}\\
       & = \left\| D_{u_k\bar{h}}^{-\frac{1}{2}} \left( D_{u_k\bar{h}}^{-\frac{1}{2}} B_{u_k\bar{h}} D_{u_k\bar{h}}^{-\frac{1}{2}} \right)^{-1}  D_{u_k\bar{h}}^{-\frac{1}{2}}\left(B_{u_k\bar{h}}R_{\bar{h}\hat{h}_k}\tilde{B}_{u_k \hat{h}_k}^{-1}a_{u_k \hat{h}_k}  -  a_{u_k \bar{h}}  \right) \right\|_{\infty}\\
       & \leq \left\| D_{u_k\bar{h}}^{-\frac{1}{2}} \left( D_{u_k\bar{h}}^{-\frac{1}{2}} B_{u_k\bar{h}} D_{u_k\bar{h}}^{-\frac{1}{2}} \right)^{-1}  D_{u_k\bar{h}}^{-\frac{1}{2}}\left(B_{u_k\bar{h}}R_{\bar{h}\hat{h}_k}\tilde{B}_{u_k \hat{h}_k}^{-1}a_{u_k \hat{h}_k}  -  a_{u_k \bar{h}}  \right) \right\|\\
       & \leq \lambda_{\min}^{-1}\left( B_{u_k\bar{h}} \right) \lambda_{\max}^{\frac{1}{2}}\left(D_{u_k\bar{h}} \right) \left\|  D_{u_k\bar{h}}^{-\frac{1}{2}}\left(B_{u_k\bar{h}}R_{\bar{h}\hat{h}_k}\tilde{B}_{u_k \hat{h}_k}^{-1}a_{u_k \hat{h}_k}  -  a_{u_k \bar{h}}  \right)  \right\|\\
       & \leq \lambda_{\min}^{-1}\left( B_{u_k\bar{h}} \right)\lambda_{\max}^{\frac{1}{2}}\left(D_{u_k\bar{h}} \right) \sqrt{N_{\ell,d}} \left\|D_{u_k\bar{h}}^{-\frac{1}{2}}\left(B_{u_k\bar{h}}R_{\bar{h}\hat{h}_k}\tilde{B}_{u_k \hat{h}_k}^{-1}a_{u_k \hat{h}_k}  -  a_{u_k \bar{h}}  \right) \right\|_{\infty}\\
       & \lesssim  \frac{\sqrt{\log n}}{\sqrt{n\bar{h}^d}},
  \end{split}
\end{equation}
where the second and fourth inequalities follow from $\|\cdot \|_{\infty} \leq \|\cdot \| \leq \sqrt{N_{\ell,d}}\|\cdot \|_{\infty} $, and the last inequality is derived using Lemma~\ref{lem:C_n_bar} and equation (\ref{eq:jiaquan2}).

By combining equations (\ref{eq:shabinv1})--(\ref{eq:shabinv2}) with the condition $M \asymp n$ and noting that $M\bar{h} \gtrsim 1$, we conclude that
\begin{equation*}
  \sup_{x' \in \Omega_k} \left| \tilde{f}_{u_k\hat{h}_k}(x') - \tilde{f}_{u_k\bar{h}}(x') \right|^q 1_{\{ \mathcal{C}_{\bar{h}} \cap \mathcal{E}_{u_k\bar{h}} \}} \lesssim \left(\frac{\log n}{n\bar{h}^d}\right)^{\frac{q}{2}}.
\end{equation*}
Therefore, we obtain
\begin{equation}\label{eq:jieduan1}
  \mathbb{E}_{\otimes^n}\sup_{k \in I_x}\sup_{x' \in \Omega_k} \left| \tilde{f}_{u_k\hat{h}_k}(x') - \tilde{f}_{u_k\bar{h}}(x') \right|^q 1_{\{ \mathcal{C}_{\bar{h}} \cap \mathcal{E}_{u_k\bar{h}} \}}  \lesssim \left(\frac{\log n}{n\bar{h}^d}\right)^{\frac{q}{2}}.
\end{equation}

\underline{\textsc{Upper bounding}~(\ref{eqn:line-6}).} The following upper bound holds:
\begin{equation}\label{eq:shabinv6}
  \begin{split}
    & \mathbb{E}_{\otimes^n}\sup_{k \in I_x}\sup_{x' \in \Omega_k} \left| \tilde{f}_{u_k\hat{h}_k}(x') - \tilde{f}_{u_k\bar{h}}(x') \right|^q 1_{\{ \mathcal{C}_{\bar{h}} \cap \mathcal{E}_{u_k\bar{h}}^c \}}   \\
       & \lesssim \mathbb{E}_{\otimes^n} \left(\sup_{k \in I_x}\sup_{x' \in \Omega_k} \left| \tilde{f}_{u_k\hat{h}_k}(x') \right|\right)^q 1_{\{ \mathcal{C}_{\bar{h}} \cap \mathcal{E}_{u_k\bar{h}}^c \}} + \mathbb{E}_{\otimes^n} \left(\sup_{k \in I_x}\sup_{x' \in \Omega_k} \left| \tilde{f}_{u_k\bar{h}}(x') \right|\right)^q 1_{\{ \mathcal{C}_{\bar{h}} \cap \mathcal{E}_{u_k\bar{h}}^c \}}.
  \end{split}
\end{equation}
We first focus on the second term. By the Cauchy-Schwarz inequality, we have
\begin{equation}\label{eq:shabinv3}
  \begin{split}
       & \mathbb{E}_{\otimes^n} \left(\sup_{k \in I_x}\sup_{x' \in \Omega_k} \left| \tilde{f}_{u_k\bar{h}}(x') \right|\right)^q 1_{\{ \mathcal{C}_{\bar{h}} \cap\mathcal{E}_{u_k\bar{h}}^c \}} \\
       & = \mathbb{E}_{\otimes^n} \left[ \mathbb{E}_{\otimes^n} \left(\sup_{k \in I_x}\sup_{x' \in \Omega_k} \left| \tilde{f}_{u_k\bar{h}}(x') \right|\right)^q 1_{\{ \mathcal{C}_{\bar{h}} \cap\mathcal{E}_{u_k\bar{h}}^c \}} \mid \mathfrak{X}_n \right]\\
       & \leq \mathbb{E}_{\otimes^n} \left\{ \left[\mathbb{E}_{\otimes^n} \left(\sup_{k \in I_x}\sup_{x' \in \Omega_k} \left| \tilde{f}_{u_k\bar{h}}(x') \right|\right)^{2q}\mid \mathfrak{X}_n \right]^{\frac{1}{2}} \left[ \mathbb{P}_{\otimes^n}\left( \mathcal{C}_{\bar{h}} \cap\mathcal{E}_{u_k\bar{h}}^c \mid \mathfrak{X}_n \right) \right]^{\frac{1}{2}} \right\}
  \end{split}
\end{equation}
For the conditional probability $\mathbb{P}_{\otimes^n}( \mathcal{C}_{\bar{h}} \cap\mathcal{E}_{u_k\bar{h}}^c \mid \mathfrak{X}_n )$, Lemma~\ref{lem:E_n} implies that it can be bounded by a vanishing polynomial probability with arbitrarily large order. Applying similar techniques as in Section~\ref{sec:exinnv11}, we obtain
\begin{equation}\label{eq:shabinv4}
  \begin{split}
       \mathbb{E}_{\otimes^n}\left[ \left(\sup_{k \in I_x}\sup_{x' \in \Omega_k} \left| \tilde{f}_{u_k\bar{h}}(x') \right|\right)^{2q}\mid \mathfrak{X}_n \right]& \lesssim \frac{1}{(\tau \bar{h}^d)^{2q}} + \frac{\mathbb{E}_{\otimes^n}\left(\| \tilde{\gamma}_{nx}\|_{\infty}^{2q}\mid \mathfrak{X}_n\right)}{(\tau n\bar{h}^d)^{q}(\tau \bar{h}^d)^{2q}}\\
       & \lesssim \frac{1}{(\tau \bar{h}^d)^{2q}} + \frac{\left[\log (rM +1) \right]^{q}}{(\tau n\bar{h}^d)^{q}(\tau \bar{h}^d)^{2q}}\\
       & \lesssim \left( n\sqrt{\log n} \right)^{2q},
  \end{split}
\end{equation}
where the last inequality follows from $\tau = \tau(\bar{h}) = (n\bar{h}^d)^{-1}$ and $M \asymp n$. Combining (\ref{eq:shabinv3})--(\ref{eq:shabinv4}) with Lemma~\ref{lem:E_n}, we obtain
\begin{equation}\label{eq:shabinv5}
  \mathbb{E}_{\otimes^n} \left(\sup_{k \in I_x}\sup_{x' \in \Omega_k} \left| \tilde{f}_{u_k\bar{h}}(x') \right|\right)^q 1_{\{ \mathcal{C}_{\bar{h}} \cap\mathcal{E}_{u_k\bar{h}}^c \}} \lesssim \left( n\sqrt{\log n} \right)^{q} n^{-\frac{C'}{2}},
\end{equation}
which is upper bounded by $n^{-q/2}$ when $C'$ is chosen larger than $3q/2$. By applying similar reasoning, we find that the first term in (\ref{eq:shabinv6}) can also be bounded by $n^{-q/2}$. Therefore, we conclude that
\begin{equation}\label{eq:shabinv7}
  \mathbb{E}_{\otimes^n}\sup_{k \in I_x}\sup_{x' \in \Omega_k} \left| \tilde{f}_{u_k\hat{h}_k}(x') - \tilde{f}_{u_k\bar{h}}(x') \right|^q 1_{\{ \mathcal{C}_{\bar{h}} \cap \mathcal{E}_{u_k\bar{h}}^c \}} \lesssim n^{-q/2}
\end{equation}
when the constant $C$ in (\ref{eq:select:band}) is sufficiently large.

\underline{\textsc{Upper bounding}~(\ref{eqn:line-7}).} By combining the techniques used in Section~\ref{sec:exinnv11} with the results in Lemma~\ref{lem:C_n_bar}, we obtain
\begin{equation}\label{eq:shabinv8}
  \mathbb{E}_{\otimes^n}\sup_{k \in I_x}\sup_{x' \in \Omega_k} \left| \tilde{f}_{u_k\hat{h}_k}(x') - \tilde{f}_{u_k\bar{h}}(x') \right|^q 1_{\{ \mathcal{C}_{\bar{h}}^{c} \}} \lesssim n^{-q/2}.
\end{equation}

\underline{\textsc{Constructing the final bound}.} Combining the upper bounds from (\ref{eq:jieduan1}), (\ref{eq:shabinv7}), and (\ref{eq:shabinv8}), we have
\begin{equation}\label{eq:shabinv9}
  \mathbb{E}_{\otimes^n}\sup_{k \in I_x}\sup_{x' \in \Omega_k} \left| \tilde{f}_{u_k\hat{h}_k}(x') - \tilde{f}_{u_k\bar{h}}(x') \right|^q \lesssim \left(\frac{\log n}{n\bar{h}^d}\right)^{\frac{q}{2}} + \left(  \frac{1}{n} \right)^{\frac{q}{2}} \lesssim \left(\frac{\log n}{n\bar{h}^d}\right)^{\frac{q}{2}},
\end{equation}
which completes the proof for the term (\ref{eqn:line-3}).

\subsubsection{Upper bounding~(\ref{eqn:line-4})}

We can upper bound the term (\ref{eqn:line-4}) using techniques similar to those used for bounding terms (\ref{eqn:line-1})--(\ref{eqn:line-2}). Specifically, we find that
\begin{equation}\label{eq:shabinv10}
\begin{split}
   \mathbb{E}_{\otimes^n}\sup_{k \in I_x}\sup_{x' \in \Omega_k} \left| \tilde{f}_{u_k\bar{h}}(x') - f_{u_k}(x') \right|^q & \lesssim \bar{h}^{q\beta} + \left[\frac{\log (rM +1) }{n\bar{h}^d}\right]^{\frac{q}{2}} \\
     & \lesssim \bar{h}^{q\beta} + \left(\frac{\log n }{n\bar{h}^d}\right)^{\frac{q}{2}}.
\end{split}
\end{equation}

\subsubsection{Completing the proof}

Combining the upper bounds from (\ref{eq:buxian1}), (\ref{eq:buxianrela2}), (\ref{eq:shabinv9}), and (\ref{eq:shabinv10}), we obtain
\begin{equation*}
\begin{split}
   R_{A,q}(\hat{f}_{\mathrm{AP}}, f) & \lesssim r^{q(1 \wedge \beta)} + \frac{1}{M^{q\beta}} + \bar{h}^{q\beta} + \left(\frac{\log n }{n\bar{h}^d}\right)^{\frac{q}{2}} \\
     & \lesssim r^{q(1 \wedge \beta)}  + \frac{1}{n^{q\beta}} + \left(\frac{\log n }{n\bar{h}^d}\right)^{\frac{q}{2}},\\
\end{split}
\end{equation*}
where the second inequality uses $M \asymp n$ and (\ref{eq:oracle_h}). Using an analysis similar to that in Lemma~\ref{lem:C_n_bar}, we can show that
$$
n\bar{h}^d \gtrsim n^{\frac{2\beta}{2\beta+d}}(\log n)^{\frac{d}{2\beta+d}}.
$$
Therefore, we conclude that
\begin{equation*}
\begin{split}
   R_{A,q}(\hat{f}_{\mathrm{AP}}, f) & \lesssim r^{q(1 \wedge \beta)} + \left[\frac{\log n}{n^{\frac{2\beta}{2\beta+d}}(\log n)^{\frac{d}{2\beta+d}}} \right]^{\frac{q}{2}} \\
     & =  r^{q(1 \wedge \beta)} + \left[\frac{(\log n)^{\frac{2\beta}{2\beta+d}}}{n^{\frac{2\beta}{2\beta+d}}} \right]^{\frac{q}{2}} = r^{q(1 \wedge \beta)} + \left(\frac{n}{\log n}\right)^{-\frac{q\beta}{2\beta+d}},
\end{split}
\end{equation*}
which establishes the desired result.

\subsection{Proof of the upper bound for $q = \infty$}

To prove the adaptive upper bound for the adversarial sup-norm risk, we employ a similar technique to that used in Section~\ref{sec:proof_upper_2} and Section~\ref{eq:meichifan1}. Given the overlap in methods, we omit the repetitive details here for brevity.

\section{Proof of Theorem~\ref{theo:lower}}\label{sec:proof_lower}

\subsection{Proof of the lower bound for $1 \leq q < \infty$}

We complete the proof of Lemma~\ref{lem:mirror} in the main text by verifying that the base regression function defined in (\ref{eq:f_00}) lies within the H\"{o}lder class $\mathcal{F}(\beta, C_{\beta}/2)$ for $0 < \beta \leq 1$. Additionally, we construct specific examples of $\mathcal{S}$ as described in Lemma~\ref{lem:bump} and provide a brief proof of Lemma~\ref{eq:dayuechen1} for completeness.

\begin{proof}[Proof of Lemma~\ref{lem:mirror}]

We verify that the function $f_0$ defined in (\ref{eq:f_00}) belongs to $\mathcal{F}(\beta, C_{\beta}/2)$ by showing that for any $x,z\in [0,1]^d$,
\begin{equation}\label{eq:laj1}
  \begin{split}
     \left| f_0(x) - f_0(z) \right| &  \leq \frac{C_{\beta}}{2} \| x - z \|^{\beta}.
  \end{split}
\end{equation}
We will prove this inequality for the case $0 \leq x_1, z_1 \leq 4r$. Similar arguments can be applied to other cases. When $0 \leq x_1,z_1 < 2r$, both $f_0(x)$ and $f_0(z)$ are equal to zero by the definition of $f_0$. Thus, $|f_0(x) - f_0(z)| = 0$, which satisfies (\ref{eq:laj1}). When $0 \leq x_1 < 2r \leq z_1 < 4r$, we have $f_0(x) = 0$, and
\begin{equation*}
  \begin{split}
     \left| f_0(x) - f_0(z) \right| & = \left| f_0(z) \right| = \frac{C_{\beta}}{2}\left( 8r \right)^{\beta
     }\left| \frac{z_1 - 2r}{8r} \right|^{\beta}\\
       & \leq \frac{C_{\beta}}{2}\left| z_1 - x_1\right|^{\beta} \leq \frac{C_{\beta}}{2} \| x - z \|^{\beta},
  \end{split}
\end{equation*}
where the second equality follows from the definition of $f_0$, and the first inequality follows from $x_1 < 2r \leq z_1$. When $2r \leq x_1,z_1 \leq 4r$,
\begin{equation*}
\begin{split}
   \left| f_0(x) - f_0(z) \right| & = \frac{C_{\beta}}{2}\left( 8r \right)^{\beta
     }\left| \left( \frac{x_1 - 2r}{8r} \right)^{\beta} - \left( \frac{z_1 - 2r}{8r} \right)^{\beta} \right| \\
     & \leq \frac{C_{\beta}}{2}\left( 8r \right)^{\beta
     }\left|\frac{ x_1 - z_1}{8r}\right|^{\beta}  \leq \frac{C_{\beta}}{2} \| x - z \|^{\beta},
\end{split}
\end{equation*}
where the first inequality follows from the fact that $x^{\beta}$ is a $(\beta,1)$-H\"{o}lder smooth function. Thus we have proved $f_0 \in \mathcal{F}(\beta, C_{\beta}/2)$.

\end{proof}

\begin{proof}[Proof of Lemma~\ref{lem:bump}]

The proof of this lemma relies on constructing specific local packing sets within the H\"{o}lder class $\mathcal{F}(\beta,C_{\beta})$. These constructions date back to the original work of \cite{Kolmogorov1959} and have been extensively explored in subsequent studies, including \cite{Ibragimov1982, Stone1982global, Korostelev1993, Yang1999Information}, and many others.

Choose a function $\phi :\mathbb{R}^d \to \mathbb{R}$ that satisfies the following conditions:
  \begin{description}
    \item[(i)] The support of $\phi$ is contained within $[-1/2,1/2]^d$.
    \item[(ii)] $ 0 < \sup_{x \in [-1/2,1/2]^d}|\phi(x)| < C$, where $C$ is a positive constant.
    \item[(iii)] $\phi\in \mathcal{F}(\beta, C_{\beta}/2)$.
  \end{description}
  A specific example of such a function $\phi$ can be found on pages 92–93 of \cite{Tsybakov2009book}.

Next, we partition the domain $\Omega = [0,1]^d$ into $L_n^d$ hypercubes $C_l, l = 1, \ldots, L_n^d$, each with a side length of $1/L_n$ and centers ${a_l}$, where $l = (l_1,\ldots,l_d)^{\top} \in \{1,\ldots,L_n\}^d$, $a_l = \frac{l - (1/2,\ldots,1/2)^{\top}}{L_n}$, and $C_l = \prod_{i=1}^d [(l_i-1)/L_n,l_i/L_n)$. For each $l = 1, \ldots, L_n^d$, define
\begin{equation}\label{eq:touxiang1}
\phi_l(x) \triangleq \frac{\phi\left[L_n(x - a_l)\right]}{L_n^{\beta}},
\end{equation}
where $L_n$ will be determined later.

Let $\bar{W} \triangleq \{w = (w_1, \ldots, w_{L_n^d}) : w_i \in \{1, -1\} \} = \{1, -1\}^{L_n^d}$ denote the set of vectors composed of components equal to either $1$ or $-1$. For each $w \in \bar{W}$, define the function
\begin{equation*}
f_{w} = f_0 + \phi_{w} = f_0 + \sum_{l \in \{1, \ldots, L_n\}^d} w_{l} \phi_{l},
\end{equation*}
where $f_0 \in \mathcal{F}(\beta, C_{\beta}/2)$ is the base regression function defined in Section~\ref{sec:base}, and $\phi_{l}(x)$ is defined by equation (\ref{eq:touxiang1}). We define the subset of functions as $\bar{\mathcal{S}} \triangleq \{f_{w} : w \in \bar{W} \}$.

We first prove that $\bar{\mathcal{S}} \subseteq \mathcal{F}(\beta, C_{\beta})$. For any multi-index $s$ with $|s| = \lfloor \beta \rfloor$ and for any $x, z \in [0,1]^d$, we have
\begin{equation*}
\begin{split} \left| D^s f_{w}(x) - D^s f_{w}(z) \right| & \leq \left| D^s f_0(x) - D^s f_0(z)\right| + \left| D^s \phi_{w}(x) - D^s \phi_{w}(z)\right| \\
& \leq \frac{C_{\beta}}{2} \left\| x - z \right\|^{\beta - \lfloor \beta \rfloor} + \left| D^s \phi_{w}(x) - D^s \phi_{w}(z)\right| \\ & \leq C_{\beta} \left\| x - z \right\|^{\beta - \lfloor \beta \rfloor},
\end{split}
\end{equation*}
where the second inequality follows from the definition of $f_0$, and the final inequality is derived using the approach outlined on pages 40–41 of \cite{Gyorfi2002book}. Therefore, we conclude that $\bar{\mathcal{S}} \subseteq \mathcal{F}(\beta, C_{\beta})$.

Then, we construct a subset $\mathcal{S}$ of $\bar{\mathcal{S}}$ which has nearly the same number of elements as $\bar{\mathcal{S}}$ but such that each pair of functions in $\mathcal{S}$ is far apart. The technique for constructing $\mathcal{S}$ is based on the Varshamov-Gilbert lemma (see, e.g., Lemma 2.9 of \cite{Tsybakov2009book}). It claims that when $L_n^d\geq 8$, there exists a subset $W=\{w^{(1)},\ldots,w^{(J_n)} \}$ of $\bar{W}$ such that $J_n \geq 2^{L_n^d/8}$, and
\begin{equation}\label{eq:vg_lemma}
  \rho\left(w^{(j)},w^{(k)}\right)\geq \frac{L_n^d}{8}, \quad \forall\; 0\leq j < k \leq J_n,
\end{equation}
where $\rho(\cdot,\cdot)$ denotes the Hamming distance. Define $\mathcal{S}\triangleq \{f_j \triangleq f_{w^{(j)}}:w^{(j)} \in W  \}$. We have that $\mathrm{Card}(\mathcal{S}) = J_n \geq 2^{L_n^d/8}$.

We now verify the conditions~(i)--(iv) stated in Lemma~\ref{lem:bump}. The main task is to explicitly provide the forms of the quantities $c_1$, $c_2$, $J_n$, and $C$. We begin by verifying the conditions (ii)--(iii). For any $f_{j}\in \mathcal{S}$, we have the following expression for the $L_2$-norm:
\begin{equation}\label{eq:haizai1}
  \begin{split}
     \left\| f_{ j} - f_0 \right\|^2_2 & = \left\| \sum_{ l\in\{1,\ldots,L_n \}^d}w_{ l}^{(j)}\phi_{ l}\right\|^2_2 = \sum_{l\in\{1,\ldots,L_n \}^d} \int_{C_{l}}\left|\phi_{l}\right(x)|^2 dx \\
       & = \frac{L_n^d}{L_n^{2\beta+d}}\int\left|\phi( x)\right|^2 d x = \frac{\int\left|\phi( x)\right|^2 d x}{L_n^{2\beta}} \triangleq  \epsilon_n^2.
  \end{split}
\end{equation}
Next, for the $L_q$-norm with $q \geq 1$, we have
\begin{equation*}\label{eq:daishu1}
  \begin{split}
     \left\| f_{ j} - f_0 \right\|^q_q & = \left\| \sum_{ l\in\{1,\ldots,L_n \}^d}w_{ l}^{(j)}\phi_{ l}\right\|^q_q = \sum_{l\in\{1,\ldots,L_n \}^d} \int_{C_{l}}\left|\phi_{l}\right(x)|^q dx \\
       & = \frac{L_n^d}{L_n^{q\beta+d}}\int\left|\phi( x)\right|^q d x = \frac{\int\left|\phi( x)\right|^q d x}{L_n^{q\beta}}\\
       & = \frac{\int\left|\phi( x)\right|^q d x}{\left( \int \left| \phi(x) \right|^2 dx \right)^{\frac{q}{2}}}\epsilon_n^q \triangleq c_2^q \epsilon_n^q.
  \end{split}
\end{equation*}
For the condition (i), we consider two functions $f_{j},f_{k} \in \mathcal{S}$. We have the following expression for the $L_q$-norm of their difference:
\begin{equation*}\label{eq:eps_n_packing}
\begin{split}
\left\| f_{j} - f_{k} \right\|_q^q &=  \left\| \sum_{l\in\{1,\ldots,L_n \}^d}\left(w_{l}^{(j)}-w_{l}^{(k)}\right)\phi_{l}(x) \right\|_q^q\\
   &= \sum_{l\in\{1,\ldots,L_n \}^d} \left[\left|w_{l}^{(j)}-w_{l}^{(k)}\right|^q\int_{C_{l}}\left|\phi_{l}(x)\right|^q dx\right]\\
   &=\frac{2^q\int\left|\phi(x)\right|^q d x}{L_n^{q\beta+d}}\rho\left( w^{(j)}, w^{(k)}\right)\geq \frac{2^{q}\int\left|\phi( x)\right|^q d x}{8L_n^{q\beta}}\\
    & = \frac{2^{q-3}\int\left|\phi( x)\right|^q d x}{\left( \int \left| \phi(x) \right|^2 dx \right)^{\frac{q}{2}}}\epsilon_n^q\triangleq c_1^q\epsilon_n^q,
\end{split}
\end{equation*}
where the second equality follows from the orthogonality of the functions $\phi_{ l}$, the third equality comes from the definition of $\phi_{ l}$, the last inequality is obtained using the bound from (\ref{eq:vg_lemma}), and the last step is a direct consequence of the definition of $\epsilon_n$ in (\ref{eq:haizai1}). Therefore, the conditions (i)--(iii) are verified, with explicit expressions for $c_1$, $c_2$, and $\epsilon_n$.

What remains is to choose an appropriate $L_n$ such that the condition (iv) holds. To this end, we define
\begin{equation}\label{eq:qusinv1}
  L_n \triangleq \left\{\left( 16\int\left|\phi( x)\right|^2 d x \right)^{\frac{1}{2\beta+d}} n^{\frac{1}{2\beta+d}} \right\}\vee\left( 32\log 2 \right)^{\frac{1}{d}} = C_1n^{\frac{1}{2\beta+d}} \vee C_2.
\end{equation}
From the relationship between $\epsilon_n$ and $L_n$ in (\ref{eq:haizai1}), we can express $\epsilon_n$ as
$$
\epsilon_n = C_3 n^{-\frac{\beta}{2\beta+d}} \wedge C_4.
$$
Using the definition of $\epsilon_n$, we then obtain the following inequality:
\begin{equation*}
\begin{split}
   2n\epsilon_n^2 & = \frac{2n\int\left|\phi( x)\right|^2 d x}{L_n^{2\beta}} \leq  \frac{L_n^d}{8} \leq \log J_n,
\end{split}
\end{equation*}
where the equality follows from (\ref{eq:haizai1}), and the first inequality follows from (\ref{eq:qusinv1}). Additionally, we observe that
\begin{equation*}
  4\log 2 \leq \frac{L_n^d}{8} \leq \log J_n.
\end{equation*}
Thus, we have
\begin{equation*}
  \log J_n \geq n\epsilon_n^2 + 2\log 2,
\end{equation*}
which verifies the condition (iv).

\end{proof}

\begin{proof}[Proof of Lemma~\ref{eq:dayuechen1}]

We begin by summarizing the key properties of $\mathcal{S}$ presented in Lemma~\ref{lem:bump}, which will be used in the subsequent analysis. According to Lemma~\ref{lem:bump}, $\mathcal{S}$ is defined as $\{f_1,\ldots,f_{J_{n}} \}$, where for $1 \leq j \neq k \leq J_{n}$, the following properties hold:
\begin{itemize}
  \item $\left\| f_j - f_k \right\|_q >  c_1\epsilon_{n}$,
  \item $\left\| f_j - f_0 \right\|_2 \leq \epsilon_{n}$,
  \item $\log J_{n} \geq n \epsilon_{n}^2 + 2\log 2$.
\end{itemize}
Here, $\epsilon_{n}$ is defined as $C(n^{-\frac{\beta}{2\beta+d}} \wedge 1)$. Our proof leverages a general reduction scheme and the local metric entropy method as discussed in \cite{Yang1999Information} and \cite{Wang2014Adaptive} (see also Chapter 2 of \cite{Tsybakov2009book} and Chapter~15 of \cite{Wainwright2019high} for related discussions). For clarity, we denote the joint distribution of $n$ i.i.d. samples with the regression function $f_j$ as $\mathbb{P}_{\otimes^n|j}$.

Let $\zeta$ denote a random variable with uniform distribution on the set $\{1,\ldots, J_{n}\}$. For a realization $\zeta=j$, let us generate a sample $Z_n = \{ (X_1, Y_1), \ldots , (X_n, Y_n) \} \sim \mathbb{P}_{\otimes^n|j}$. Let $\mathbb{Q}_{\otimes^n}$ denote the joint distribution of the random pair $(Z_n, \zeta)$. Obviously, the marginal distribution of $Z_n$ in $\mathbb{Q}_{\otimes^n}$ is given by $\bar{\mathbb{Q}}_{\otimes^n}=\frac{1}{J_{n}}\sum_{j=1}^{J_{n}}\mathbb{P}_{\otimes^n|j}$. Let the uniform distribution on $\{1,\ldots, J_{n}\}$ denoted by $\mathbb{Q}_{\zeta}$ Therefore, the maximum probability involved in Lemma~\ref{eq:dayuechen1} can be lower bounded by
\begin{equation}\label{eq:dada1}
  \begin{split}
     \max_{f_j \in \mathcal{S}}\mathbb{P}_{\otimes^n} \left( \|\hat{f}- f_{j} \|_q \geq \frac{c_1 \epsilon_{n}}{2}  \right) & = \max_{j \in \{1,\ldots, J_{n}\}}\mathbb{P}_{\otimes^n|j} \left( \|\hat{f}- f_{j} \|_q \geq \frac{c_1 \epsilon_{n}}{2}  \right)\\
       & \geq \frac{1}{J_{n}}\sum_{j=1}^{J_{n}}\mathbb{P}_{\otimes^n|j} \left( \|\hat{f}- f_{j} \|_q \geq \frac{c_1 \epsilon_{n}}{2}  \right)\\
       & = \mathbb{Q}_{\otimes^n} \left( \|\hat{f}- f_{\zeta} \|_q \geq \frac{c_1 \epsilon_{n}}{2}  \right).
  \end{split}
\end{equation}

Let us construct a testing procedure $\psi(Z_n)=\arg\min_{j\in \{1,\ldots,J_{n}\}}\|\hat{f}- f_{j} \|_q$. If the true regression function is $f_j$, then when the event $\|\hat{f}- f_{j} \|_q < \frac{c_1 \epsilon_{n}}{2}$ occurs, we must have $\psi(Z_n) = j$ since the functions in $\{f_1,\ldots,f_{J_{n}} \}$ are $c_1 \epsilon_{n}$ apart in $L_q$-distance. Therefore, continuing with (\ref{eq:dada1}), we see
\begin{equation*}
  \begin{split}
      \max_{f_j \in \mathcal{S}}\mathbb{P}_{\otimes^n} \left( \|\hat{f}- f_{j} \|_q \geq \frac{c_1 \epsilon_{n}}{2}  \right) & \geq \frac{1}{J_{n}}\sum_{j=1}^{J_{n}}\mathbb{P}_{\otimes^n|j} \left( \|\hat{f}- f_{j} \|_q \geq \frac{c_1 \epsilon_{n}}{2}  \right)\\
        &\geq \frac{1}{J_{n}}\sum_{j=1}^{J_{n}}\mathbb{P}_{\otimes^n|j} \left[ \psi(Z_n) \neq j  \right]\\
        & = \mathbb{Q}_{\otimes^n} \left[ \psi(Z_n) \neq \zeta \right].
  \end{split}
\end{equation*}
Based on Fano's inequality \citep{Fano1961}, we have the following lower bound:
\begin{equation*}
  \begin{split}
    & \inf_{\hat{f}}\max_{f_j \in \mathcal{S}}\mathbb{P}_{\otimes^n} \left( \|\hat{f}- f_{j} \|_q \geq \frac{c_1 \epsilon_{n}}{2}  \right) \geq \inf_{\hat{f}}\mathbb{Q}_{\otimes^n} \left[ \psi(Z_n) \neq \zeta \right] \\
     & \geq \inf_{\psi}\mathbb{Q}_{\otimes^n} \left[ \psi(Z_n) \neq \zeta \right] \geq  1-\frac{I(Z_n,\zeta)+\log 2}{\log J_{n}},\\
  \end{split}
\end{equation*}
where $I(Z_n,\zeta)$ denotes the mutual information between the random variables $Z_n, \zeta$, which is defined by the Kullback–Leibler divergence between $\mathbb{Q}_{\otimes^n}$ and $\bar{\mathbb{Q}}_{\otimes^n}\mathbb{Q}_{\zeta}$. This mutual information is upper bounded by
\begin{equation*}
\begin{split}
   I(Z_n,\zeta) & = \mathrm{KL}\left( \mathbb{Q}_{\otimes^n} || \bar{\mathbb{Q}}_{\otimes^n}\mathbb{Q}_{\zeta} \right)=\frac{1}{J_{n}}\sum_{j=1}^{J_{n}}\mathrm{KL}(\mathbb{P}_{\otimes^n|j}\| \bar{\mathbb{Q}}_{\otimes^n}) \\
     & \leq \frac{1}{J_{n}}\sum_{j=1}^{J_{n}}\mathrm{KL}(\mathbb{P}_{\otimes^n|j}\| \mathbb{P}_{\otimes^n|0}) \leq \max_{j=1,\ldots,J_{n}}\mathrm{KL}(\mathbb{P}_{\otimes^n|j}\| \mathbb{P}_{\otimes^n|0}),
\end{split}
\end{equation*}
where the first inequality follows from the fact that the Bayes mixture $\bar{\mathbb{Q}}_{\otimes^n}$ minimizes the average Kullback–Leibler divergence.

Let $\mathbb{P}_{X^n} $ denote the joint distribution of $(X_1,\ldots,X_n)$, and let $\mathbb{P}_{Y^n|f_j(X^n)}$ be the conditional distribution of $(Y_1,\ldots,Y_n)$ given $(X_1,\ldots,X_n)$ under the regression function $f_j$. In the Gaussian regression case, we have
\begin{equation*}
  \begin{split}
     \mathrm{KL}(\mathbb{P}_{\otimes^n|j}\| \mathbb{P}_{\otimes^n|0}) & =\frac{1}{2} \mathbb{E}_{X^n} \sum_{i=1}^{n}\left[ f_j(X_i) - f_0(X_i) \right]^2  = \frac{n}{2} \left\| f_j - f_0 \right\|_2^2 \leq \frac{n\epsilon_{n}^2}{2}.
  \end{split}
\end{equation*}
Therefore, we obtain the bound $I(Z_n,\zeta) \leq n\epsilon_{n}^2/2$. This leads to the following lower bound:
\begin{equation*}
  \begin{split}
     \inf_{\hat{f}}\max_{f_j \in \mathcal{S}}\mathbb{P}_{\otimes^n} \left( \|\hat{f}- f_{j} \|_q \geq \frac{c_1 \epsilon_{n}}{2}  \right) & \geq  1-\frac{I(Z_n,\zeta)+\log 2}{\log J_{n}}   \geq \frac{1}{2},
  \end{split}
\end{equation*}
which completes the proof.

\end{proof}

\subsection{Proof of the lower bound for $q = \infty$}

The proof of the lower bound for the adversarial sup-norm risk can be derived directly from the results in \cite{Peng2024}. Specifically, given any $A \in \mathcal{T}(r)$, fixing an adversarial set $A'(x) \subseteq A(x)$ as the additive perturbation set considered in \cite{Peng2024}, the results presented in Section 4.1 of \cite{Peng2024} then imply the lower bound stated in Lemma~\ref{theo:lower}.

\newpage
\bibliographystyle{apalike}
\bibliography{bibliography}

\end{document}